\documentclass{article}

\usepackage{microtype}
\usepackage{graphicx}
\usepackage{subcaption}
\usepackage{booktabs} %
\usepackage{enumitem}
\usepackage{pifont}

\usepackage{makecell}

\newcommand{\cellci}[3]{#1}
\usepackage{adjustbox}
\usepackage{hyperref}

\usepackage{amsthm}
\usepackage{amsmath}
\usepackage{amssymb}
\usepackage{mathtools}

\usepackage[accepted]{icml2026}

\usepackage{mathtools}
\usepackage{thmtools}

\newcommand{\loss}{\mathcal{L}}
\newcommand{\hess}{\widehat{H}}
\newcommand{\reg}{\mathcal{R}}

\newcommand{\obsSpace}{\mathbb{X}}

\newcommand{\lt}{\mathopen{}\mathclose\bgroup\left}
\newcommand{\rt}{\aftergroup\egroup\right}

\DeclarePairedDelimiter\norm{\lVert}{\rVert}

\DeclarePairedDelimiter\rbra{(}{)}
\DeclarePairedDelimiter\cbra{\{}{\}}

\newcommand{\veps}{\varepsilon}

\def\[#1\]{\begin{align}#1\end{align}}		
\def\(#1\){\begin{align*}#1\end{align*}} 	

\renewcommand\d{\ifmmode\mathrm{d}\else\oldd\fi} 

\def\argmin{\operatornamewithlimits{arg\,min}}

\mathtoolsset{centercolon}
\newcommand{\defas}{:=}

\newcommand{\grad}{\nabla}

\newcommand{\iid}{\textrm{i.i.d.}}

\newcommand{\var}{\operatorname{Var}}	
\newcommand{\cov}{\operatorname{Cov}}	

\newcommand{\E}{\mathbb{E}}

\newcommand{\dist}{\sim}
\newcommand{\distiid}{\overset{\text{iid}}{\dist}}

\newcommand{\reals}{\ensuremath{\mathbb{R}}}

\newcommand{\Norm}{\mathcal{N}}

\newcommand{\MLEN}{\widehat{\theta}^{(N)}}
\newcommand{\MLE}{\widehat{\theta}}

\newcommand{\RN}[1]{%
	\textup{\uppercase\expandafter{\romannumeral#1}}%
}

\definecolor{WowColor}{rgb}{.75,0,.75}
\definecolor{SubtleColor}{rgb}{0,0,.50}

\newcounter{margincounter}

\DeclareMathOperator{\Tr}{Tr}

\providecommand{\cellci}[3]{%
#1%
}

\usepackage[capitalize,noabbrev]{cleveref}
\usepackage{autonum}

\crefname{lemma}{Lemma}{Lemmas}
\crefname{corollary}{Corollary}{Corollaries}
\crefname{theorem}{Theorem}{Theorems}
\crefname{assumption}{Assumption}{Assumptions}

\theoremstyle{plain}
\newtheorem{theorem}{Theorem}[section]
\newtheorem{proposition}[theorem]{Proposition}
\newtheorem{lemma}[theorem]{Lemma}
\newtheorem{corollary}[theorem]{Corollary}
\theoremstyle{definition}

\theoremstyle{remark}
\newtheorem{remark}[theorem]{Remark}

\usepackage[textsize=small]{todonotes}

\icmltitlerunning{Accurate Large-sample UQ using SG-MCMC}

\begin{document}

\twocolumn[
  \icmltitle{Accurate Large-sample Uncertainty Quantification \\
  using Stochastic Gradient Markov Chain Monte Carlo}
  \icmlsetsymbol{equal}{*}

  \begin{icmlauthorlist}
  \icmlauthor{Yu Wang}{MS}
  \icmlauthor{Jie Ding}{QST}
  \icmlauthor{Jonathan H.~Huggins}{MS,CDS}
  \end{icmlauthorlist}

    \icmlaffiliation{MS}{Department of Mathematics \& Statistics, Boston University}
    \icmlaffiliation{QST}{Questrom School of Business, Boston University}
     \icmlaffiliation{CDS}{Faculty of Computing \& Data Sciences, Boston University}
    
    \icmlcorrespondingauthor{Jonathan H.~Huggins}{huggins@bu.edu}

  \vskip 0.3in
]

\printAffiliationsAndNotice{}  %

\begin{abstract}
Tuning algorithms such as stochastic gradient descent (SGD) and 
stochastic gradient Langevin dynamics (SGLD) for approximate sampling and uncertainty quantification 
remains challenging, particularly in the practically relevant settings 
when the batch size is large or the model is misspecified. 
Existing theory that provides tuning guidance relies on continuous-time limits or strong statistical assumptions, 
which can become quantitatively inaccurate in these regimes.
We address these shortcomings by proposing new discrete-time approximations to SG(L)D with and without momentum, 
which enables accurate predictions of the stationary covariance, iterate average covariance, and integrated autocorrelation time.
Moreover, we prove quantitative, non-asymptotic error bounds showing that these estimates are sufficiently accurate for practical tuning and uncertainty quantification.
Numerical experiments demonstrate that our theory yields improved tuning guidance across a range of models and data-generating distributions where
existing approaches fail, including when using the $\beta$-divergence rather than log-loss to obtain statistically robust inferences.
\end{abstract}

\section{Introduction}

Stochastic gradient–based methods have become the default tool for large-sample optimization in machine learning. 
Algorithms such as stochastic gradient descent (SGD) and its variants dominate modern practice because subsampling dramatically reduces per-iteration computational cost while having strong empirical performance and favorable generalization properties \citep{bottou2010large,hardt2016train,Goodfellow:2016:DL-book}. 

From a Bayesian perspective, subsampling-based Markov chain Monte Carlo (MCMC) methods seem to offer
an analogous path toward scalable sampling and uncertainty quantification (UQ). 
In particular, stochastic gradient MCMC (SG-MCMC) algorithms such as stochastic gradient Langevin dynamics (SGLD) replace full-data likelihood gradients with unbiased minibatch estimates, promising posterior sampling at a computational cost comparable to SGD \citep{welling2011bayesian,li2016preconditioned,raginsky2017non,brosse2018sgld,Nemeth:2021:SGLD}. 
In practice, however, SG-MCMC methods are notoriously difficult to tune because
the step size, batch size, and temperature parameters must be carefully chosen to control discretization bias and mixing behavior while simultaneously providing accurate UQ 
\citep{Nemeth:2021:SGLD,coullon2023,negrea2022statistical,rajpal2025adaptive,kim2024learningtoexplore, Mauri2024Robust, alexos2022structured, paulin2025sampling, akyildiz2024nonasymptotic}.
These challenges are exacerbated when the statistical model is misspecified, a setting in which standard Bayesian posteriors are no longer well-calibrated.
The same calibration issue applies when using a generalized Bayesian loss, whether the model is correctly specified or not \citep{bissiri2016general,jewson2018principles}.

Recent work has begun to address these challenges by explicitly combining algorithmic and statistical asymptotic perspectives. 
For example, \citet{mandt2017stochastic} adopted a heuristic perspective that was motivated by two lines of work.
The first considers scaling limits in stochastic approximations and show that, after appropriate rescaling of space and time, the iterates jointly converge to a continuous-time Ornstein--Uhlenbeck process \citep{kushner1981asymptotic, pflug1986stochastic, walk1977invariance, Kushner:1993, kushner2003stochastic}.
The second concerns the asymptotics of the Bayesian posterior, known as Bernstein--von Mises (or Bayesian Central Limit) theorems \citep{kleijn2012bernstein, van2000asymptotic}.

More recently, \citet{negrea2022statistical,wang2025quantitative} formalize and extend the heuristic arguments of \citet{mandt2017stochastic} by analyzing stochastic gradient algorithms through joint limits in which both the dataset size and algorithm parameters (e.g., step size and batch size) scale together. 
Further, \citet{wang2026lvm} extend the results of \citet{negrea2022statistical} to models with local latent variables.
These results, which characterize the limiting stochastic process of the iterate sample paths, make it possible to not only determine the limiting stationary distribution (which is important for UQ) but also the mixing time and iterate average distribution, which determine the algorithm's computational efficiency and the accuracy of posterior expectation estimates.
Hence, these results are able to provide precise tuning advice that maximizes computational efficiency while targeting the desired form of UQ such as frequentist coverage \citep{white1982maximum}, Bayesian model uncertainty \citep{kleijn2012bernstein}, or both \citep{huggins2024reproducible}.

A major limitation of these results, however, is that they rely on taking continuous-time stochastic differential equation (SDE) limits, which approximate discrete-time algorithms only in the vanishing step-size regime \citep{wang2025quantitative,li2019stochastic}.
These limiting approximations become quantitatively inaccurate precisely in the large batch-size regimes most relevant to practice.
The problem is that using a large batch size requires using a relatively large step size 
\citep{goyal2017accurate,negrea2022statistical}, 
so continuous-time approximations can substantially mischaracterize stationary covariance structure, which can result in inaccurate UQ. 

\Cref{Fig: linear_regression_misspecify_dependent_noise_batch_size} illustrates how the these issues can arise even in simple misspecified linear models. 
In this example, as the batch size increases, the accuracy of the tuning rules derived from SDE limits decreases rapidly, leading to the stationary covariance failing to match the sandwich covariance $\mathcal{S}_\star$ \citep{white1982maximum}. Such failures persist even with increasing data size, highlighting a fundamental limitation of continuous-time approximations for guiding practical tuning decisions \citep{wang2025quantitative}.

Recent work has used discrete-time approximations to stochastic gradient algorithms  that remain valid at large batch sizes and/or large step sizes \citep{dieuleveut2020bridging,liu2021noise,ziyin2021strength}. 
While promising, existing results either assume a constant noise covariance,
apply only to linear models, or to not account for model misspecification.
Moreover, most approximations lack rigorous non-asymptotic error guarantees; and none provide estimates for the  mixing time or iterate-average distribution. 
\Cref{Fig: linear_regression_misspecify_dependent_noise_batch_size} illustrates how, as a result, they can fall short of providing reliable guidance for uncertainty quantification -- in this case, due to model misspecification.

\begin{figure}[t]
    \centering
    \includegraphics[width=0.45\textwidth]{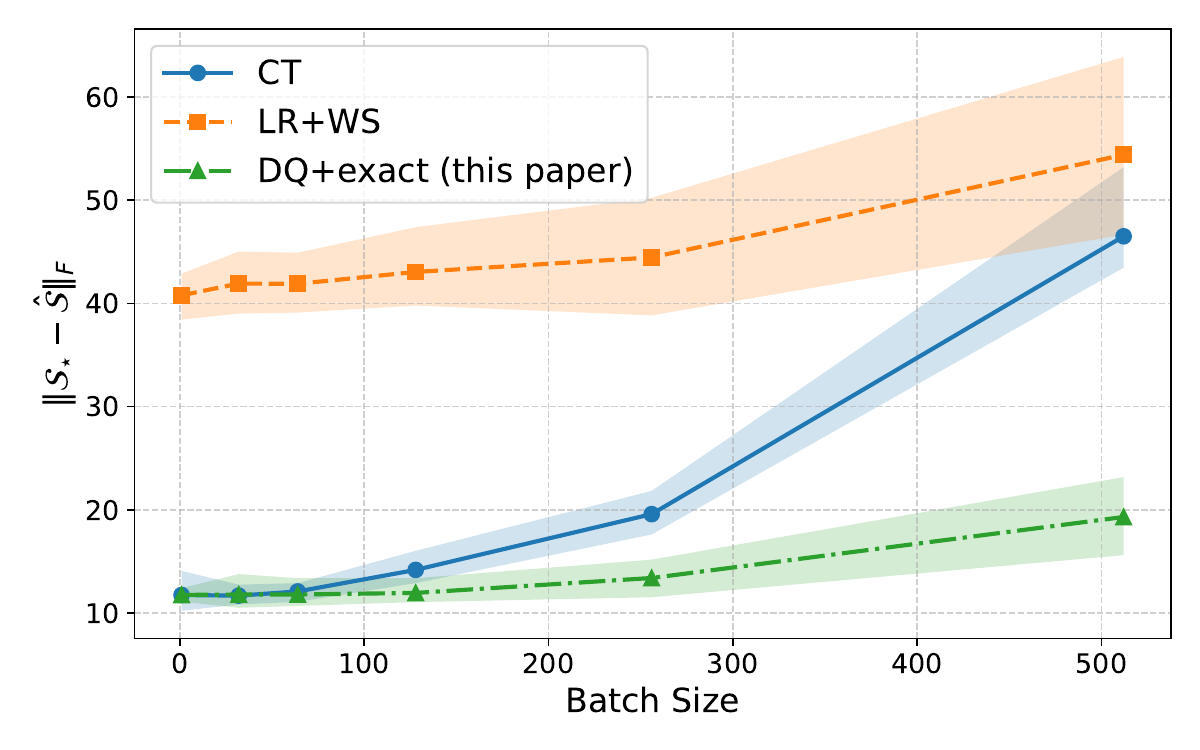} 
    \caption{Misspecified linear regression with heteroskedastic noise. Data are generated according to $y_{n} \dist \Norm(x_{n}^{\top}\theta_{\star}, 1+\|x_{i}\|_{2}^{2})$, 
    where $\theta_{\star} \dist \Norm(0, I_{D})$ is fixed and $x_{n} \distiid \Norm(0, I_{D})$. 
    A linear model is fitted using constant-step-size SGD.
    $\mathcal{S}_{\star} = \mathcal{J}_{\star}^{-1} \mathcal{I}_{\star} \mathcal{J}_{\star}^{-1}$: sandwich covariance; $\hat{\mathcal{S}}$: covariance obtained under step-size tuning rules derived from different theories.}
    \vspace{-1em}
    \label{Fig: linear_regression_misspecify_dependent_noise_batch_size}
\end{figure}

\newcommand{\redx}{\textcolor{red}{\ding{55}}}
\newcommand{\bluecheck}{\textcolor{blue}{$\checkmark$}}
\begin{table*}[t]
\caption{Comparison of approximations used to tune SG(L)D for sampling. References are to the works most directly relevant to tuning. 
\textbf{Large batch:} Is the approach accurate for large batch sizes?
\textbf{Non-const.\ noise:} Does the approach account for non-constant stochastic gradient noise?
\textbf{General model/loss:} Does the approach account for model misspecification or the use of a generalized loss? 
\textbf{Mixing:} Does the approach provide mixing time and iterate average covariance estimates?
\textbf{Bounds:} Are quantitative error bounds available?}
\centering
\begin{tabular}{@{}p{6.2cm}ccccc@{}}
\toprule
Approach                                                                                          & Large batch & Non-const.\ noise & General model/loss & Mixing & Bounds \\ \midrule
\textbf{Continuous-time} \citep{mandt2017stochastic,negrea2022statistical,wang2025quantitative}                               & \redx             & \redx               & \bluecheck           & \bluecheck   & \bluecheck    \\
\textbf{Discrete quadratic + constant noise}  \citep{dieuleveut2020bridging,liu2021noise} & \bluecheck        & \redx          & \bluecheck                & \redx        & \bluecheck         \\
\textbf{Linear regression + well-specified} \citep{ziyin2021strength} & \bluecheck        & \bluecheck          & \redx                & \redx        & \redx         \\
\textbf{Discrete quadratic + exact noise} ({this work})                                                  & \bluecheck        & \bluecheck          & \bluecheck           & \bluecheck   & \bluecheck    \\ \bottomrule
\end{tabular}
\label{tbl:summary}
\end{table*}

In this work, we address these limitations by developing a discrete-time theoretical framework for SGD and SGLD that remains accurate at large batch sizes, and under model misspecification. 
\cref{tbl:summary} compares our approach to alternatives. 
Our contributions are as follows: 
\vspace{-1em}
\begin{enumerate}
\item \textbf{(minor)} We introduce a \emph{proxy algorithm framework} that clarifies the differences and limitations of existing approaches, and thereby helps identify where further theory is needed. (\cref{sec:proxy-algs})
\item \textbf{(major)} We derive a \emph{new discrete-time approximation for SGD and SGLD} (with and without momentum) that remains accurate for large batch sizes and misspecified models. (\cref{section:Uncertainty Quantification for Generalized Linear Models})
\item \textbf{(major)} We provide \emph{quantitative, non-asymptotic error analyses} demonstrating that the resulting stationary covariance estimates are sufficiently accurate for practical tuning \emph{for the purpose of sampling and uncertainty quantification}. (\cref{sec:wasserstein-results})
\item \textbf{(major)} We use our results to propose a practical, tuning-free 
procedure for scalable uncertainty quantification (\cref{alg:uq_tuning}). 
Through numerical experiments, we show that our theory provides improved tuning guidance for a different models, batch size regimes, and loss functions. 
(\cref{sec:more-applications})
\item  \textbf{(minor)} Finally, while our focus in the paper is on uncertainty quantification and sampling, our results also shed light on the training dynamics and generalization behavior of SGD and its use for frequentist inference  \citep{jantre2024learning, hwang2022uncertainty, chang2017active, lyle2020bayesian, mandt2017stochastic, zhu2019anisotropic, lewkowycz2020large, keskar2017large, hoffer2017train, mori2020improved}.
For completeness, we illustrate some of these directions, which may be of interest to the wider ML community, through some preliminary experiments (Appendix~\ref{sec:applications}).
\end{enumerate}

\section{Background}

\subsection{Setting}

Let $\{x_{n}\}_{n=1}^{N}$ denote the observed data with $x_{n} \in \obsSpace$ .
For parameter $\theta \in \reals^{D}$, assume an observation-level differentiable loss or negative log-likelihood
$\ell : \obsSpace \times \reals^{D} \to \reals$, and regularizer  $\reg : \reals^{D} \to \reals$, which in the sampling setting we should interpret as a negative log prior $-\log \pi_0(\theta)$ (up to an additive constant). 
Together, these lead to the negative potential (or loss)
\begin{align}
\textstyle \loss(\theta) \defas N^{-1} \sum_{n=1}^{N} \ell(x_{n}, \theta) + N^{-1}\reg(\theta). 
\label{eq:loss function}
\end{align}
Define the stochastic gradient 
\[
\textstyle G_{t}(\theta) \defas B^{-1} \sum_{n \in S_{t}} \nabla \ell(x_{n}, \theta) + N^{-1} \nabla \reg(\theta),
\]
where $S_{t} = \{I_{t1}, I_{t2}, \dots, I_{tB}\}$ is a set of $B$ independent random integers sampled uniformly from $\{1, \dots, N\}$ 
either with or without replacement. 
\emph{Stochastic gradient Langevin dynamics} \citep[SGLD; ][]{welling2011bayesian} is a Markov chain Monte Carlo (MCMC) algorithm
with the single-step update equation 
\begin{align}
	\theta_{t} = \theta_{t-1} - \Lambda\,G_{t}(\theta_{t-1}) + \sqrt{2\beta^{-1} \Lambda}\,\xi_{t-1},
	\label{eq: general SGLD update rule}
\end{align}
where $\Lambda \in \reals^{D\times D}$ is a positive definite step size matrix, $\beta \in (0, \infty]$ is the inverse temperature (canonically set to $\beta = N$), and $\xi_{t-1} \distiid \Norm(0, I)$.
SGLD is the prototypical example of a \emph{subsampling MCMC} algorithm, 
variants of which have been applied for learning a wide variety of large-sample models \citep{ICML2012Ahn_782, Nemeth:2021:SGLD, Aicher2025SGMCMC, kim2024learningtoexplore,rajpal2025adaptive, Mauri2024Robust, alexos2022structured, paulin2025sampling}.
If $\beta = \infty$ (with $1/\infty \defas 0$), then SGLD reduces to SGD.
Setting $\Lambda = \lambda I_{D}$ for some $\lambda > 0$ results in the usual formulation of SG(L)D with fixed step size $\lambda$.

\begin{remark}
    We focus on the fixed step size case in this work.
    While diminishing step sizes guarantee asymptotic exactness by driving stochastic-gradient noise and discretization error to zero, using a fixed step size usually leads to substantially faster convergence \citep{dieuleveut2020bridging,vollmer2016exploration,teh2016consistency, merad2025convergence} and empirically leads to better generalization by discouraging convergence to sharp minima \citep{keskar2017large}, instead biasing iterates toward flatter solutions that tend to have larger posterior mass \citep{mackay1992practical,rissanen1983universal}.
\end{remark}

\subsection{Uncertainty Quantification}

Both SGD and SGLD have been used for quantifying uncertainty about model parameters \citep{welling2011bayesian, ICML2012Ahn_782, Nemeth:2021:SGLD, mandt2017stochastic}.
Assuming observations are \iid\ from an unknown distribution $P_{\star}$, then the optimal parameter is given 
by $\theta_{\star} \defas \argmin_{\theta} \E\left[ \ell(X, \theta)\right]$, where $X \dist P_{\star}$. 
In the Bayesian setting, the Bernstein-von Mises theorem states that the posterior is approximately $\Norm(\MLE, \mathcal{J}_{\star}^{-1}/N)$, 
where $\mathcal{J}_{\star} \defas \E\left[ \nabla_\theta^{2} \ell\left(X, \theta_{\star} \right) \right]$ \citep{kleijn2012bernstein}.
Thus, one possible goal when using SG(L)D is to obtain samples with a distribution that is approximately equal to $\Norm(\MLE, \mathcal{J}_{\star}^{-1}/N)$.
However, the sampling distribution of $\MLE$ is asymptotically normal with mean $\theta_{\star}$ and covariance equal to $\mathcal{J}_{\star}^{-1} \mathcal{I}_{\star} \mathcal{J}_{\star}^{-1}/N$,
where $\mathcal{I}_{\star} \defas \E[\nabla_\theta \ell\left(X, Y, \theta_{\star} \right)\nabla_\theta \ell\left(X, Y, \theta_{\star}\right)^{\top}]$ \citep{white1982maximum}.
The matrix $\mathcal{J}_{\star}^{-1} \mathcal{I}_{\star} \mathcal{J}_{\star}^{-1}$ is known as the ``sandwich'' covariance matrix, and it suggests that for proper uncertainty quantification
we want the stationary SG(L)D distribution to be approximately $\Norm (\MLE, \mathcal{J}_{\star}^{-1} \mathcal{I}_{\star} \mathcal{J}_{\star}^{-1}/N )$.
When the model is correctly specified, $\mathcal{I}_{\star} = \mathcal{J}_{\star}$, so the sandwich covariance is equal to $\mathcal{J}_{\star}^{-1}$ and the Bayesian
posterior (and the Laplace approximation) provides correct uncertainty quantification.
In the case of a generalized loss, there is no notion of well-specification, and so the ``model covariance'' $\mathcal{J}_{\star}^{-1}$ is not a coherent target for uncertainty quantification \citep{bissiri2016general,jewson2018principles}.
Therefore, tuning SG(L)D to satisfy $\Sigma_{\theta} \approx \mathcal{J}_{\star}^{-1} \mathcal{I}_{\star} \mathcal{J}_{\star}^{-1}/N$ will capture the sampling uncertainty in both the 
model-based and generalized loss settings.

\section{Proxy Algorithms} \label{sec:proxy-algs}

Given the extensive use of SGD and SGLD, both algorithms have been studied from a wide variety of perspectives.
Many such analyses can be viewed as proposing a \emph{proxy algorithm}: an alternative stochastic process
that is ``close'' to the actual algorithm of interest. 
The idea is to characterize important properties of the proxy algorithm, then argue either heuristically or rigorously that 
these properties can be transferred back to apply to the original (exact) algorithm. 
We will focus on proxy algorithms that, at least implicitly, require that the loss is well-approximated by a quadratic function:
\begin{align}
	\textstyle\loss(\theta_{t}) \approx \tilde{\loss}(\theta_{t}) \defas \frac{1}{2} \rbra[\big]{\theta_{t}-\MLEN}^{\top} \hess \rbra[\big]{\theta_{t}-\MLEN} + \mathrm{const},
\end{align}
where $\hess \defas \grad^{2} \loss(\MLE)$ is the Hessian of the loss (evaluated at $\MLE$). 
While such a condition may seem quite limiting, it turns out to be 
reasonable in many interesting settings. 

\paragraph{Continuous-time proxies.}
Perhaps the most popular proxy approach is to replace discrete dynamics of the iterative algorithm by a continuous-time stochastic process \citep{mandt2017stochastic, zhu2019anisotropic,negrea2022statistical}.
Let $\widehat{C} = \cov(G_{1}(\MLE))$ denote the gradient noise covariance at the minimizer and let $W_{t}$ be a $D$-dimensional Brownian motion.
Focusing on the case of SGD for clarity, the Ornstein--Uhlenbeck process $(\vartheta_t)_{t \ge 0}$ defined by the stochastic differential equation (SDE) 
\begin{align}
	\mathrm{d} \vartheta_{t} = -\Lambda \hess \vartheta_{t} \mathrm{d} t + \Lambda \widehat{C}^{1/2} \mathrm{d} W_{t},
	\label{eq: OU process}
\end{align}
provides a proxy to the discrete-time dynamics after appropriate rescaling and discretization.
\footnote{\citet{li2017stochastic, li2019stochastic} propose using stochastic modified equations (SMEs) to approximate SGD and perform error analysis. 
However, SMEs serve as a close approximation to SGD only for small learning rates, making it challenging to justify this approach for non-vanishing values of $\lambda$ \citep{li2017stochastic}. Moreover, in most cases, SMEs are not conducive to exact analysis; this leads to, for example, \citet[Section 5.1]{li2019stochastic} focusing on cases with an explicit solution that match \cref{eq: OU process}.}
This approach can be made rigorous via both numerical analysis and statistical (large-sample) perspectives \citep{Kushner:1993, kushner1981asymptotic, kushner2003stochastic, negrea2022statistical, wang2025quantitative}.
Similar types of arguments have also been widely used to study MCMC algorithms that do not use subsampling \citep{roberts1998optimal, dalalyan2017theoretical, roberts2001optimal,Wibisono2018sampling}.

The continuous-time approach is appealing because $(\vartheta_{t})_{t \ge 0}$ is a Gaussian process, so its properties are straightforward to analyze \citep{mandt2017stochastic, negrea2022statistical, Kushner:1993}.
For example, if the process has stationary distribution $\pi_{\vartheta}$, the covariance matrix of the stationary distribution $\Sigma_{\vartheta} \defas \cov(\pi_{\vartheta})$ 
must satisfy $\Sigma_{\vartheta} \hess + \hess \Sigma_{\vartheta} = \Lambda \widehat{C}$ \citep{gardiner1985handbook}.
In particular, setting \vphantom{\cref{eq: OU learning rate}}
\begin{align}
    \Lambda = (\Sigma \hess + \hess\Sigma)\widehat{C}^{-1}
     \label{eq: OU learning rate}
\end{align}
results in a stationary covariance of $\Sigma_{\vartheta} = \Sigma$.\footnote{Or, when we are interested in characterizing the stationary covariance, if $\Sigma_{\vartheta}$ and $\hess$ commute, then 
that $\Sigma_{\vartheta} = \frac{1}{2}\Lambda \widehat{C} \hess^{-1}$.}
Furthermore, \citet{negrea2022statistical} show that 
the asymptotic mixing time is heuristically equal to $2 /\lambda_{\min}(\Lambda \hess)$ iterations, where $\lambda_{\min}(A)$ denotes the minimum 
eigenvalue of matrix $A$.
This result suggests that, to optimize mixing time, set $\Lambda \propto \hess^{-1}$. 

\paragraph{Discrete-time proxies.}

The continuous-time proxy approach requires the step size matrix $\Lambda$ to be sufficiently small that
\emph{(i)} the continuous-time dynamics (driven by Gaussian noise) is a good approximation to the discrete-time dynamics and 
\emph{(ii)} the gradient noise is approximately constant (that is, $G_{t}(\theta_{t-1}) \approx G_{t}(\MLE)$ for all $t = 1,\dots, T$).
However, in practice it is often desirable to use a 
relatively large batch size (e.g, 1\%--10\% of the data), in which case following the guidance of \citet{negrea2022statistical} requires the use of a 
relatively large $\Lambda$ -- exactly the regime in which the continuous-time theory
often breaks down, leading to inaccurate predictions about real algorithm behavior \citep{liu2021noise, ziyin2021strength}. 
The importance of capturing the location-dependence of noise has been widely observed \citep{simsekli2019tail, simsekli2020hausdorff, hodgkinson2021multiplicative, meng2020dynamic, mori2022power, ziyin2021strength}.

A number of papers aim to overcome these limitations by using the discrete-time proxy algorithm
\begin{align}
	\label{eq:discrete-time-model-general}
	\textstyle \psi_{t} = \psi_{t-1} - \frac{\Lambda}{B}\sum_{n\in S_{t}} \hess_{n} (\psi_{t-1} - \MLE),
\end{align}
where $\hess_{n} \defas  \grad^{2} \ell(x_{n}, \MLE)$.
Assuming it exists, let $\pi_{\psi}$ denote the stationary distribution of the proxy algorithm given in \cref{eq:discrete-time-model-general},
let $\Sigma_{\psi} \defas \cov(\pi_{\psi})$, and for $\psi_{\infty} \dist \pi_{\psi}$, let $\overline{C}_{\psi} \defas \E[\cov\{G_{1}(\psi_{\infty})\}]$ denote the expected covariance of the gradient noise. 
\citet{liu2021noise} show that  the stationary covariance $\Sigma_{\psi}$ of
discrete-time update described in \cref{eq:discrete-time-model-general}  satisfies 
\begin{align}
	\Lambda \hess \Sigma_{\psi} + \Sigma_{\psi}  \hess \Lambda = \Lambda \left(\overline{C}_{\psi} + \hess \Sigma_{\psi} \hess \right)\Lambda.
	\label{eq: SGD stationary covariance}
\end{align}
\citet{dieuleveut2020bridging} provides exact discrete-time analyses of constant-step stochastic gradient descent, treating SGD as a time-homogeneous Markov chain rather than as a discretization of a continuous-time diffusion. 
In the quadratic setting, they also show that SGD converges to a stationary distribution whose covariance satisfies \cref{eq: SGD stationary covariance}, as also given by \citet{liu2021noise}.
Notably, the higher-order covariance terms $\Lambda \overline{C}_{\psi} \Lambda$ -- which capture finite learning-rate effects intrinsic to the discrete-time dynamics -- is missing from diffusion-based SDE approximations.

A key challenge when using \cref{eq: SGD stationary covariance} is that it only provides an \emph{implicit} characterization of $\Sigma_{\psi}$ because the average noise covariance $\overline{C}_{\psi}$ 
also depends on the stationary distribution $\pi_{\psi}$. 
Hence, $\overline{C}_{\psi}$ must either be approximated or, in special cases, computed exactly. 
An important special case is the linear regression model, where $x_{n} = (z_{n}, y_{n}) \in \reals^{D} \times \reals$ %
and the observation-level loss is $\ell(x_{n}, \theta) = \frac{1}{2\sigma^{2}}(y_{n} - \theta^{\top}z_{n})^{2}$.  %
In a follow-up to \citet{liu2021noise},  \citet{ziyin2021strength} show that, assuming $z_{n} \dist \Norm(0, A)$ and the model is well-specified 
(i.e., $y_{n} \dist \Norm(\theta_{\star}^{\top}z_{n}, \sigma^{2})$ for some $\theta_{\star} \in \reals^{D}$), for large $N$, 
\begin{align}
	\overline{C}_{\psi} %
	\approx B^{-1} \left( A \Sigma_{\psi} A + \Tr \left[ A \Sigma_{\psi} \right]A + \sigma^{2} A \right).
	\label{eq: sgd noise in well-specified linear model}
\end{align}
Using this approximate expected covariance for SGD noise, %
\citet{ziyin2021strength} are able to show, for example, better test loss estimation, the benefits of negative regularization, the role of overparameterization
in the steady-state dynamics of SGD, and power-law tail behavior of SGD noise.

Hence, using discrete-time proxies rather than continuous-time ones can 
lead to more precise tuning advice and new insights. 
Nevertheless, as summarized in \cref{tbl:summary}, existing approaches are not yet sufficiently reliable for practical use: some leave the noise covariance implicit \citep{dieuleveut2020bridging}, others rely on heuristic approximations to the noise covariance \citep{liu2021noise}, and others focus on restricted settings, such as well-specified models with $N \gg D$ \citep{ziyin2021strength}.
In addition, they do not characterize the mixing time or iterate average error. 
Our results, which are presented in the next section, address all of these limitations, as described in the last row of \cref{tbl:summary} and 
illustrated in \cref{Fig: linear_regression_misspecify_dependent_noise_batch_size}.

\section{A New Proxy Algorithm for Analyzing SG(L)D}
\label{section:Uncertainty Quantification for Generalized Linear Models}

Our approach to creating an improved proxy algorithm is to apply a second-order Taylor approximation to each loss term $\ell_{n}(\theta) \defas \ell(x_{n}, \theta)$:
\begin{align}
    \begin{split}
	\tilde\ell_{n}(\theta)  
    &\defas \ell_{n}(\MLE) + \nabla \ell_{n}^{\top}( \MLE) (\theta-\MLE) \\
    &\phantom{\defas~} + \frac{1}{2}(\theta-\MLE)^{\top}\grad^{2} \ell_{n}(\MLE) (\theta-\MLE).
	\label{eq: proxy algorithm framework for general loss}
    \end{split}
\end{align}
We apply SG(L)D (with or without momentum) to the approximation, $\tilde\loss(\theta) \defas N^{-1} \sum_{n=1}^{N} \tilde\ell_{n}(\theta) + N^{-1}\reg(\theta)$. 
Letting $\mathcal{J}_{n} \defas \grad^{2} \ell_{n}(\MLE)$ and using \cref{eq: proxy algorithm framework for general loss,eq: general SGLD update rule}, the update equation for our proxy algorithm is 
\begin{align}
    \psi_t
&  = \psi_{t-1}
- \Lambda \Bigl[
G_t(\MLE)
+  \nabla G_t(\MLE)(\psi_{t-1} - \MLE)
\Bigr] \\
& \quad + \sqrt{2\beta^{-1}\Lambda}\,\xi_{t-1}.
	\label{eq: proxy discrete-time update for general loss}
\end{align}

Assuming the iterates $(\psi_{t})_{t \ge 0}$ have a well-defined stationary distribution, the stationary covariance $\Sigma_{\psi}$ provides an approximation $\widehat{\Sigma}_{\theta} \defas \Sigma_{\psi}$ to $\Sigma_{\theta}$.
The quadratic form of \cref{eq: proxy algorithm framework for general loss} facilitates analyses that allow us to address  limitations of previous work.
First, the quadratic loss results is the linear structure of the SG(L)D update given in \cref{eq: proxy discrete-time update for general loss}, which makes it amenable to direct analysis. 
Building on the techniques of previous work \citep{liu2021noise, ziyin2021strength}, in \cref{sec:stationary-analysis} we derive an \emph{exact}, solvable relationship between $\Sigma_{\psi}$ and $\Lambda$.
\emph{Thus, unlike previous results, we do not require any additional assumptions or approximations}.

In addition, the use of a Taylor series approximation for the observation-level losses lends itself to rigorous error analysis. 
Specifically, we are able to bound the Wasserstein distance between the distributions of $\psi_{t}$ and $\theta_{t}$ in \cref{sec:wasserstein-results}. 
Using this result we obtain relative error bounds on the marginal standard deviation and covariance matrix estimates under standard assumptions, which 
hold for logistic regression and, assuming a bounded parameter space, for Poisson and gamma regression as well \citep[see, e.g.,][]{brosse2018sgld, bach2011non-asymptotic, toulis2014statistical}:
\begin{enumerate}[label=(\Alph*)]
	\item \label{assump:statistic-convex} The observation-level losses $\ell_{1},\dots,\ell_{N}$ are convex. 
	\item \label{assump:statistic-smooth} For each $n = 1,\dots, N$, for finite positive $L_{n} $ and $M_{n}$, the loss $\ell_{n}$ is $L_{n}$-smooth and 
	satisfies $\sup_{\theta}\sum_{d=1}^{D} \norm{\grad^{2} (\partial_{d} \ell_{n}(\theta))}^{2}  \le M_{n}^{2}$. 
	\item \label{assump:loss} For some $\mu > 0$, the loss $\loss$ is $\mu$-strongly convex. 
\end{enumerate}

\begin{theorem} \label{thm:covariance-error-bounds}
If Assumptions \ref{assump:statistic-convex}--\ref{assump:loss} hold and $\Lambda = \lambda I_{D}$ for some $\lambda \in (0, 1/(2L))$, then
there exist constants $C_{v}$ and $C_{s}$ independent of $\lambda$ such that 
\begin{align}
	\label{eq:cov-relative-error}
	\norm{\Sigma_{\theta} - \Sigma_{\psi}}/\norm{\Sigma_{\theta}} &\le C_{v} \lambda^{1/2} \\
\intertext{and, for $d=1,\dots,D$,}
    \label{eq:stdev-relative-error}
	|\sigma_{\theta,d} - \sigma_{\psi,d}|/\sigma_{\theta,d} &\le C_{s} \lambda^{1/2}.
\end{align}
\end{theorem}
Hence, it follows from our results that the approximation $\widehat{\Sigma}_{\theta} \defas \Sigma_{\psi}$
is close enough to $\Sigma_{\theta}$ to provide a practically useful estimate. 
    
\subsection{Stationary Analysis} \label{sec:stationary-analysis}
To prove our main result \cref{thm:covariance-error-bounds}, we first obtain an exact relationship between the learning rate matrix $\Lambda$, the stationary covariance $\Sigma_{\psi}$, and the average noise $\overline{C}_{\psi}$.

\begin{proposition} \label{thm:SGLD-stationary-covariance} 
	Assuming the iterates $(\psi_{t})_{t \ge 0}$ have a well-defined stationary distribution, the stationary covariance $\Sigma_{\psi}$ satisfies 
	\begin{align}
		\Lambda \hess \Sigma_{\psi} + \Sigma_{\psi}  \hess \Lambda = \Lambda \big(\overline{C}_{\psi} + \hess \Sigma_{\psi} \hess \big)\Lambda + 2\beta^{-1}\Lambda.
		\label{eq: SGLD stationary covariance}
	\end{align}
\end{proposition}

It follows from \cref{eq: SGLD stationary covariance} that to obtain a solvable relationship between $\Lambda$ and $\Sigma_{\psi}$ using \cref{thm:SGLD-stationary-covariance}, we must compute the expected covariance $\overline{C}_{\psi}$. 
Such calculation is feasible when using $\Norm(0, \Gamma^{-1})$ with $\Gamma \in \reals^{D \times D}$ positive-definite, as a prior for $\theta$ 
-- that is, using $\reg(\theta) = \frac{1}{2} \theta^{\top}\Gamma \theta$.
\begin{theorem} \label{theorem: general case covariance matrix of noise}
	For the proxy algorithm \cref{eq: proxy discrete-time update for general loss}, if $\reg(\theta) = \frac{1}{2} \theta^{\top}\Gamma \theta^{\top}$ and the mini-batches
	are sampled with replacement, then 
    \begin{align}
		\label{eq: covariance of stationary noise}
		\begin{aligned}
            \overline{C}_{\psi} = \frac{1}{B}\left(\mathcal{I}-\frac{\norm{\Gamma \widehat{\theta}}^2}{N^2} + \frac{1}{N} \sum_{n=1}^N \mathcal{J}_n \Sigma_\psi \mathcal{J}_n-\mathcal{J} \Sigma_\psi \mathcal{J}\right),
		\end{aligned} 
	\end{align}
	where $\mathcal{I} \defas  \frac{1}{N} \sum_{n=1}^{N} \nabla\ell_{n}\bigl(\MLE\bigr) \nabla\ell_{n}\bigl(\MLE\bigr)^{\top}$.
	If the mini-batches are sampled without replacement, the same result holds but with the right-hand side multiplied by $(N-B)/(N-1)$. 
\end{theorem}
Plugging \cref{eq: covariance of stationary noise} into \cref{eq: SGLD stationary covariance} provides an exact relationship between $\Lambda$ and $\Sigma_{\psi}$. 
Hence, given a fixed learning rate matrix $\Lambda$ (or a scalar learning rate $\lambda$), we can, in principle, compute the stationary covariance $\Sigma_{\psi}$ as an estimate for $\Sigma_{\theta}$.

Our final result in this section improves upon the heuristic mixing time estimate of \citet{negrea2022statistical} (see Appendix \ref{sec: Discussions on Mixing speed} for more on mixing time). 
Unlike the stationary covariance case, our result is identical except for an
additive $-1$. 

\begin{proposition}
\label{prop: mixing time}
Consider the proxy update in \cref{eq: proxy discrete-time update for general loss} and suppose that
$0 < \lambda < 2/\mu_{\max}(\hat{H})$, where $\mu_{\max}(A)$ and $\mu_{\min}(A)$ denote the largest and smallest eigenvalues of a matrix $A$, respectively.
Under $L$-smoothness, it simplifies to $0 < \lambda < 2/L$ , which is consistent with step-size condition used in \citet{dieuleveut2020bridging}.
Under this condition, the resulting SG(L)D Markov chain admits a unique stationary distribution $\pi_\theta$.
For each $v \in \reals^D$, define the projection $f_v(\theta) \defas v^\top\theta$ and let
\[
\rho_{k,v} \defas \mathrm{Corr}_{\pi_\theta}\!\bigl(v^\top\theta_{0}, v^\top\theta_{k}\bigr),
\]
and 
\[
\textstyle  \tau_{\mathrm{int}}(f_{v}) \defas 1 + 2\sum_{t=1}^{\infty}\rho_{k, v}.
\]
Then the worst-case integrated autocorrelation time $\tau \defas \sup_v \tau_{\mathrm{int}}(f_v)$ is equal to ${2}/{\mu_{\min}(\Lambda \hess)} - 1$ iterations. 
\end{proposition}

\subsection{Error Analysis} \label{sec:wasserstein-results}

We assess the accuracy of our proxy algorithm by bounding %
the 2-Wasserstein distance between 
the distributions of $\theta_{t}$ and $\psi_{t}$. 
The 2-Wasserstein distance between distributions $\pi$ and $\tilde\pi$ is given by
\[
W_{2}(\pi, \tilde\pi) = \inf \E(\norm{\theta - \tilde\theta}^{2})^{1/2},
\]
where the infimum is over all joint distributions of $(\theta, \tilde\theta)$ such that $\theta \dist \pi$ and $\tilde\theta \dist \tilde\pi$. 
A small Wasserstein distance between distributions implies the covariance and marginal standard deviations are also close.
Let $\sigma_{\theta,d} \defas \Sigma_{\theta,dd}^{1/2}$ and $\sigma_{\psi,d} \defas \Sigma_{\psi,dd}^{1/2}$.
Then, by \citet[Theorem 3.4]{huggins2020validated}, $W_{2}(\pi_{\theta}, \pi_{\psi}) \leq \veps$ implies that
\begin{align}
	\begin{aligned}
		&\label{eq:general-stdev-and-cov-error-bounds}
		|\sigma_{\theta,d} - \sigma_{\psi,d}| \leq \veps~(d = 1,\dots,D) \\
		&\norm{\Sigma_{\theta} - \Sigma_{\psi}} \le 2 \veps (\norm{\Sigma_{\theta}}^{1/2} \wedge \norm{\Sigma_{\psi}}^{1/2} + \veps).
	\end{aligned}
\end{align}
Hence, bounding $W_{2}(\pi_{\theta}, \pi_{\psi})$ enables us to bound the error of the proxy stationary covariance $\Sigma_{\psi}$.

We first give a bound on the Wasserstein distance between the distributions of  $\theta_{t}$ and $\psi_{t}$, which we denote by, respectively, $\pi_{\theta,t}$ and $\pi_{\psi,t}$.
\begin{theorem} \label{thm:wasserstein-error-bound}
	If Assumptions \ref{assump:statistic-convex}--\ref{assump:loss} hold and $\Lambda = \lambda I_{D}$ for some $\lambda \in (0, 1/(2L))$,
	then, letting $\bar{\beta} \defas 1 - \lambda\mu \rbra*{1 - 2\lambda L}$, $\overline{M^p} \defas N^{-1} \sum_{n=1}^{N} M_{n}^{p}, p\in \{1,2\} $,
	and $C_{s} \defas  \E\rbra{\norm{\psi_{s} - \MLE}^{4}}$, for all $t = 1,2,\dots$,
	\begin{align} \label{eq:wasserstein-error-bound} 
	\lefteqn{W_{2}^{2}(\pi_{\theta,t}, \pi_{\psi,t})} \\
		&\le \bar{\beta}^{t}W_{2}^{2}(\theta_{0}, \psi_{0})  
		+ \lambda  \cbra*{\frac{\lambda  \overline{M^2}}{2} + \frac{\overline{M}^2}{4\mu}} \sum_{s=1}^{t}\bar{\beta}^{t-s}C_{s-1}.  
    \end{align}
\end{theorem}
\Cref{thm:wasserstein-error-bound} is quite general, and we conjecture it could be useful beyond our application to bounding the stationary covariance error. 
Typically we would expect to take $\psi_{0} = \theta_{0}$, in which case the first term on the righthand side of \cref{eq:wasserstein-error-bound} is zero. 
We note that \cref{thm:wasserstein-error-bound} is similar in spirit to the 2-Wasserstein bound provided by \citet{jin2024subsampling} for a continuous-time Langevin-based proxy algorithm
that uses Poissonized data subsampling; however, \citet{jin2024subsampling} do not use their proxy algorithm to estimate the stationary covariance of SG(L)D. 

Using \cref{eq:wasserstein-error-bound} to obtain an explicit quantitative bound requires upper-bounding the 4th moment of $\psi_{t}$, which we do in \cref{lem:fourth-moment-bound}. 
The following corollary gives our main error bound, which for simplicity we state for the case of SGD since the SGLD case is qualitatively identical. 
\begin{corollary} \label{cor:stationary-wasserstein-error-bound} 
	Under the same assumptions as \cref{thm:wasserstein-error-bound} and with $\beta = \infty$ (i.e., for the case of SGD), if $\lambda <\min\{B\hat{\mu}/(200L^2), 1/(4L)\}$, then 
	there exists an explicit constant $A$ given in \cref{eq:wasserstein-error-bound-constant-sgld} such that 
	$
	W_{2}(\pi_{\theta}, \pi_{\psi}) \le A\,{\lambda}/{B}.
	$
\end{corollary}

\begin{remark}[Dimension dependence]
Recall that $D$ is the parameter dimension.
For $I \dist \mathrm{Unif}(\{1,\dots,N\})$ independent,
define the single-sample stochastic gradient 
$g_I \defas \nabla \ell(x_I, \MLE)$.
Suppose that $g_I$ satisfies the mild scaling requirement
$\E[\|g_I\|^2] = O(D)$ and $\E[\|g_I\|^4] = O(D^2)$.
Such a condition holds, for example, for generalized linear model fit to data with sub-Gaussian covariate distribution \citep{Vershynin2018highd}.
Then $W_2(\pi_\theta,\pi_\psi)\le C\, D\,(\lambda/B+1/\beta)$ with $C$ independent of $D$. 
Thus, our proxy algorithm remains accurate in high dimensions provided that $\lambda/B+1/\beta \ll 1/(CD)$.
Since typically $\lambda = O(1/N)$ and 
$\beta = \infty$ (for SGD) or $\beta = N$ (for SGLD), it follows 
that in that case we require either 
\emph{(i)} $N \gg C D$ for SGD or SGLD, or
\emph{(ii)} $N B \gg C D$ for SGD.
Note that the latter case supports high-dimensional problems as long as the batch size is sufficiently large. 
See Appendix~\ref{sec:high-dim} for further details and discussion, including the sparse high-dimensional regime. 
\end{remark}

\subsection{SGLD with Momentum}

Our theoretical results extend to the case of SG(L)D with momentum. 
These extensions are tight, in the sense that we 
recover our non-momentum results as special cases. 
Due to space limitations, we defer details to Appendix~\ref{sec:momentum-theory}.

\begin{algorithm}[tb]
\caption{SG(L)D with Target Covariance Tuning. \\[-1ex]
\rule{\linewidth}{0.6pt}\\[-0.4ex]
\footnotesize\textbf{DQ+exact:} discrete quadratic + exact noise (this work). \\ 
\footnotesize \textbf{CT:} continuous time. \textbf{DQ+const:} discrete quadratic + constant noise. \textbf{LR+WS:} linear regression + well-specified.}
\label{alg:uq_tuning}
\begin{algorithmic}[1]
\REQUIRE Dataset $\{x_n\}_{n=1}^N$, tuning method choice, per-sample loss $\ell(\theta; x)$, offline subsample size $M$, inverse temperature $\beta$, batch size $B$, number of iteration $T$ 
\vspace{0.25em} \\
\textbf{Step 1: Offline UQ tuning} \\
\STATE Subsample $M$ observations $\{x_m'\}_{m=1}^M \subseteq \{x_n\}_{n=1}^N$
\STATE Use subsample to obtain MAP estimate $\hat\theta$ \\
 \emph{Estimate sandwich covariance at $\hat\theta$ for UQ:} \\
    \STATE \quad $\widehat{\mathcal{J}} \gets \frac{1}{M}\sum_{m=1}^M \nabla^2 \ell(\hat\theta;x_m')$ 
    \STATE \quad $\widehat{\mathcal{I}} \gets \frac{1}{M}\sum_{m=1}^M \nabla \ell(\hat\theta;x_m')\nabla \ell(\hat\theta;x_m')^\top$ 
    \STATE \quad $\widehat{\mathcal{S}} \gets \widehat{\mathcal{J}}^{-1}\,\widehat{\mathcal{I}}\,\widehat{\mathcal{J}}^{-1}$ 
\STATE \emph{Determine best $\Lambda$ using chosen tuning method:} \\
     \quad \textbf{DQ+exact:} solve eqs.~\eqref{eq: SGLD stationary covariance} and \eqref{eq: covariance of stationary noise} with $\Sigma_\psi =  \widehat{\mathcal{S}}$ \\
     \quad \textcolor{gray}{\textbf{CT:} use eq.\ \eqref{eq: OU learning rate} with $\Sigma =  \widehat{\mathcal{S}}$} \\
     \quad \textcolor{gray}{\textbf{DQ+const:} solve eq.\ \eqref{eq: SGD stationary covariance} with $\Sigma_\psi =  \widehat{\mathcal{S}}$ and $\overline{C}_{\psi} = \widehat{\mathcal{J}}$} \\
     \quad \textcolor{gray}{\textbf{LR+WS:} solve eqs.\ \eqref{eq: SGD stationary covariance} and \eqref{eq: sgd noise in well-specified linear model} with $\Sigma_\psi =  \widehat{\mathcal{S}}$}
\vspace{0.25em} \\
\textbf{Step 2: Preconditioned SG(L)D sampling}
\STATE Initialize $\theta_0 \gets \hat\theta$ (or any warm start).
\FOR{$t=1$ {\bfseries to} $T$}
     \STATE Sample minibatch $\mathcal{B}_t \subset \{1,\ldots,N\}$ with $|\mathcal{B}_t|=B$.
     \STATE Compute gradient $g_t \gets \frac{1}{B}\sum_{n\in\mathcal{B}_t}\nabla \ell(\theta_t;x_n)$
     \STATE Sample update $\theta_{t} \dist \Norm(\theta_{t-1} - \Lambda\, g_t , 2\,\beta^{-1}\Lambda)$
\ENDFOR
\STATE \textbf{return} $\{\theta_t\}_{t=1}^T$ and $\Lambda$.
\end{algorithmic}
\end{algorithm}

\begin{table*}[t]
\centering
\caption{
Results for linear regression experiments with simulated data.
Calibration error is the Kolmogorov--Smirnov distance to $\mathrm{Unif}(0,1)$.
Covariance error is $\|\mathcal{S}_{\star} - \hat{\mathcal{S}}\|_{F}/\|\mathcal{S}_{\star}\|_{F}$.
Within each metric row and loss block, for a fixed batch size $B$, bold indicates methods whose 95\% confidence intervals overlap with the confidence interval of the method with the lowest mean error.
Confidence intervals are computed over 30 independent runs and are reported in the full table in Appendix~D.
Sandwich Gauss is included as the target sandwich Gaussian reference, while NUTS and the exact posterior are included to illustrate the discrepancy between the posterior distribution and the sandwich target.
}
\label{tab:robust_linear_regression}

\begin{tabular}{lcccc|ccccc}
\toprule
& \multicolumn{4}{c|}{Log loss}
& \multicolumn{5}{c}{$\beta$-loss ($\beta = 1.5$)} \\
\cmidrule(lr){2-5}\cmidrule(lr){6-10}
$B$
& Posterior
& CT
& LR+WS
& DQ+exact
& NUTS
& Sandwich Gauss
& CT
& LR+WS
& DQ+exact \\
\midrule

\multicolumn{10}{l}{\textbf{Calibration error}} \\
\midrule
$16$
& 0.418
& \textbf{\cellci{0.171}{0.141}{0.206}}
& \cellci{0.529}{0.484}{0.580}
& \textbf{\cellci{0.169}{0.129}{0.214}}
& 0.195
& 0.156
& \textbf{\cellci{0.201}{0.162}{0.241}}
& \textbf{\cellci{0.178}{0.148}{0.207}}
& \textbf{\cellci{0.172}{0.133}{0.210}} \\

$\lfloor 0.1 \times N \rfloor$
& 0.418
& \textbf{\cellci{0.179}{0.154}{0.207}}
& \cellci{0.517}{0.469}{0.581}
& \textbf{\cellci{0.174}{0.139}{0.216}}
& 0.195
& 0.156
& \textbf{\cellci{0.196}{0.157}{0.228}}
& \textbf{\cellci{0.177}{0.142}{0.216}}
& \textbf{\cellci{0.190}{0.151}{0.231}} \\

\midrule
\multicolumn{10}{l}{\textbf{Covariance error}} \\
\midrule
$16$
& 0.943
& \textbf{\cellci{0.672}{0.629}{0.715}}
& \cellci{0.995}{0.995}{0.996}
& \textbf{\cellci{0.664}{0.616}{0.710}}
& 0.795
& 0.000
& \textbf{\cellci{0.640}{0.595}{0.685}}
& \cellci{1.115}{1.023}{1.204}
& \textbf{\cellci{0.695}{0.651}{0.734}} \\

$\lfloor 0.1 \times N \rfloor$
& 0.943
& \cellci{0.975}{0.905}{1.045}
& \cellci{0.996}{0.995}{0.996}
& \textbf{\cellci{0.672}{0.625}{0.726}}
& 0.799
& 0.000
& \cellci{1.006}{0.948}{1.068}
& \cellci{1.322}{1.211}{1.438}
& \textbf{\cellci{0.748}{0.713}{0.784}} \\

\bottomrule
\end{tabular}
\end{table*}

\begin{table*}[t]
\centering
\caption{
Results for linear regression experiments with Boston housing data.
See \cref{tab:robust_linear_regression} caption for further explanation. 
$\infty$ denotes divergence under this tuning guidance.
}
\label{tab:robust_linear_regression_cov_real}

\begin{tabular}{lcccc|ccccc}
\toprule
& \multicolumn{4}{c|}{Log loss}
& \multicolumn{5}{c}{$\beta$-loss ($\beta = 1.5$)} \\
\cmidrule(lr){2-5}\cmidrule(lr){6-10}
$B$
& Posterior
& CT
& LR+WS
& DQ+exact
& NUTS
& Sandwich Gauss
& CT
& LR+WS
& DQ+exact \\
\midrule

\multicolumn{10}{l}{\textbf{Covariance error}} \\
\midrule

$16$
& 0.358
& \textbf{\cellci{0.247}{0.194}{0.310}}
& \cellci{$9.23\times 10^{8}$}{$3.62\times 10^{4}$}{$6.17\times 10^{9}$}
& \textbf{\cellci{0.337}{0.262}{0.405}}
& 2.528
& 0
& \textbf{\cellci{2.054}{1.723}{2.328}}
& $\infty$
& \textbf{\cellci{2.782}{0.965}{9.328}}
\\

$\lfloor 0.1 \times N \rfloor$
& 0.358
& \cellci{0.589}{0.443}{0.804}
& \cellci{$1.40\times 10^{7}$}{$4.85\times 10^{3}$}{$9.01\times 10^{7}$}
& \textbf{\cellci{0.352}{0.274}{0.441}}
& 2.528
& 0
& \cellci{3.126}{2.313}{5.338}
& $\infty$
& \textbf{\cellci{1.398}{0.844}{2.132}}
\\

\bottomrule
\end{tabular}
\end{table*}

 {%
\setlength{\tabcolsep}{3.5pt}  %
\renewcommand{\arraystretch}{1.05}

\begin{table}[t]
\centering
\caption{%
Results for Poisson regression experiments. 
See \cref{tab:robust_linear_regression} caption for further explanation. 
}
\label{tab:poisson_glm_and_german}
\begin{tabular}{l l c c c}
\toprule
& & \multicolumn{2}{c}{\textbf{Simulated}} &  \textbf{Credit} \\
\cmidrule(lr){3-4} \cmidrule(lr){5-5}
$B$ & method %
& calib.\ err. & cov.\ err.  & cov.\ err.\\
\midrule
$16$
& CT %
& \textbf{\cellci{0.069}{0.062}{0.075}}
& \textbf{\cellci{0.207}{0.199}{0.215}} 
& \textbf{\cellci{0.132}{0.112}{0.168}} \\
& DQ+const %
& \cellci{0.646}{0.639}{1.344}
& \cellci{0.672}{0.664}{0.678} 
& \cellci{0.982}{0.975}{0.987}  \\
& DQ+exact %
& \textbf{\cellci{0.074}{0.068}{0.080}}
& \textbf{\cellci{0.208}{0.201}{0.217}} 
& \textbf{\cellci{0.157}{0.132}{0.193}}  \\
\midrule
$\lfloor0.1\!\times\!N\rfloor$
& CT %
& \cellci{0.089}{0.078}{0.100}
& \cellci{0.230}{0.218}{0.245} 
& \cellci{0.191}{0.155}{0.240}  \\
& DQ+const %
& \cellci{1.376}{1.370}{1.382}
& \cellci{0.991}{0.990}{0.992} 
& \cellci{0.997}{0.996}{0.999}  \\
& DQ+exact %
& \textbf{\cellci{0.075}{0.066}{0.083}}
& \textbf{\cellci{0.211}{0.203}{0.220}} 
& \textbf{\cellci{0.154}{0.138}{0.181}}  \\
\bottomrule
\end{tabular}

\end{table}

\section{A General Procedure for Calibrated SG(L)D Sampling}

\cref{alg:uq_tuning} outlines a practical tuning procedure for SG(L)D uncertainty calibration, which covers all the approaches listed in \cref{tbl:summary}. 
When tuning of $\Lambda$ using our proposed approach (DQ+exact), \cref{alg:uq_tuning} is applicable to the large-sample, low-to-moderate dimensional regime. 
Its computational complexity is $O(MD^2 + D^3) + O\left(T(BD + D^2)\right)$
where the first term corresponds to a one-time offline cost using a subsample of size $M$, and the second term is the cost of $T$ stochastic gradient iterations with minibatch size $B \ll N$. 
The $O(D^3)$ term is incurred only once and becomes negligible when $N \gg D^3$. 
Since the mixing time of tuned SG(L)D is $O(1)$ epochs (equivalently $O(N/B)$ iterations), relative Monte Carlo error $\delta$ is achievable with $T = O(N/[B\delta])$ iterations. 
Hence, the overall computational complexity is $O\bigl(N[D + D^2/B]/\delta\bigr)$.
This result also suggests a benefit to using a large batch size of at least $B \gg D$ to reduce the number of preconditioning operations, improving the computational efficiency of SG(L)D.  

The tuning procedure also requires $O(D^2)$ memory to store and manipulate quantities such as Hessian $\hat{\mathcal{J}}$ and Fisher information $\hat{\mathcal{I}}$, making accurate computation challenging in very high-dimensional settings. 
A promising future direction is to develop structured, low-rank, diagonal, or trajectory-based approximations of $\hat{\mathcal{J}}$ or $\hat{\mathcal{I}}$ computations to improve scalability.

\section{Experiments} \label{sec:more-applications}

We compare the accuracy of the learning rate tuning guidance provided by our theory versus previous work (see \cref{tbl:summary}). 
For fair comparison, we follow \cref{alg:uq_tuning}, with the only difference across approaches being how $\Lambda$ is determined. 
In our experiments, we use SGD (so, $\beta = \infty$). 
The code for all experiments is publicly available at \url{https://github.com/wangyu1369/large-sample-sgmcmc-uq}.

In our experiments we compute $\Lambda$ using a numerical optimization procedure since obtaining a close-form solution is 
challenging \citep{hammarling1982numerical, ye1998stability}. %
Specifically, we substitute the stationary noise expression from \cref{eq: covariance of stationary noise} into the stationary covariance equation \cref{eq: SGLD stationary covariance} and set $\Sigma_\psi = \widehat{\mathcal{S}}$. This yields a matrix equation of the form $F(\Lambda)=0$, where $\Lambda$ is the only unknown. We solve this system numerically by vectorizing $\Lambda$ and applying \texttt{scipy.optimize.root} with the Powell hybrid method.
While our approach requires solving jointly \cref{eq: SGLD stationary covariance,eq: covariance of stationary noise}, this cost is incurred only once per problem and is negligible compared to the dominant cost of running SG(L)D trajectories. We empirically verify this in \cref{sec:Computational Cost}.

\subsection{Robust Linear Regression}
\label{sec:Robust Linear Regression}

While standard Bayesian inference and maximum likelihood estimation optimize the KL divergence, the resulting log loss is notoriously sensitive to outliers and misspecification, allowing a single atypical datapoint to dominate the gradient.
The $\beta$-divergence provides a robust alternative by downweighting low-probability observations through a tunable power parameter $\beta$, effectively controlling heavy-tailed effects \citep{ghosh2016robust, jewson2018principles, jewson2024stability}.
We consider a regression setting with observations $x = (y, z)$ where $y$ denotes the response and $z$ the covariates.
Then the $\beta$-divergence loss is defined as $\ell(\theta ; x)=-\frac{1}{\beta-1} f(y ; \theta, z)^{\beta-1}+\frac{1}{\beta} \int f\left(y^{\prime} ; \theta, z\right)^\beta \mathrm{d} y^{\prime}$, where $f(y ; \theta, z)$ denotes the likelihood.
For a detailed discussion of tuning guidance under the $\beta$-divergence loss, see \cref{sec:beta discussion}.
\paragraph{\emph{Simulated misspecified data with outliers.}}
First, we consider a misspecified linear regression model with heteroskedastic errors. Data $\{(x_n, y_n)\}_{n=1}^N$ are generated according to
$y_n \mid x_n \sim \mathcal{N}\!\left(x_n^\top \theta_\star,\; 1 + \|x_n\|_2^2\right)$,
where the true parameter $\theta_\star \sim \mathcal{N}(0, I_D)$ is fixed throughout the experiment, and the covariates are drawn independently as $x_n \sim \mathcal{N}(0, I_D)$.

To model heavy-tailed contamination, a fraction $p \in [0,1]$ of samples is selected uniformly at random and replaced by outliers. For these contaminated observations, responses are generated as
$y_n \mid x_n \sim \mathcal{N}\!\left(x_n^\top \theta_\star + b,\; s^2\bigl(1 + \|x_n\|_2^2\bigr)\right)$,
where $s > 1$ controls variance inflation and $b$ introduces a mean shift. Unless otherwise specified, we set $D = 50$, $N = 5000$, $p = 0.01$, $b = 5.0$, and $s = 5.0$.

As shown in \cref{tab:robust_linear_regression}, LR+WS performs poorly under log loss due to its reliance on well-specified model assumptions. 
Replacing the log loss with the $\beta$-divergence substantially improves its quantile calibration. However, comparable calibration does not imply accurate uncertainty quantification: existing tunings (CT and LR+WS) exhibit significant covariance mismatch, particularly for large batch sizes. 
In contrast, our tuning consistently yields a stationary covariance closest to the target sandwich covariance $\mathcal{S}_{\star}$, with the largest gains observed in the large batch size regime.

\paragraph{\emph{Boston housing data} \citep{harrison1978hedonic}}We next consider the Boston housing dataset, where the linear model is strongly misspecified.
In this setting, LR+WS becomes unstable and produces extremely large covariance errors.
These results corroborate the simulated experiments and highlight the importance of our tuning guidance for accurate uncertainty quantification in misspecified and large-batch regimes.

\subsection{Poisson Regression}

Finally, we demonstrate how our theory provides tuning guidance for the more challenging case of Poisson regression.
\paragraph{\emph{Simulated data.}}
We first consider simulated data generated from the assumed model
$
    y_{n} \sim \text{Poisson}(\exp\{x_{n}^{\top}\theta_{\star}\}),
$
where $\theta_{\star} \dist \Norm(0, I_{D})$ and $x_{n} \distiid \Norm(0, I_{D})$.
We use a sample size of $N = 5{,}000$ and set dimension $D = 50$.

\paragraph{\emph{Credit data} \citep{UCIGermanCredit}} 
The German credit data contains data on $D = 20$ variables and the credibility of $N = 1{,}000$ loan applicants.
For both, we consider batch sizes $B \in \{16, \lfloor 0.1\times N \rfloor\}$. 
The results in \cref{tab:poisson_glm_and_german} show that while continuous-time tuning remains competitive for small batch sizes, its accuracy degrades substantially as the batch size increases, leading to larger covariance and calibration errors. 
In contrast, the tuning guidance derived from our discrete-time theory consistently yields more accurate uncertainty quantification in the large-batch regime. 
Existing discrete-time approaches based on {quadratic objectives} and {constant noise} suffer from severe miscalibration once these restrictive assumptions are violated. 
The improved calibration and covariance accuracy demonstrate that our method provides reliable UQ beyond the regimes where continuous-time or constant-noise approximations apply.

\subsection{Neural Network}
\label{sec: neural network exp}
We further compare the stationary covariance predicted by the different theories in \cref{tbl:summary} on a real-world neural network task. 
Specifically, we fit a two-hidden-layer $\tanh$ neural network with hidden widths $(2,3)$ on the \emph{Diabetes} dataset \citep{efron2004least}. 

As we can see from \cref{fig:neural-network-plot}, at small learning rates, all methods are comparable. However, as the learning rate increases, continuous-time approximations become quantitatively inaccurate, while our discrete-time method remains accurate.
While our non-asymptotic error analysis assumes convexity, the characterization of the stationary covariance $\Sigma_{\psi}$ (\cref{thm:SGLD-stationary-covariance}) and minibatch noise $\overline{C}_{\psi}$ (\cref{theorem: general case covariance matrix of noise}) does not rely on this assumption. 
Empirically, we observe that our discrete-time proxy remains effective for neural networks, suggesting that the approach can potentially extend beyond the convex setting and motivating future work on non-convex analysis.
\begin{figure}[t]
    \centering
    \includegraphics[width=\columnwidth]{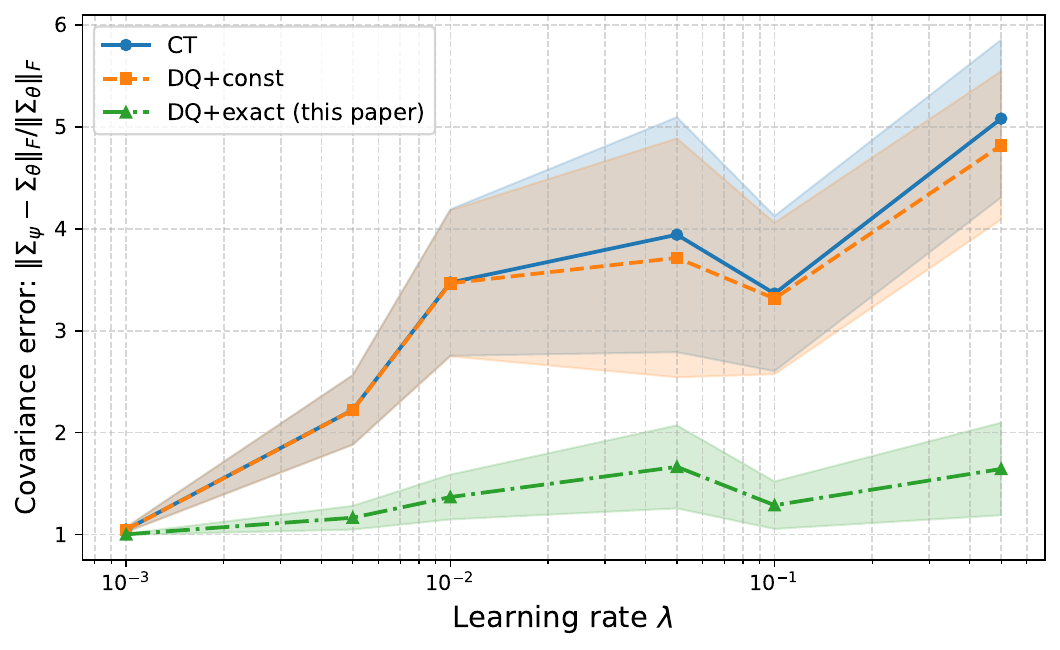}
    \caption{
    Covariance prediction error for neural network with hidden on the \emph{Diabetes} dataset. 
    The error is measured as 
    $\|\Sigma_{\psi}-\Sigma_{\theta}\|_{F}/\|\Sigma_{\theta}\|_{F}$, 
    where $\Sigma_{\theta}$ is the empirical stationary covariance estimated from SGD tail iterates and $\Sigma_{\psi}$ is the covariance predicted by each theory. 
    Shaded regions denote $95\%$ confidence intervals for the mean across 30 independent repetitions.
    }
    \label{fig:neural-network-plot}
\end{figure}

\section{Conclusion}

We study uncertainty quantification for stochastic gradient methods from a discrete-time perspective. 
Our results show that accurate characterization of SGD and SGLD stationary behavior requires moving beyond continuous-time approximations, particularly at large batch sizes and non-vanishing learning rates. 
By explicitly modeling the stationary covariance and minibatch-induced noise structure, our framework provides principled and practical tuning strategies for SGD and SGLD under both well-specified and misspecified settings. 
Empirically, we demonstrate improved covariance estimation and calibration across synthetic and real-world tasks when using \cref{alg:uq_tuning}. 

A limitation of our finite-sample analysis is that it relies on strong convexity assumptions. 
Empirically, we observe that the proposed discrete-time proxy remains effective for neural networks (\cref{sec: neural network exp}), suggesting that the approach may extend beyond convex settings. Developing finite-sample guarantees for non-convex models remains an important direction for future work.
Another natural extension is to characterize more precisely the regimes of learning rate and batch size in which continuous-time approximations fail, and to identify sharp thresholds at which discrete-time effects dominate stationary behavior.

\clearpage  

\section*{Acknowledgments}
Y.~Wang and J.~H.~Huggins were partially supported by National Science Foundation CAREER award
IIS-2340586. 

\section*{Impact Statement}
This paper presents work whose goal is to advance the field
of probabilistic machine learning. There are many potential societal
consequences of our work, none which we feel must be
specifically highlighted here.
}%

\bibliography{scaling-limit-sgld}
\bibliographystyle{icml2026}

\newpage

\appendix
\onecolumn

\counterwithin{equation}{section}
\counterwithin{figure}{section}
\renewcommand{\thefigure}{\Alph{section}.\arabic{figure}}
\counterwithin{table}{section}
\renewcommand{\thetable}{\Alph{section}.\arabic{table}}
\counterwithin{theorem}{section}
\renewcommand{\thetheorem}{\Alph{section}.\arabic{theorem}}
\counterwithin{lemma}{section}
\renewcommand{\thelemma}{\Alph{section}.\arabic{lemma}}
\counterwithin{assumption}{section}
\renewcommand{\theassumption}{\Alph{section}.\arabic{assumption}}
\counterwithin{definition}{section}
\renewcommand{\thedefinition}{\Alph{section}.\arabic{definition}}
\counterwithin{proposition}{section}
\renewcommand{\theproposition}{\Alph{section}.\arabic{proposition}}

\renewcommand{\cellci}[3]{%
  \makecell[c]{#1\\{\footnotesize[#2,#3]}}%
}

\setlength{\tabcolsep}{3pt}  
\renewcommand{\arraystretch}{1.05}
\section{Discussions on Mixing Speed}
\label{sec: Discussions on Mixing speed}
Beyond matching a desired stationary covariance, practical uncertainty quantification also requires that the SG(L)D Markov chain mixes rapidly
so that Monte Carlo estimates are accurate at low computational cost. Let $(\theta_t)_{t\ge 0}$ denote the SG(L)D iterates and let $\pi_\theta$ be their stationary distribution. 
Given a scalar functional $f \colon \mathbb{R}^D\to\mathbb{R}$, let
$\pi_\theta(f) \defas \int f(\theta)\,\pi_\theta(\mathrm{d}\theta)$
denote its expectation under the invariant distribution, which is the quantity we ultimately want to estimate.
Given $T$ iterates, the standard Monte Carlo estimator for $\pi_\theta(f)$ is
$\hat f_T \defas T^{-1}\sum_{t=1}^T f(\theta_t).$

To isolate the effect of mixing, suppose the chain is started at stationarity: $\theta_0 \sim \pi_\theta$. 
Letting $\rho_k(f) \defas \mathrm{Corr}_{\pi_\theta}(f(\theta_0), f(\theta_k))$ denote the lag-$k$ autocorrelation of the stationary time series $(f(\theta_t))_{t\ge 0}$, the \emph{integrated autocorrelation time}
\[
\textstyle  \tau_{\mathrm{int}}(f) \defas 1 + 2\sum_{t=1}^{\infty}\rho_k(f)
\] \citep{geyer1992practical,sokal1997_montecarlo}
quantifies how much serial dependence inflates Monte Carlo variance relative to $T$ i.i.d.\ draws: rapid mixing corresponds to fast decay and/or negative values of $\rho_t(f)$, which results in small $\tau_{\mathrm{int}}(f)$.
In particular, $\hat f_T$ is unbiased and, under standard regularity conditions, its variance takes the form
\begin{align}
    \var(\hat f_T)
    \approx \frac{\mathrm{Var}_{\pi_\theta}(f)}{T}\,\tau_{\mathrm{int}}(f),
\end{align}\citep{jones2004_mcmc_clt,geyer1992practical}
where $\mathrm{Var}_{\pi_\theta}(f)$ is the marginal variance of $f(\theta)$ when $\theta\sim\pi_\theta$. 
Equivalently, the \emph{effective sample size} is $T/\tau_{\mathrm{int}}(f)$ \citep{gelman1995bayesian}, making $\tau_{\mathrm{int}}(f)$ a direct measure of sampling efficiency.

\section{SGLD with Momentum}
\label{sec:momentum-theory}

SGLD with momentum $\kappa$ is defined by the one-step update equations
\begin{align}
	\label{eq: original SGD with momentum}
	\left\{\begin{array}{l} %
		m_{t}=\kappa m_{t-1}+ G_{t}(\theta_{t-1}) \\ %
		\theta_t=\theta_{t-1}-\Lambda m_{t} + \sqrt{2\beta^{-1} \Lambda}\,\xi_{t-1}.
	\end{array}\right. 
\end{align}

Combining \cref{eq: original SGD with momentum} with the approximation given in \cref{eq: proxy algorithm framework for general loss} leads to the proxy algorithm with one-step update 
\begin{align}
	\label{eq: proxy SGD with momentum}
	\left\{\begin{array}{l}
		\nu_{t}=\kappa \nu_{t-1}+ 
		G_{t}(\MLE) +  \grad G_{t}(\MLE) \left(\psi_{t-1}-\MLE\right) %
        \\
		\psi_{t}=\psi_{t-1}-\Lambda \nu_{t}+ \sqrt{2\beta^{-1} \Lambda}\,\xi_{t-1}
	\end{array}\right.
\end{align}

We first present the stationary covariance of SGLD with momentum, which recovers the case without momentum by taking $\kappa = 0$. 
Proofs of results in this section are in Appendix~\ref{sec:momentum-proofs}. 

\begin{proposition}
	\label{Proposition: stationary covariance analysis for SGLD with momentum}
	If the iterates are updated according to \cref{eq: proxy SGD with momentum} and they have a stationary distribution, then the stationary covariance $\Sigma_{\psi}$ satisfies
    \begin{align}
	\label{eq:stationary covariance for momentum SGLD}
        \begin{aligned}
		(1-\kappa)(\Lambda \hess \Sigma+\Sigma \hess \Lambda)+\frac{\kappa}{1-\kappa^2}(\Lambda \hess \Lambda \hess \Sigma+\Sigma \hess \Lambda \hess \Lambda)
        & =\Lambda  \overline{C}_{\psi} \Lambda + \frac{1+\kappa^2}{1-\kappa^2} \Lambda \hess \Sigma \hess \Lambda + (1+\kappa^{2}) \frac{2\Lambda}{\beta}.
        \end{aligned}
	\end{align}
\end{proposition}

\begin{proposition} \label{prop:iterate-average-momentum}
Under the same hypotheses as \cref{Proposition: stationary covariance analysis for SGLD with momentum}, the iterate average $\bar{\psi}_{k} = \frac{1}{k}\sum_{k^{\prime}=1}^{k} \psi_{k^{\prime}}$ has stationary covariance 
\[
\Sigma_{\psi}^{(k)} = \frac{1}{k^{2}} \left( k\Sigma_{\psi} + 2\sum_{k^{\prime}=1}^{k-1} \left( I - \Lambda \hess\right)^{k^{\prime}} \Sigma_{\psi}\right),
\]
where $\Sigma_{\psi}$ is defined by \cref{eq:stationary covariance for momentum SGLD}.
\end{proposition}

In the momentum setting, we obtain 2-Wasserstein error bounds analogous to the non-momentum case. The main difference is that the contraction factor and the corresponding constants now depend on the momentum parameter $\kappa$.

\begin{theorem}\label{thm:wasserstein-error-bound-momentum}
latexIf assumptions \ref{assump:statistic-convex}--\ref{assump:loss} hold and $\Lambda=\lambda I_D$
for some $\lambda>0$, $\kappa\in(0,1)$, $\lambda\in(0,(1-\kappa)/(4L))$,  and 
$\frac{2\lambda L^2}{1-\kappa} + \frac{2L^2\kappa(1+\lambda)}{(1-\kappa)^2} + \kappa < \mu$, 
then, letting $\bar\beta \defas \rho(A) < 1$ for the coefficient matrix $A$ defined in 
\cref{eq:matrix-A-mom}, 
$\overline{M^p} \defas N^{-1}\sum_{n=1}^N M_n^p$ for $p \in \{1,2\}$, 
$C_s \defas \E\|\psi_s - \MLE\|^4$, and $\mathcal P$ as in \cref{eq:definition-of-P}, 
for all $t = 1, 2, \dots$,
\[
W_2^2\bigl(\pi_{\theta,t},\pi_{\psi,t}\bigr)
\le
\bar\beta^{\,t}\,W_2^2\!\bigl(\pi_{\theta,0},\pi_{\psi,0}\bigr)
+
\mathcal{P}\,\sum_{s=1}^{t}\bar\beta^{\,t-s}\,C_{s-1}.
\]
\end{theorem}

Finally, we find that, with the momentum proportional to $\lambda$, the Wasserstein error remains of order $\lambda/B$; moreover, the bound recovers the non-momentum result when $\kappa=0$.

\begin{corollary}\label{cor:stationary-wasserstein-error-bound-momentum}
Under the same assumptions as \cref{thm:wasserstein-error-bound-momentum} with $\beta = \infty$, assume the \emph{scaled momentum} regime $\kappa = c_\kappa \lambda$ with
$0<c_\kappa\le\min\Big\{\tfrac{\mu^2}{32L^2},\ \tfrac{\hat\mu}{c_1\hat L^3}\Big\}$ and
$\lambda\le\min\Big\{
1,\ \tfrac{1}{\hat L},\ \tfrac{\mu}{c_2L^2},\ \tfrac{B\hat\mu}{c_3L^2},\ (\hat\mu/c_4)^{1/4}
\Big\}.$ 
Then there exists a constant $A_\star > 0$ (given at \cref{eq:Amom}) such that
\[
W_2(\pi_\theta, \pi_\psi) \le A_\star \frac{\lambda}{B}.
\]
\end{corollary}

\section{Application to High-dimensional Problems}
\label{sec:high-dim}

\paragraph{Dense setting.}
Let $D$ denote the parameter dimension, let $I \dist \mathrm{Unif}(\{1,\dots,N\})$ be independent,
and define
$g_I \defas \nabla_\theta \ell(x_I,y_I,\hat\theta),
\tau_2^2 \defas \E\|g_I\|^2$, and $
\tau_4^4 \defas \E\|g_I\|^4$.
Assume there exist constants $c_2,c_4<\infty$ independent of $D$ such that
$\tau_2^2 \le c_2 D$ and $\tau_4^4\le c_4 D^2$.
Such bounds hold, for example, under sub-Gaussian designs with uniformly bounded GLM weights;
see \citealp{Vershynin2018highd}.
Let $\xi\sim\mathcal N(0,I_D)$ so that $\E\|\xi\|^2=D$ and $\E\|\xi\|^4=D(D+2)$.
Corollaries~\ref{cor:stationary-wasserstein-error-bound} and~\ref{cor:stationary-wasserstein-error-bound-momentum} yield
\[
W_2(\pi_\theta,\pi_\psi)\ \le\ A_{\mathrm{eff}}\Big(\frac{\lambda}{B}+\frac{1}{\beta}\Big),
\]
where $A_{\mathrm{eff}}$ is the explicit constant in the corresponding corollary
(e.g., $A_{\mathrm{eff}}=\sqrt{A}$ when the corollary is stated as $W_2^2\le A(\lambda/B+1/\beta)^2$; see
\cref{eq:wasserstein-error-bound-constant-sgld,eq:Amom} for the explicit definitions).
If the curvature constants entering these corollaries (e.g., $\mu,L,\hat\mu,\hat L$) are bounded above and below by constants
independent of $d$, then inserting the bounds on $\tau_2$, $\tau_4$, and the Gaussian moments above
into \cref{eq:wasserstein-error-bound-constant-sgld,eq:Amom} shows that there exists $C>0$ independent of $D$ such that
$A_{\mathrm{eff}}\le C D$, and hence
\[
W_2(\pi_\theta,\pi_\psi)\ \le\ C\, D\Big(\frac{\lambda}{B}+\frac{1}{\beta}\Big).
\]
In particular, if $B\ge c D$ then $W_2(\pi_\theta,\pi_\psi)\le C'(\lambda+D/\beta)$, and thus
$W_2(\pi_\theta,\pi_\psi)\le C''\lambda$ uniformly in $D$ when $\beta=\infty$ or when $\beta\ge c''' D/\lambda$.
If instead $\beta=N$ and $N/D\in[\gamma_{\min},\gamma_{\max}]$ for fixed $0<\gamma_{\min}\le\gamma_{\max}<\infty$,
then $d/\beta=D/N\in[1/\gamma_{\max},1/\gamma_{\min}]$ and the bound need not vanish as $D\to\infty$.

\paragraph{Sparse setting.}
Alternatively, we can consider the sparse regime.
Let $S\subseteq[D]$ with $|S|=s\ll D$ and let $P_S$ be the coordinate projector.
Assume both the exact and proxy chains evolve on the affine subspace
$\mathcal A_S\defas\hat\theta+\mathrm{range}(P_S)$; for example, if $P_{S^c}\theta_0=P_{S^c}\psi_0=P_{S^c}\hat\theta$
and $P_S$ is applied to every drift and injected-noise term so that $\theta_t,\psi_t\in\mathcal A_S$ for all $t$. 
Define the restricted gradient moments at $\hat\theta$ by
$g_{I,S}\defas P_S\nabla_\theta \ell(x_I,y_I,\hat\theta),
\tau_{2,S}^2 \defas \E\|g_{I,S}\|^2$ and $
\tau_{4,S}^4 \defas \E\|g_{I,S}\|^4$.
Assume the curvature constants, when restricted to $\mathcal A_S$, are bounded above and below by
constants independent of $D$ and $s$, and there exist constants $c_2,c_4<\infty$ independent of $D$ and $s$ such that
$\tau_{4,S}^2 \le c_2\, s$ and 
$\tau_{4,S}^4 \le c_4\, s^2$
A sufficient condition is isotropic sub-Gaussian designs on $S$ with uniformly bounded per-sample weights; see \citep{Vershynin2018highd}. 
If the injected noise is also projected, then for $\xi\sim\mathcal N(0,I_D)$ it holds that 
$\E\|P_S\xi\|^2=s$ and $\E\|P_S\xi\|^4=s(s+2)$.
Then the same inspection of the constants in \cref{cor:stationary-wasserstein-error-bound} and \cref{cor:stationary-wasserstein-error-bound-momentum} yields a constant $C>0$ independent of $D$ and $s$ such that
\[
W_2(\pi_\theta,\pi_\psi) \le Cs\left(\frac{\lambda}{B}+\frac{1}{\beta}\right).
\]
In particular, if $B\ge c s$ then $W_2(\pi_\theta,\pi_\psi)\le C'(\lambda+s/\beta)$, hence
$W_2(\pi_\theta,\pi_\psi)\le C''\lambda$ for $\beta=\infty$ and also for $\beta\ge c'''\, s/\lambda$.
If the injected noise is not rank-$s$ , then the diffusion moments scale with $D$
(since $\E\|\xi\|^2=D$ and $\E\|\xi\|^4=D(D+2)$), so the temperature-dependent contribution generally scales with $D/\beta$ rather than $s/\beta$.

\section{Additional Experiment Details}
\subsection{Empirical Validation of the Wasserstein Bound}

We empirically validate the Wasserstein error bound in 
\cref{cor:stationary-wasserstein-error-bound} using Poisson regression in both
well-specified and misspecified settings. We generate covariates
$x_i \in \mathbb{R}^D$ and responses according to two synthetic data-generating
mechanisms. In the well-specified setting, the responses are generated from
\[
y_i \sim \mathrm{Poisson}\!\left(\exp\{x_i^\top \theta_\star\}\right),
\]
and the fitted model is also Poisson regression. In the misspecified setting, the
responses are generated from a negative binomial model,
\[
y_i \sim \mathrm{NegBin}(r_i,p_i),
\qquad
p_i = \frac{r_i}{r_i+\exp\{x_i^\top\theta_\star\}},
\]
but are still fitted using a Poisson model. 
We use sample size $N=2{,}000$ and dimension $D=50$.

To instantiate the theoretical upper bound, we use plug-in estimates evaluated at
$\MLE$. Specifically, we estimate the smoothness constant $L$ and strong convexity
constant $\mu$ using the largest and smallest eigenvalues of the empirical Hessian,
respectively. The higher-order quantities $\overline{M}$, $\overline{M^2}$, and
$\tau_4$ are estimated empirically from the sample gradients. We then evaluate
$W_2(\pi_\theta,\pi_\psi)$ across different batch sizes $B$ and different values of
the ratio $\lambda/B$.

As shown in \cref{Fig:w2_bound}, the empirical Wasserstein distance exhibits a clear
approximately linear scaling in $\lambda/B$ across all batch sizes. Moreover, in both
the well-specified Poisson setting and the misspecified negative binomial setting,
the empirical curves remain below the theoretical upper bound. This confirms that
the bound captures the correct dependence on $\lambda/B$, although it is
conservative in magnitude.

\begin{figure}[t]
    \centering
    \includegraphics[width=\textwidth]{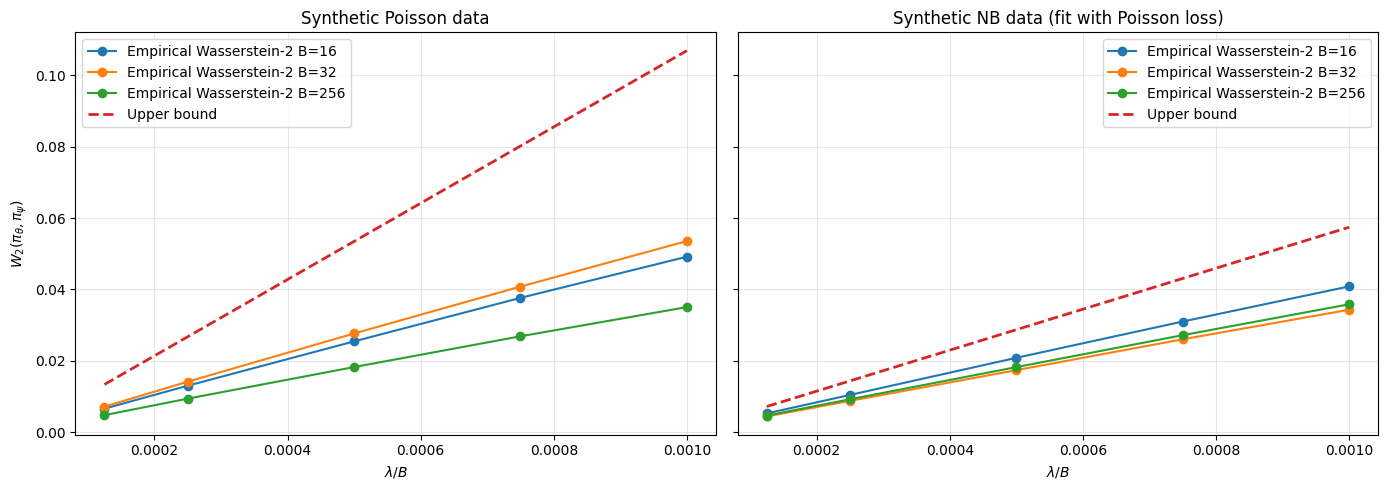}
    \caption{
    Empirical validation of the Wasserstein error bound in
    \cref{cor:stationary-wasserstein-error-bound}. 
    \textbf{Left:} well-specified Poisson data fitted with a Poisson model.
    \textbf{Right:} misspecified negative binomial data fitted with a Poisson model.
    }
    \label{Fig:w2_bound}
\end{figure}

\subsection{Computational Cost of Determining $\Lambda$}
\label{sec:Computational Cost}
In this subsection, we benchmark the wall-clock cost of computing the preconditioner $\Lambda$—by solving the matrix equation induced by \cref{eq: SGLD stationary covariance,eq: covariance of stationary noise}—against the cost of running the resulting MCMC chain. \cref{tab: cost analysis linear regression,tab: cost analysis Poisson regression} show that $\Lambda$ can be computed extremely quickly relative to sampling: across all tested dimensions, $\Lambda$ construction takes at most a few milliseconds and is typically $10^{-5}$–$10^{-3}$ of the MCMC runtime. In the following exploratory study, we fix $N=1000$ and $B=64$, and run MCMC for 10 epochs.

\begin{table}[t]
\label{tab: cost analysis linear regression}
\centering
\caption{
Linear regression: comparison of  $\Lambda$ computation and MCMC running time.
Entries are mean wall-clock time (seconds) averaged over independent runs.
}
\label{tab:timing_linear_lambda_exact}
\small
\begin{tabular}{c c c c}
\toprule
$D$ & $\Lambda$ computation time & MCMC time &  $\Lambda$ computation time / MCMC \\
\midrule
5  & $9.2 \times 10^{-5}$ & $1.49$ & $6.2 \times 10^{-5}$ \\
10 & $8.4 \times 10^{-4}$ & $1.67$ & $5.0 \times 10^{-4}$ \\
20 & $6.5 \times 10^{-4}$ & $1.83$ & $3.6 \times 10^{-4}$ \\
50 & $8.0 \times 10^{-4}$ & $3.03$ & $2.6 \times 10^{-4}$ \\
\bottomrule
\end{tabular}
\end{table}

\begin{table}[t]
\label{tab: cost analysis Poisson regression}
\centering
\caption{
Poisson regression: comparison of $\Lambda$ computation and MCMC time.
Entries are mean wall-clock time (seconds) averaged over independent runs.
}
\label{tab:timing_poisson_lambda_taylor}
\small
\begin{tabular}{c c c c}
\toprule
$D$ &  $\Lambda$ computation time & MCMC time &  $\Lambda$ computation time / MCMC \\
\midrule
5  & $9.5 \times 10^{-5}$ & $1.71$ & $5.6 \times 10^{-5}$ \\
10 & $4.0 \times 10^{-3}$ & $1.74$ & $2.3 \times 10^{-3}$ \\
20 & $8.5 \times 10^{-4}$ & $1.79$ & $4.8 \times 10^{-4}$ \\
50 & $7.8 \times 10^{-4}$ & $1.81$ & $4.3 \times 10^{-4}$ \\
\bottomrule
\end{tabular}
\end{table}

\subsection{Details of the $\beta$-divergence.}
\label{sec:beta discussion}
Under the definition of $\beta$-divergence $\ell^{(\beta)}(y, f(\cdot ; \theta))=-\frac{1}{\beta-1} f(y ; \theta)^{\beta-1}+\frac{1}{\beta} \int f(z ; \theta)^\beta d z$, the loss $\loss$ can be rewriten as follows
\begin{align}
	\begin{aligned}
	\loss(\theta) &= \frac{1}{N}\sum_{n=1}^{N} \ell^{(\beta)}(y_{n}, f(\cdot ; \theta)) + \frac{1}{N} \reg(\theta) \\
	& = -\frac{1}{N}\sum_{n=1}^{N} \frac{1}{\beta-1} f(y_{n} ; \theta)^{\beta-1} +\frac{1}{\beta} \int f(z ; \theta)^\beta d z +  \frac{1}{N} \reg(\theta)\\
	& = \frac{1}{N}\sum_{n=1}^{N} \tilde{\ell}_{n}^{(\beta)} + \frac{1}{N} \left(\Omega^{(\beta)}(\theta) + \reg(\theta)\right) ,
	\end{aligned}
\end{align}
where $\tilde{\ell}_{n}^{(\beta)}(\theta) = -\frac{1}{\beta-1} f(y_{n} ; \theta)^{\beta-1}$ and $\Omega^{(\beta)}(\theta) = \frac{N}{\beta} \int f(z ; \theta)^\beta d z $.

Then the loss $\loss$ can be rewritten as 
\begin{align}
	\loss(\theta) = \frac{1}{N} \sum_{n=1}^{N} \ell_{n}^{(\beta)} (\theta) ,
\end{align}
where $\ell_{n}^{(\beta)}(\theta)  =  \tilde{\ell}_{n}^{(\beta)} + \frac{1}{\beta} \int f(z ; \theta)^\beta d z + \frac{1}{N} \reg(\theta)$.

Similarly, we can compute $\widehat{\mathcal{J}}^{(\beta)} = \frac{1}{N}\sum_{n=1}^N \nabla^2 \ell_{n}^{(\beta)}(\hat\theta)$ and $\widehat{\mathcal{I}}^{(\beta)} = \frac{1}{N}\sum_{n=1}^N \nabla \ell_{n}^{(\beta)}(\hat\theta)(\nabla \ell_{n}^{(\beta)}(\hat\theta))^\top$ to use our \cref{alg:uq_tuning} under $\beta$-divergence loss. 
\subsection{Full Experiment Results}
\label{sec: Full Experiment Results}
Due to space constraints, we are unable to report parameter errors and detailed confidence intervals in the main text.
This subsection therefore presents the complete experimental results for all settings.
To aid interpretation of the reported relative covariance errors, we additionally compare the marginal variances of the estimated covariance $\hat{\mathcal{S}}$ with those of the target covariance $\mathcal{S}_{\star}$.

\Cref{Fig: Linear regression on synthetic dataset} shows that for both datasets, using our results leads to the desired marginal variances when using either a small or large batch size. 
The continuous-time tuning performs well when the batch size is small, since a small batch size requires using a small learning rate.
However, the variances are too large in the large batch size case. 
The large-sample+well-specified tuning, on the other hand, leads to excessive variance for both small and large batch size regimes since the assumption that the model is well-specified is violated. \Cref{Fig: Poisson regression on synthetic dataset} shows that theories based on the heuristic SGD noise $\overline{C} = \frac{1}{B}H$ 
lead to an excessively large stationary covariance for the simulated data but a too smaller covariance for the German credit data. 
The continuous-time tuning leads to too large covariance for the large batch size in both cases. 
In our theory, on the other hand, is accurate in all scenarios. 

\begin{figure}[t]
	\centering
	\includegraphics[width=\textwidth]{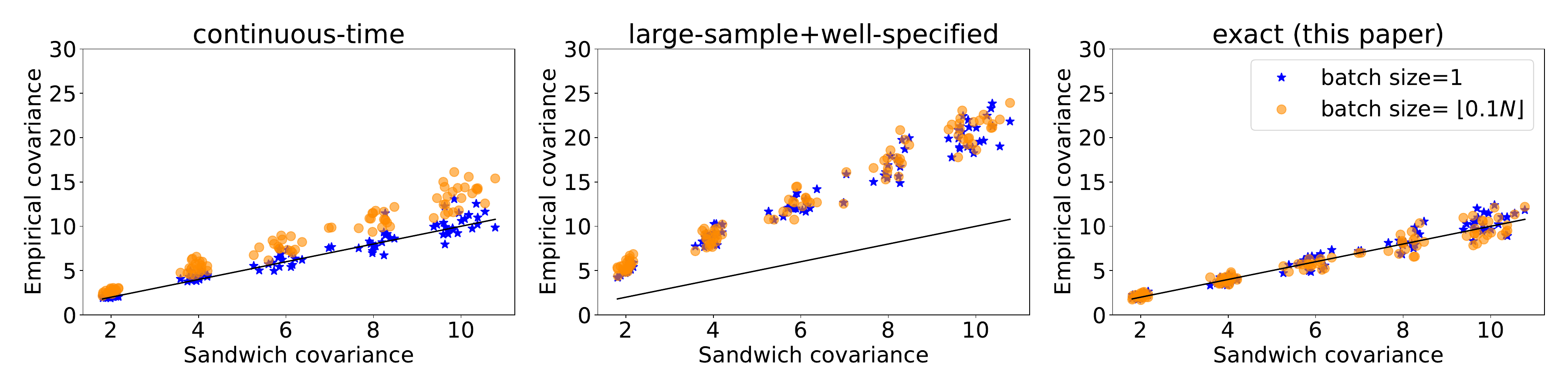}
	\includegraphics[width=\textwidth]{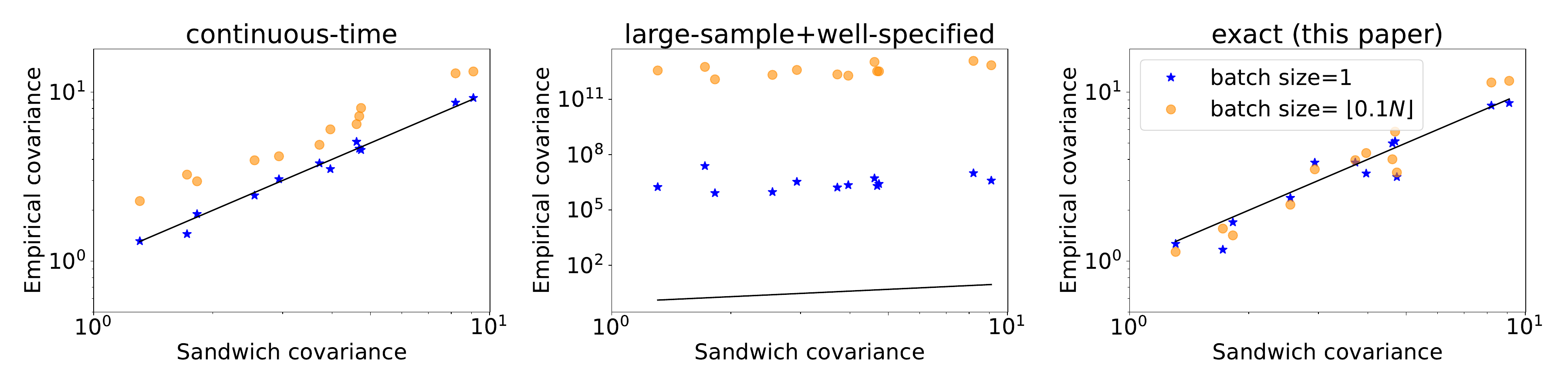}
	\caption{Comparison of step size tuning guidance for linear regression with \textbf{(top)} simulated misspecified data with heteroskedastic noise 
		and \textbf{(bottom)} the classic Boston housing dataset.
	}
	\label{Fig: linear regression on boston housing}
	\label{Fig: Linear regression on synthetic dataset}
\end{figure}

\begin{figure}[t]
	\includegraphics[width=\textwidth]{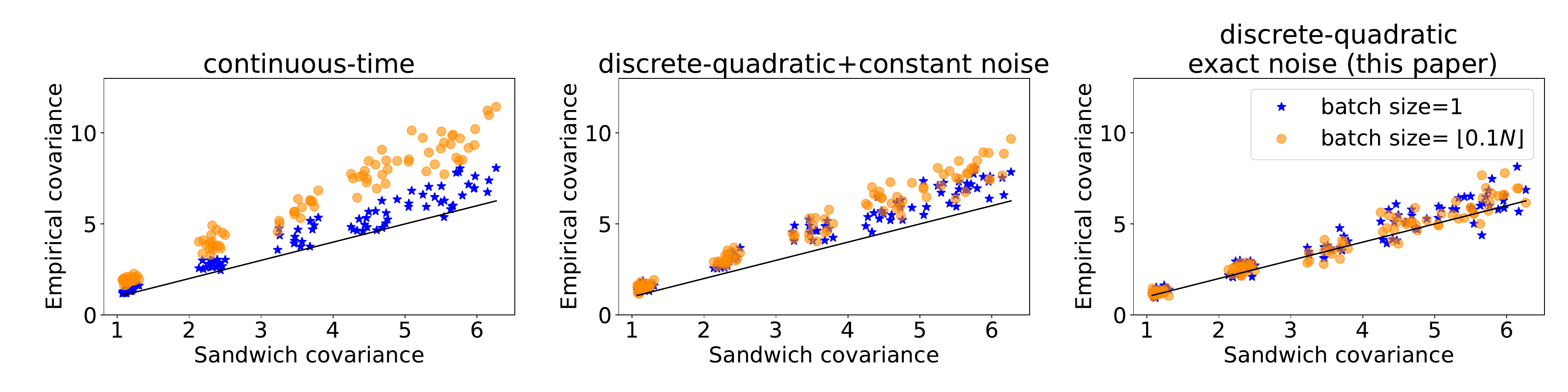}
	\includegraphics[width=\textwidth]{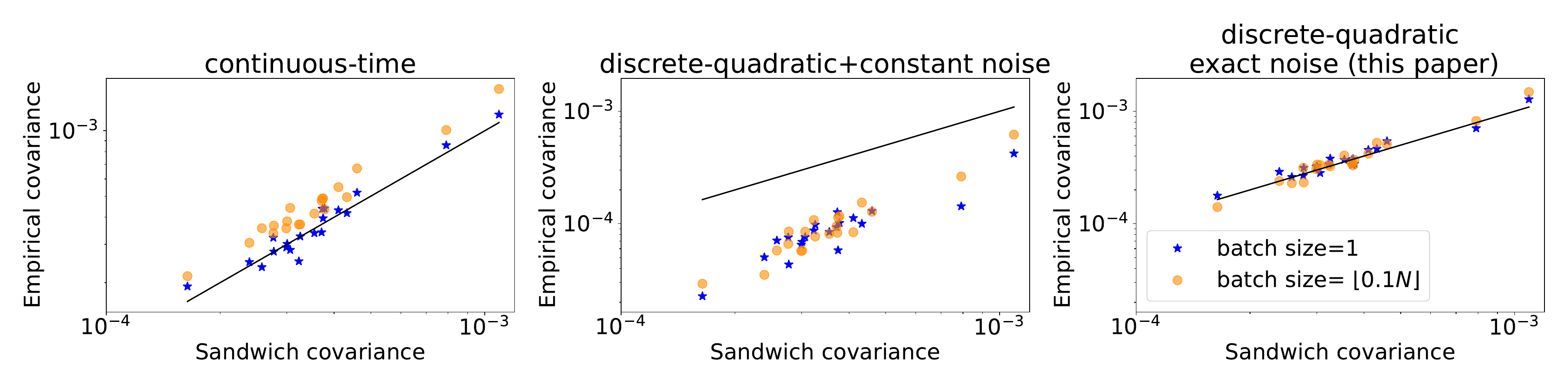}
	\caption{Comparison of step size tuning guidance for Poisson regression with \textbf{(top)} simulated well-specified data 
		and \textbf{(bottom)} the German credit data.
	}
	\label{Fig: Poisson regression on synthetic dataset}
	\label{Fig: Poisson regression on credibility dataset}
\end{figure}

\begin{table*}[htbp]
\centering
\caption{
Full results for linear regression experiments with simulated data. 
Calibration error is the Kolmogorov--Smirnov distance to $\mathrm{Unif}(0,1)$; lower is better.
Covariance error is $(\|\mathcal{S}_{\star}-\hat{\mathcal{S}}\|_{F})/\|\mathcal{S}_{\star}\|_{F}$; lower is better.
Within each metric row and loss block, for a fixed batch size $B$, bold indicates methods whose 95\% confidence intervals overlap with the confidence interval of the method with the lowest mean error.
}
\label{tab:full_robust_linear_regression}

\begingroup
\setlength{\tabcolsep}{3pt}
\renewcommand{\arraystretch}{1.08}
\newcommand{\smallcellci}[3]{%
  \begin{tabular}{@{}c@{}}
  #1\\[-1pt]
  {\scriptsize [#2, #3]}
  \end{tabular}%
}

\begin{tabular}{lcccc|ccccc}
\toprule
& \multicolumn{4}{c|}{Log loss}
& \multicolumn{5}{c}{$\beta$-loss ($\beta = 1.5$)} \\
\cmidrule(lr){2-5}\cmidrule(lr){6-10}
$B$
& Posterior
& CT
& LR+WS
& DQ+exact
& NUTS
& Sandwich Gauss
& CT
& LR+WS
& DQ+exact \\
\midrule

\midrule
\multicolumn{10}{l}{\textbf{Calibration error}} \\
\midrule
$16$
& 0.418
& \textbf{\smallcellci{0.171}{0.141}{0.206}}
& \smallcellci{0.529}{0.484}{0.580}
& \textbf{\smallcellci{0.169}{0.129}{0.214}}
& 0.195
& 0.156
& \textbf{\smallcellci{0.201}{0.162}{0.241}}
& \textbf{\smallcellci{0.178}{0.148}{0.207}}
& \textbf{\smallcellci{0.172}{0.133}{0.210}}
\\

$\lfloor 0.1 \times N \rfloor$
& 0.418
& \textbf{\smallcellci{0.179}{0.154}{0.207}}
& \smallcellci{0.517}{0.469}{0.581}
& \textbf{\smallcellci{0.174}{0.139}{0.216}}
& 0.195
& 0.156
& \textbf{\smallcellci{0.196}{0.157}{0.228}}
& \textbf{\smallcellci{0.177}{0.142}{0.216}}
& \textbf{\smallcellci{0.190}{0.151}{0.231}}
\\

\midrule
\multicolumn{10}{l}{\textbf{Covariance error}} \\
\midrule
$16$
& 0.943
& \textbf{\smallcellci{0.672}{0.629}{0.715}}
& \smallcellci{0.995}{0.995}{0.996}
& \textbf{\smallcellci{0.664}{0.616}{0.710}}
& 0.795
& 0.000
& \textbf{\smallcellci{0.640}{0.595}{0.685}}
& \smallcellci{1.115}{1.023}{1.204}
& \textbf{\smallcellci{0.695}{0.651}{0.734}}
\\

$\lfloor 0.1 \times N \rfloor$
& 0.943
& \smallcellci{0.975}{0.905}{1.045}
& \smallcellci{0.996}{0.995}{0.996}
& \textbf{\smallcellci{0.672}{0.625}{0.726}}
& 0.799
& 0.000
& \smallcellci{1.006}{0.948}{1.068}
& \smallcellci{1.322}{1.211}{1.438}
& \textbf{\smallcellci{0.748}{0.713}{0.784}}
\\

\bottomrule
\end{tabular}
\endgroup
\end{table*}

\begin{table*}[htbp]
\centering
\caption{
Full results for linear regression experiments with Boston housing data.
See \cref{tab:full_robust_linear_regression} caption for further explanation. 
}
\label{tab:full_robust_linear_regression_cov_real}

\begingroup
\setlength{\tabcolsep}{3pt}
\renewcommand{\arraystretch}{1.08}
\newcommand{\smallcellci}[3]{%
  \begin{tabular}{@{}c@{}}
  #1\\[-1pt]
  {\scriptsize [#2, #3]}
  \end{tabular}%
}

\begin{tabular}{lcccc|ccccc}
\toprule
& \multicolumn{4}{c|}{Log loss}
& \multicolumn{5}{c}{$\beta$-loss ($\beta = 1.5$)} \\
\cmidrule(lr){2-5}\cmidrule(lr){6-10}
$B$
& Posterior
& CT
& LR+WS
& DQ+exact
& NUTS
& Sandwich Gauss
& CT
& LR+WS
& DQ+exact \\
\midrule

\multicolumn{10}{l}{\textbf{Covariance error}} \\
\midrule

$16$
& 0.358
& \textbf{\smallcellci{0.247}{0.194}{0.310}}
& \smallcellci{$9.23{\times}10^{8}$}{$3.62{\times}10^{4}$}{$6.17{\times}10^{9}$}
& \textbf{\smallcellci{0.337}{0.262}{0.405}}
& 2.528
& 0
& \textbf{\smallcellci{2.054}{1.723}{2.328}}
& $\infty$
& \textbf{\smallcellci{2.782}{0.965}{9.328}}
\\

$\lfloor 0.1 \times N \rfloor$
& 0.358
& \smallcellci{0.589}{0.443}{0.804}
& \smallcellci{$1.40{\times}10^{7}$}{$4.85{\times}10^{3}$}{$9.01{\times}10^{7}$}
& \textbf{\smallcellci{0.352}{0.274}{0.441}}
& 2.528
& 0
& \smallcellci{3.126}{2.313}{5.338}
& $\infty$
& \textbf{\smallcellci{1.398}{0.844}{2.132}}
\\

\bottomrule
\end{tabular}
\endgroup
\end{table*}

\begin{table}[htbp]
\centering
\caption{%
Results for Poisson regression experiments. 
See \cref{tab:full_robust_linear_regression} caption for further explanation. 
}
\label{tab:full poisson_glm_and_german}
\begin{tabular}{l l c c c}
\toprule
& & \multicolumn{2}{c}{\textbf{Simulated}} &  \textbf{Credit} \\
\cmidrule(lr){3-4} \cmidrule(lr){5-5}
$B$ & method %
& calib.\ err. & cov.\ err.  & cov.\ err.\\
\midrule
$16$
& CT %
& \textbf{\cellci{0.069}{0.062}{0.075}}
& \textbf{\cellci{0.207}{0.199}{0.215}} 
& \textbf{\cellci{0.132}{0.112}{0.168}} \\
& DQ+const %
& \cellci{0.646}{0.639}{1.344}
& \cellci{0.672}{0.664}{0.678} 
& \cellci{0.982}{0.975}{0.987}  \\
& DQ+exact %
& \textbf{\cellci{0.074}{0.068}{0.080}}
& \textbf{\cellci{0.208}{0.201}{0.217}} 
& \textbf{\cellci{0.157}{0.132}{0.193}}  \\
\midrule
$\lfloor0.1\!\times\!N\rfloor$
& CT %
& \cellci{0.089}{0.078}{0.100}
& \cellci{0.230}{0.218}{0.245} 
& \cellci{0.191}{0.155}{0.240}  \\
& DQ+const %
& \cellci{1.376}{1.370}{1.382}
& \cellci{0.991}{0.990}{0.992} 
& \cellci{0.997}{0.996}{0.999}  \\
& DQ+exact %
& \textbf{\cellci{0.075}{0.066}{0.083}}
& \textbf{\cellci{0.211}{0.203}{0.220}} 
& \textbf{\cellci{0.154}{0.138}{0.181}}  \\
\bottomrule
\end{tabular}

\end{table}

\section{Application: Stationary Covariance for a Fixed Learning Rate}
\label{sec:applications}
In this section, we discuss how our theory can be used to justify the stationary covariance structure at a fixed learning rate.

\subsection{Linear Regression}
\label{section:Linear Regression}

As an illustration of the usefulness of, and new insights provided by \cref{theorem: general case covariance matrix of noise}, we first focus on the special case of linear regression
without regularization (i.e., where $\reg \equiv 0$). 
Since in the case of linear regression the proxy algorithm is identical to the exact algorithm, we will give all our results in terms of the original process $(\theta_{t})_{t \ge 0}$. 
In linear regression we can specialize \cref{eq: covariance of stationary noise} to obtain
\begin{align}
	\begin{aligned}
		\overline{C}_{\theta} 
		&= \textstyle \frac{1}{B} ( N^{-1} \sum_{n=1}^{N} x_{n}x_{n}^{\top}\Sigma_{\theta} x_{n}x_{n}^{\top} -  \hess \Sigma_{\theta} \hess ) + \frac{1}{BN}\sum_{n=1}^N r_{n}^{2}x_{n}x_{n}^{\top},
	\end{aligned}
	\label{eq: SGD noise for linear regression}
\end{align}
where $r_{n} = y_{n} - \MLE^{\top}x_{n}$ is the residual and $\hess = N^{-1} \sum_{n=1}^{N} x_{n}x_{n}^{\top}$. 

\paragraph{Relation to large-sample approximation of \citet{ziyin2021strength}.}
We can recover the approximation given in \cref{eq: sgd noise in well-specified linear model} by making the same simplifying assumptions and approximations (see \cref{section: discussion for well-specified linear model} for details). 
First, if $x_{n} \dist \Norm(0, A)$ and $N$ is large, then, using the properties of the Gaussian,
the first term on the righthand side of \cref{eq: SGD noise for linear regression} is well-approximated by $2 A\Sigma_{\theta}A + \Tr[A\Sigma_{\psi}]A$ and $\hess \approx A$. 
Hence, the first two terms together are approximately equal to $A\Sigma_{\theta}A + \Tr[A\Sigma_{\psi}]A$.
However, in many scenarios the covariates may not be normally distribution (e.g., they may be binary or have heavier tails) and $N$ may not be large relative to 
the parameter/covariate dimension $D$. 
To simplify the final term in \cref{eq: SGD noise for linear regression}, we must also assume the model is well-specified, 
which implies that $x_{n}$ and $r_{n}$ are independent and $r_{n} \dist \Norm(0, \sigma^{2})$.
Hence, when $N$ is large, the final term is approximately $\sigma^{2} A$.
However, when the model is misspecified, the term $N^{-1} \sum_{i=1}^N r_{i}^{2}x_{i}x_{i}^{\top}$ can capture additional variability due to,
for example, a poor model fit, heteroskedastic errors, and/or heavy-tailed errors.
We illustrate this latter point next. 
\paragraph{Numerical illustrations.}

To validate our theory, we compare the predicted stationary covariance structure obtained from combining \cref{thm:SGLD-stationary-covariance,eq: SGD noise for linear regression}
with predictions based on (1) the continuous-time theory and (2) the discrete-time theory that assumes large $N$ and a well-specified model. %
We focus on the effect of varying the (scalar) learning rate. 

\paragraph{\emph{Simulated misspecified data.}}
First, we consider a misspecified simulated dataset with heteroskedastic error generated according to the model 
\begin{align}
	\label{eq: linear regression misspecified model}
	y_{n} \dist \Norm(x_{n}^{\top}\theta_{\star}, 1+\|x_{i}\|_{2}^{2}),
\end{align}
where $\theta_{\star} \dist \Norm(0, I_{D})$ is fixed and $x_{n} \distiid \Norm(0, I_{D})$.
We take $D = 20$ and $N = 2{,}000$. 
\Cref{Fig: linear regression fixed_lr}(left) illustrates the predicted covariance for the parameters $(\theta_{1}, \theta_{2})^{\top}$.
The results show that our theory delivers the most accurate covariance predictions across all learning rate levels. In contrast, the continuous-time theory underestimates the parameter variances, while the discrete-time approximation
that assume $N$ is large and the model is correct overestimates them.

\paragraph{\emph{Boston housing data.}}
Next, we reconsider the real-world Boston housing data
Similar to the results on simulated data, \Cref{Fig: linear regression fixed_lr}(right) demonstrates that our theory can accurately predict the covariance. 
The alternative approximations consistently underestimate it.

\begin{figure}[t]
	\centering
	\begin{minipage}[b]{0.48\textwidth}
		\centering
		\includegraphics[width=\textwidth]{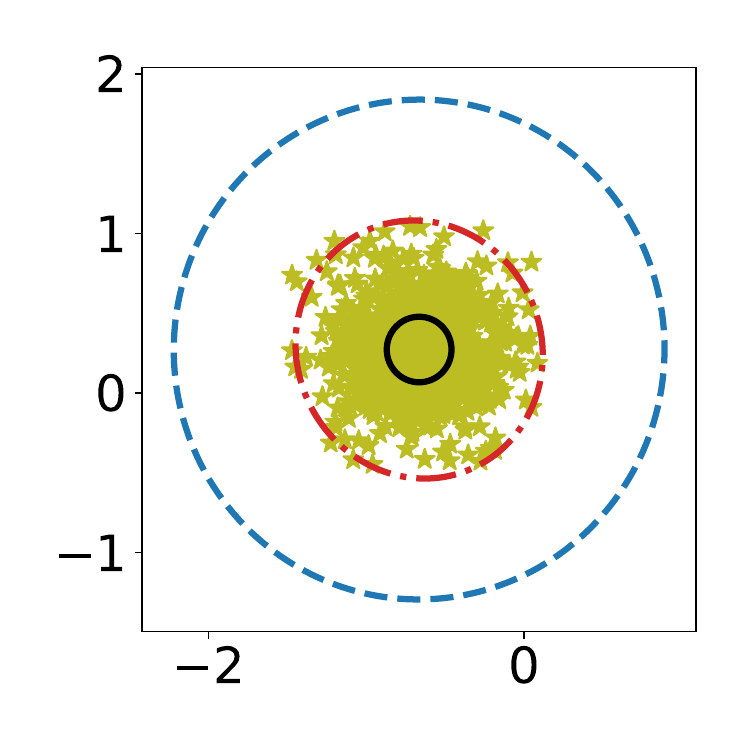} 
	\end{minipage}
	\hspace{0.0001\textwidth}  %
	\begin{minipage}[b]{0.48\textwidth}
		\centering
		\includegraphics[width=\textwidth]{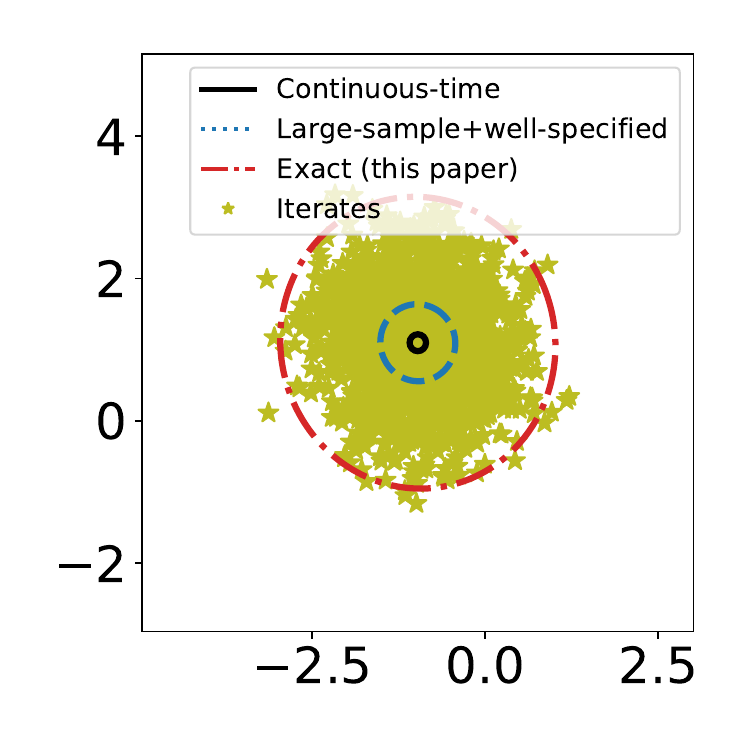} 
	\end{minipage}

	\caption{Comparison of estimated stationary covariance structure for linear regression at $3\sigma$ confidence region on \textbf{(left)} simulated misspecified data with heteroskedastic noise 
		and \textbf{(right)} the classic Boston housing dataset with $\lambda = 0.1$ and $B = 32$.
		Our theory provides more accurate stationary covariance predictions in both cases.
	}\vspace{-1em}
	\label{Fig: linear regression fixed_lr}
\end{figure}

\subsection{Poisson Regression}
\label{sec: poission regression with fixed lr}
Similar to the linear regression experiments, we compare the stationary covariance predicted by our theory with those derived from continuous-time theory and the discrete-time quadratic loss proxy with constant noise (that is, using $\overline{C}_{\psi} \approx \frac{1}{B}\hess$ in \cref{eq: SGLD stationary covariance}).
However, unlike in linear regression, the proxy algorithm is no longer exact, and so we must rely on our error analysis to justify its use.

\begin{figure}[t]
	\centering
	\begin{minipage}[b]{0.48\textwidth}
		\centering
		\includegraphics[width=\textwidth]{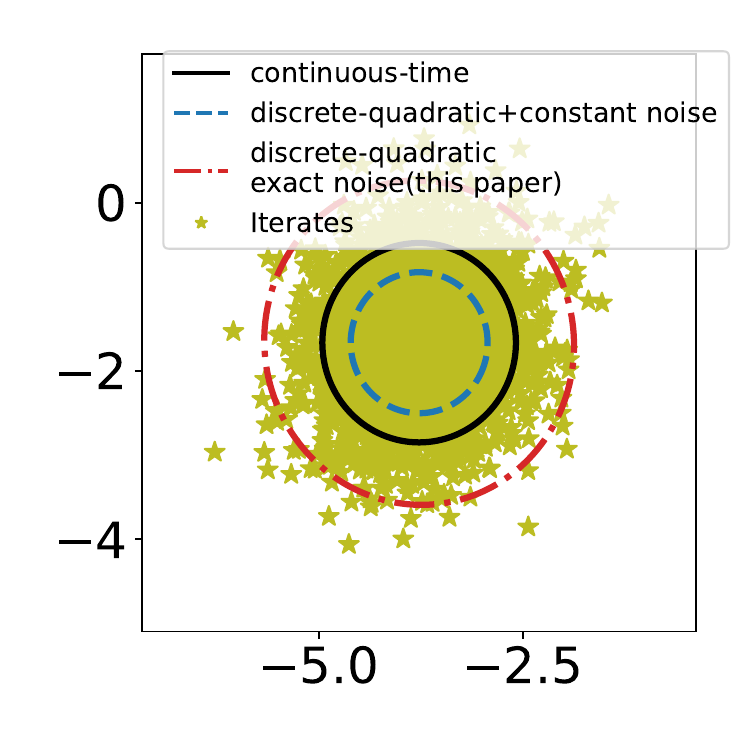} 
	\end{minipage}
	\hspace{0.0001\textwidth}  %
	\begin{minipage}[b]{0.48\textwidth}
		\centering
		\includegraphics[width=\textwidth]{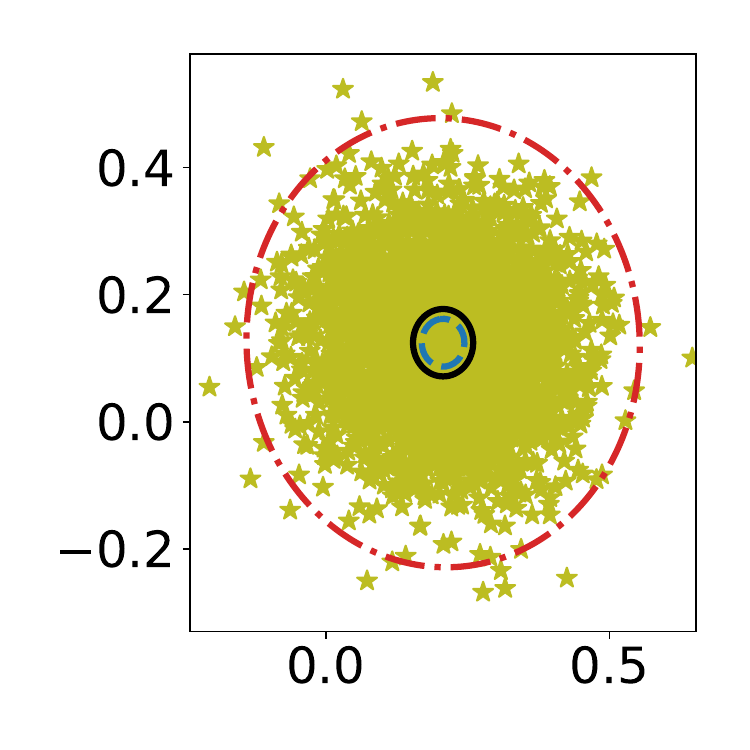} 
	\end{minipage}
	
	\caption{Comparison of estimated stationary covariance structure for Poisson regression at $3\sigma$ confidence region with \textbf{(left)} simulated well-specified data and \textbf{(right)} the German credit data by setting batch size $\lambda = 0.1$, and $B = 32$.
	} %
\label{Fig: Poisson regression fixed_lr}
\end{figure}

\begin{table}[htbp] %
	\centering %
	\begin{tabular}{@{}lccc@{}} %
		\toprule
		Learning Rate $\lambda$ & continuous-time & discrete-quadratic+constant noise & discrete-quadratic+exact noise  \\
		\midrule
		\multicolumn{4}{c}{$\left|\Sigma_\psi-\Sigma_\theta\right|_F$ \textbf{for Poisson regression on well-specified simulated dataset}} \\ %
		\midrule
		0.1   & 0.237 & 0.302 & \textbf{0.030} \\
		0.3   & 0.479 & 0.631 & \textbf{0.096}  \\
		0.5  & 0.545 & 0.651 & \textbf{0.202} \\
		\midrule
		\multicolumn{4}{c}{$\left|\Sigma_\psi-\Sigma_\theta\right|_F$ \textbf{for Poisson regression on misspecified German credit dataset}} \\ %
		\midrule
		0.1   & 0.0367 & 0.037& \textbf{0.004}  \\
		0.3   & 0.098 & 0.099 & \textbf{0.025}  \\
		0.5  & 0.124 & 0.126 & \textbf{	0.041} \\
		\bottomrule
	\end{tabular}
	\caption{Comparison of difference between estimated stationary covariance structure $\Sigma_{\psi}$ and the ground truth using Frobenius norm for Poisson regression.}
	\label{Table: Frobenius norm comparison for Poisson regression}
\end{table}

\Cref{Fig: Poisson regression fixed_lr}(right) shows that our theory provides an accurate estimate of the stationary covariance while alternatives provide severe underestimates. 

For both simulated and real-world dataset, our approximation demonstrates an improvement in accuracy with errors that are 3–10 times smaller than the baseline approaches as shown in \cref{Table: Frobenius norm comparison for Poisson regression}. 

\subsection{Optimal weight decay and batch size}
A direct application of accurate stationary covariance prediction is to estimate the test loss. 
To simplify our analysis, we will focus on linear regression. 
The test loss depends on the stationary covariance by $\loss_{\text{test}} =  \frac{1}{M} \sum_{m=1}^{M} (y_{m}-\MLE^{\top}x_{m})^{2} + \frac{1}{M} \sum_{m=1}^{M}x_{m}^{\top} \Sigma_{\theta} x_{m}$ (see \cref{eq: test loss for linear regression}), where $\{(x_{m}, y_{m})\}_{m=1}^{M}$ is the test dataset.
As illustrated in \cref{Fig: Test loss}, our theory offers the most accurate test loss estimation across different decay weights and batch sizes.
\begin{figure}[t]
	\centering
	\begin{minipage}[b]{0.48\textwidth}
		\centering
		\includegraphics[width=\textwidth]{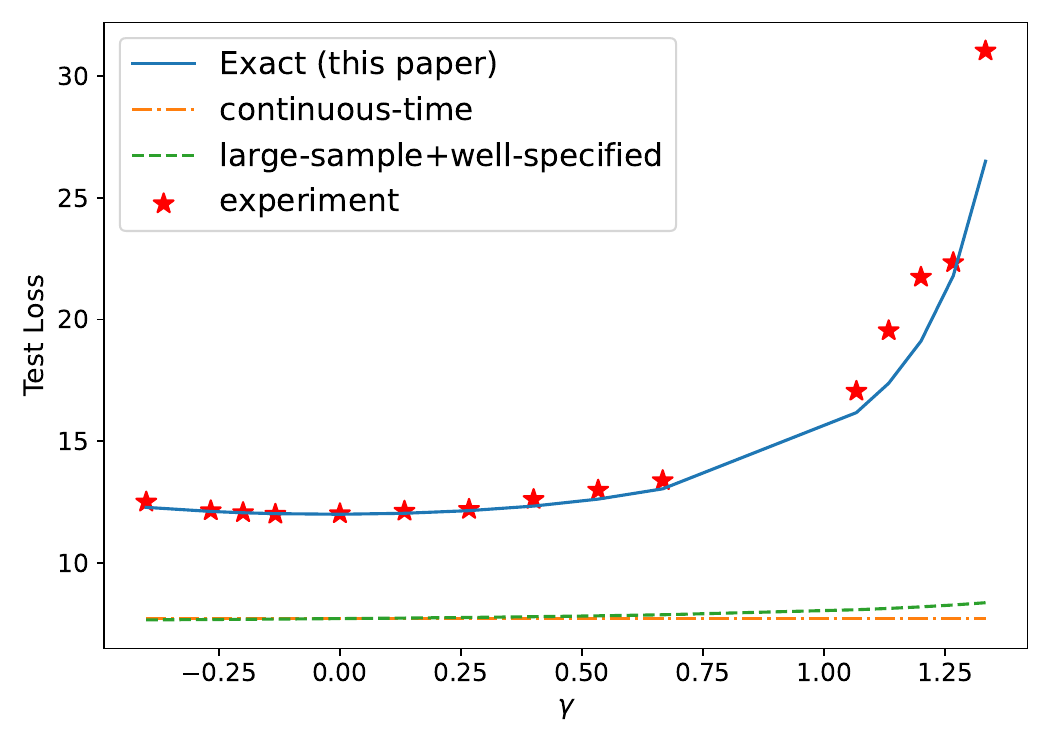} 
	\end{minipage}
	\hspace{0.0001\textwidth}  %
	\begin{minipage}[b]{0.48\textwidth}
		\centering
		\includegraphics[width=\textwidth]{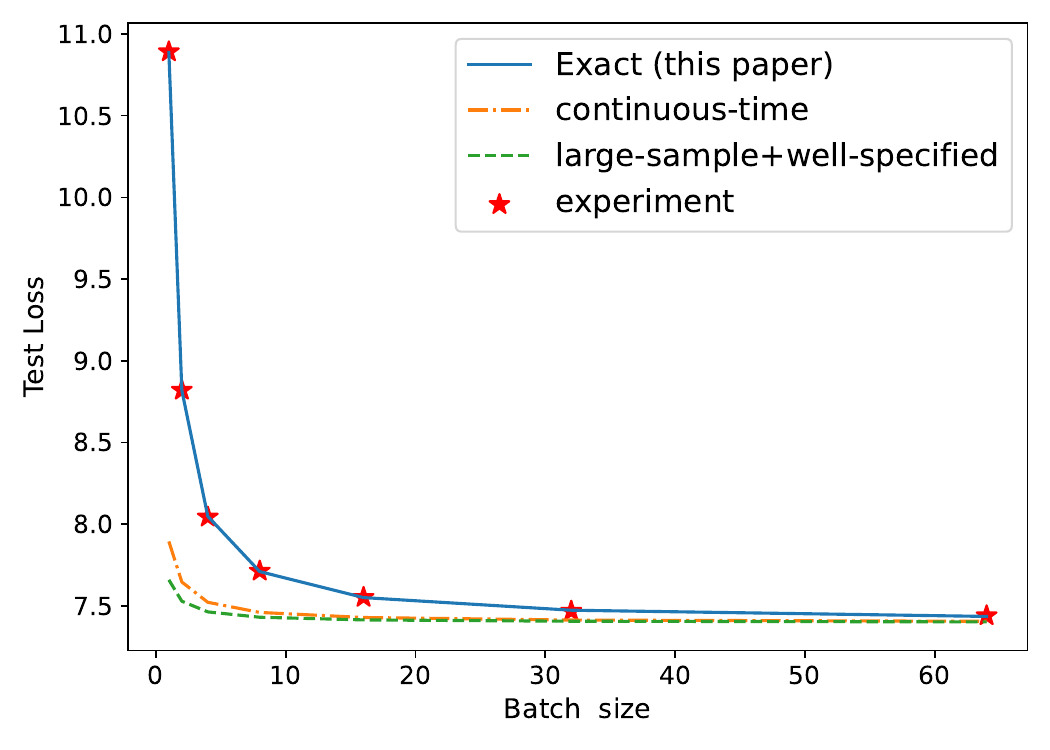} 
	\end{minipage}
	\caption{Comparison of estimated test loss ridge regression ($\Gamma = \gamma I_{D}$) on simulated misspecified data with heteroskedastic noise considered in \cref{eq: linear regression misspecified model}. 
		$\textbf{(left)}$ We set $\lambda = 0.1$, $B = 32$. $\textbf{(right)}$ We set $\lambda = 0.1$, $\gamma = 0$. }
	\label{Fig: Test loss}
\end{figure}

\subsubsection{More discussion about \cref{section:Linear Regression}}
\label{section: discussion for well-specified linear model}
Recall that in linear regression we can specialize \cref{eq: covariance of stationary noise} to obtain

\begin{align}
		\overline{C}_{\theta} 
		&= \frac{1}{B} ( N^{-1} \sum_{n=1}^{N} x_{n}x_{n}^{\top}\Sigma_{\theta} x_{n}x_{n}^{\top} -  \hess \Sigma_{\theta} \hess ) + \frac{1}{BN}\sum_{n=1}^N r_{n}^{2}x_{n}x_{n}^{\top},
\end{align}

where $r_{n} = y_{n} - \MLE^{\top}x_{n}$ is the residual and $\hess = N^{-1} \sum_{n=1}^{N} x_{n}x_{n}^{\top}$. 

Now suppose that the data $\{(x_{n}, y_{n})\}_{n=1}^{N}$ are generated from a linear model, there exists a $\theta_{\star} \in \reals^{D}$ such that $y_{n} = x_{n}^{\top} \theta_{\star} + \epsilon_{n}$, where $\epsilon_{n} \distiid \Norm(0, \sigma^{2})$, for $i=1, 2, ..., N$.
Now we will focus on the MSE loss defined as 
\begin{equation}
	\loss(\theta)  = \frac{1}{N} \sum_{n=1}^{N} \ell(x_{n}, y_{n}, \theta)= \frac{1}{2N \sigma^{2}} \sum_{n=1}^{N}(y_{n}-x_{n}^{\top}\theta)^{2}.
	\label{eq: linear regression MSE loss}
\end{equation}
Note that $\MLE \sim \Norm\left(\theta_{\star}, \sigma^{2}\left(\mathbf{X}^{\top}\mathbf{X}\right)^{-1}\right)$, where $\mathbf{X} \in \reals^{N \times D}$, then we have
\begin{equation}
	\begin{aligned}
		\frac{1}{N} \mathbb{E}\left[\sum_{i=1}^{N}r_{i}^{2}x_{i}x_{i}^{\top} \right]
		&= \frac{1}{N} \sum_{i=1}^{N} \mathbb{E}\left[\left( y_{i} -  x_{i}^{\top}\MLE \right)^{2}\right]x_{i}x_{i}^{\top}\\
		& = \frac{1}{N} \sum_{i=1}^{N}\left(  \mathbb{E}\left[\left( y_{i} - \mathbb{E}\left[ y_{i}\right]\right)^{2} \right] +  \mathbb{E}\left[\left(  x_{i}^{\top} \MLE  - \mathbb{E}\left[ y_{i}\right]\right)^{2} \right]\right)x_{i}x_{i}^{\top}\\
		& = \frac{1}{N} \sum_{i=1}^{N} \sigma^{2}\left(I + \left(\mathbf{X}^{\top}\mathbf{X}\right)^{-1}\right)x_{i}x_{i}^{\top}\\
		& = \sigma^{2}\left( A + \frac{1}{N} I \right).
	\end{aligned}
\end{equation}
Then we have
\begin{align}
	\lim_{N \ to \infty}\mathbb{E}\left[\sum_{i=1}^{N}r_{i}^{2}x_{i}x_{i}^{\top} \right] = \sigma^{2} A. 
\end{align}
Under the assumptions of $x_{n} \dist \Norm(0, A)$ and $N$ being large, we have
\begin{align}
	\lim_{N \to \infty }\frac{1}{N} \sum_{n=1}^{N} x_{n}x_{n}^{\top}\Sigma_{\theta} x_{n}x_{n}^{\top} -  \hess \Sigma_{\theta} \hess = A\Sigma_{\theta}A + \Tr[A\Sigma_{\psi}]A.
\end{align}
Then, we will get exactly the same result of Lemma 1 in \citet{ziyin2021strength}. 

\subsubsection{Test Loss of Linear Regression}
\label{eq: test loss for linear regression}
The test loss in linear regression can be decomposed as follows:
\begin{align}
\loss_{\text{test}}
&=
\mathbb{E}\!\left[
\frac{1}{M}\sum_{m=1}^{M}
\bigl(y_m - \theta_t^{\top} x_m\bigr)^2
\right] \\
&=
\mathbb{E}\!\left[
\frac{1}{M}\sum_{m=1}^{M}
\bigl(
y_m - \hat{\theta}^{\top} x_m
+
(\hat{\theta} - \theta_t)^{\top} x_m
\bigr)^2
\right] \\
&=
\frac{1}{M}\sum_{m=1}^{M}
\bigl(y_m - \hat{\theta}^{\top} x_m\bigr)^2
+
\frac{1}{M}\sum_{m=1}^{M}
x_m^{\top}
\mathbb{E}\!\left[
(\theta_t - \hat{\theta})(\theta_t - \hat{\theta})^{\top}
\right]
x_m \\
&=
\frac{1}{M}\sum_{m=1}^{M}
\bigl(y_m - \hat{\theta}^{\top} x_m\bigr)^2
+
\frac{1}{M}\sum_{m=1}^{M}
x_m^{\top} \Sigma_{\theta} x_m,
\end{align}
where $\Sigma_{\theta} = \mathbb{E}\!\left[(\theta_t - \hat{\theta})(\theta_t - \hat{\theta})^{\top}\right]$.

\section{Proofs from Main Text}
\label{section: Proofs}

\begin{lemma}
	Assume the parameters $\psi$ are updated based on discrete-time proxy algorithm \cref{eq: proxy algorithm framework for general loss}, $\reg(\psi) = \frac{1}{2}\psi^{\top} \Gamma \psi$, and the stationary distribution of $\psi$ exists, then the stationary mean $\mu_{\psi}$ satisfies $\mu_{\psi} = \MLE$.
	If the parameters $\psi$ are updated based on discrete-time proxy algorithm \cref{eq: proxy algorithm framework for general loss}, and the stationary distribution of $\psi$ exists, then the stationary mean $\mu_{\psi}$ satisfies $\mu_{\psi} = \MLE$.
\end{lemma}
\begin{proof}
	Now we assume that $\psi_{0}$ are sampled from the stationary distribution.
	Then by taking expectation we have
	\begin{align}
		\begin{aligned}
			\mu_{\psi} = \mathbb{E} \left[ \psi_{t} \right]& = \mathbb{E} \left[\psi_{t-1} - \Lambda \left\lbrace 	G_{t}(\MLE) +   \grad G_{t}(\MLE)\left( \psi_{t-1}-\MLE\right)\right\rbrace+\sqrt{2\beta^{-1} \Lambda}\,\xi_{t-1} \right]\\
			& = \mu_{\psi} - \Lambda \mathbb{E}(	G_{t}(\MLE))- \Lambda\mathbb{E}\left(\grad G_t(\hat{\theta})(\psi_{t-1}-\hat{\theta})\right)
		\end{aligned}
	\end{align}
        Then we have
        \begin{align}
             \Lambda \mathbb{E}(	G_{t}(\MLE))+\Lambda (\mathcal{J}+\frac{1}{N}\Gamma)(\mu_\psi-\hat{\theta})= 0
        \end{align}
        Since $\MLE$ satisfies $\nabla \loss(\MLE) = 0$, we have 
        	\begin{align}
		\mathbb{E}(G_t(\hat{\theta)})=\frac{\Gamma}{N} \MLE + \frac{1}{N}\sum_{n=1}^{N} \nabla \ell\left(x_{n}, y_{n}, \MLE\right) = 0.
		\label{eq: property of MLE}
	\end{align}
        Therefore, combining the previous two displayed equations and if , we conclude that $\mu_{\psi} = \MLE$.

\end{proof}

\subsection{Proof of \cref{thm:SGLD-stationary-covariance}}
\begin{proof}
	The proxy algorithm leads to the discrete-time update
        \begin{align}
        \psi_{t} = \psi_{t-1}
        - \Lambda
          \Bigl[
            G_{t}(\MLE)
            + \nabla G_{t}(\MLE)(\psi_{t-1} - \MLE)\,
          \Bigr]
        + \sqrt{2\beta^{-1}\Lambda}\,\xi_{t-1},
        \end{align}
	where $\xi_{t-1} \sim \Norm(0, I)$.	
	Let
        \begin{align}
        \eta_{t-1}
        = \Bigl[G_{t}(\MLE)
          + \nabla G_{t}(\MLE)(\psi_{t-1} - \MLE)\,\Bigr]
        - \E\Bigl[G_{t}(\MLE)
          + \nabla G_{t}(\MLE)(\psi_{t-1} - \MLE)\,\Bigr],
        \end{align}
	it follows from \cref{eq: property of MLE} that $\text{Cov}(\eta_{t-1}) = \overline{C}_{\psi}$. 
	Noting that $\mathcal{J} = \frac{1}{N} \sum_{n=1}^{N}\mathcal{J}_{n}$, we have
        \begin{align}
        \hess = \E\bigl[\nabla G_{t}(\MLE)\bigr]=\mathcal{J} + \frac{1}{N} \Gamma.
        \end{align}
        
	\cref{eq: proxy discrete-time update for general loss} can also be rewritten as 
	\begin{align}
		\begin{aligned}
			\psi_{t}-\MLE
			& =  \psi_{t-1}-\MLE - \Lambda \Bigl[G_{t}(\MLE)+ \nabla G_{t}(\MLE)(\psi_{t-1} - \MLE) \Bigr]+ \sqrt{2\beta^{-1} \Lambda}\,\xi_{t-1}\\
			& =  \psi_{t-1}-\MLE - \Lambda \E\Bigl[G_{t}(\MLE)
          + \nabla G_{t}(\MLE)(\psi_{t-1} - \MLE)\,\Bigr] -\Lambda\eta_{t-1} + \sqrt{2\beta^{-1} \Lambda}\,\xi_{t-1}\\
			& = \psi_{t-1} -\MLE- \Lambda \left(\frac{1}{N} \sum_{n=1}^{N}   \mathcal{J}_{n}-\frac{1}{N}\Gamma \right)(\psi_{t-1}-\MLE) - \Lambda \eta_{t-1}  + \sqrt{2\beta^{-1} \Lambda}\,\xi_{t-1}\\
			& = \left( I - \Lambda \mathcal{J}-\frac{1}{N} \Lambda \Gamma\right) \left( \psi_{t-1} -\MLE\right) - \Lambda \eta_{t-1} + \sqrt{2\beta^{-1} \Lambda}\,\xi_{t-1}\\
			& = \left( I - \Lambda \hess \right) \left( \psi_{t-1} -\MLE\right) - \Lambda \eta_{t-1} + \sqrt{2\beta^{-1} \Lambda}\,\xi_{t-1}.
		\end{aligned}
	\end{align}
	Note that
	\begin{align}
		\begin{aligned}
			\Sigma_{\psi} &= \mathbb{E} \left[ \left( \psi_{t}-\MLE\right) \left( \psi_{t}-\MLE\right) ^{\top}\right]\\
			& = \left( I - \Lambda \hess \right)  \Sigma_{\psi} \left( I - \Lambda \hess \right)^{\top}+ \Lambda \overline{C}_{\psi} \Lambda +  \frac{2\Lambda}{\beta}.
		\end{aligned}
	\end{align}
	Then, after some algebra, we have
	\begin{align}
		\Lambda \hess  \Sigma_{\psi} +  \Sigma_{\psi}  \hess \Lambda=\Lambda \left(\overline{C}_{\psi}+ \hess  \Sigma_{\psi} \hess \right)\Lambda + \frac{2\Lambda}{\beta}.
		\label{eq:relationship for general sgld}
	\end{align}
\end{proof}

\subsection{Proof of \cref{theorem: general case covariance matrix of noise}}

The covariance of the gradient noise for parameter $\psi$ is given by %
\[
C(\psi)= \begin{cases}
	\frac{1}{B}\left[\frac{1}{N} \sum_{n=1}^N \nabla \tilde{\ell}_{n}(\psi) \nabla \tilde{\ell}_{n}(\psi)^{\top}-\nabla \tilde\loss(\psi) \nabla \tilde\loss(\psi)^{\top}\right] & \text {if with replacement} \\ 
	\frac{N-B}{B(N-1)}\left[\frac{1}{N} \sum_{n=1}^N \nabla \tilde{\ell}_{n}(\psi) \nabla \tilde{\ell}_{n}(\psi)^{\top}-\nabla \tilde\loss(\psi) \nabla \tilde\loss(\psi)^{\top}\right] & \text {if without replacement.}\end{cases}
\label{eq: SGD noise covariance for both sampling with and without replacement}
\]
We focus on sampling with replacement since our results can be easily extended to the sampling without replacement case by substituting each 
$C(\psi)$ term with $\frac{N-B}{N-1}C(\psi)$.

We have 
\begin{align}
    \begin{aligned}
        C(\psi_{t-1}) & =\frac{1}{N B} \sum_{n=1}^{N} \nabla \ell_{n}\left(\psi_{t-1}\right) \nabla \ell_{n}\left(\psi_{t-1}\right)^{\top}-\frac{1}{B} \nabla \mathcal{L}\left(\psi_{t-1}\right) \nabla \mathcal{L}\left(\psi_{t-1}\right)^{\top} \\
        & =  \underbrace{\frac{1}{B}\frac{1}{N} 
        \sum_{n=1}^{N} \left[\nabla\ell\left( x_n,y_n, \MLE\right)+ \mathcal{J}_n \left( \psi_{t-1} - \MLE\right)\right] \left[\nabla\ell\left( x_n,y_n, \MLE\right)+ \mathcal{J}_n \left( \psi_{t-1} - \MLE\right)\right]^{\top}}_{C_{3}(\psi_{t-1})} \\
        &\quad - \underbrace{\frac{1}{B} \left[ \frac{1}{N}\sum_{n=1}^{N}\nabla\ell\left( x_n,y_n, \MLE\right)+ \mathcal{J}_n \left( \psi_{t-1} - \MLE\right)\right]\left[ \frac{1}{N}\sum_{n=1}^{N}\nabla\ell\left( x_n,y_n, \MLE\right)+ \mathcal{J}_n \left( \psi_{t-1} - \MLE\right)\right]^{\top}}_{C_{4}(\psi_{t-1})}.
    \end{aligned}
\end{align}
Note that
\begin{align}
    \begin{aligned}
        \mathbb{E}\left[C_{3}(\psi_{t-1}) \right] &= \frac{1}{BN} \sum_{n=1}^{N}  \left[\nabla\ell\left(x_{n}, y_{n}, \MLE\right) \right] \left[\nabla\ell\left(x_{n}, y_{n}, \MLE\right) \right]^{\top} + \frac{1}{BN} \sum_{n=1}^{N} \mathcal{J}_{n} \mathbb{E}\left[\left(\psi_{t-1}-\MLE\right)\left(\psi_{t-1}-\MLE\right)^{\top}\right] \mathcal{J}_{n}^{\top}\\
        & = \frac{1}{BN} \sum_{n=1}^{N}  \left[\nabla\ell\left(x_{n}, y_{n}, \MLE\right) \right] \left[\nabla\ell\left(x_{n}, y_{n}, \MLE\right) \right]^{\top} + \frac{1}{BN} \sum_{n=1}^{N} \mathcal{J}_{n} \Sigma_{\psi} \mathcal{J}_{n}\\
        & = \frac{1}{B} \mathcal{I} + \frac{1}{BN} \sum_{n=1}^{N} \mathcal{J}_{n} \Sigma_{\psi} \mathcal{J}_{n}.
    \end{aligned}
\end{align}
Also note that, using \cref{eq: property of MLE},
\begin{align}
    \begin{aligned}
        \mathbb{E} \left[ C_{4}(\psi_{t-1}) \right] 
        &= \frac{1}{B} \left\lbrace  \left(\frac{1}{N}\sum_{n=1}^{N}  \nabla\ell\left(x_{n}, y_{n}, \MLE\right)  \right) \left(\frac{1}{N}\sum_{n=1}^{N}  \nabla\ell\left(x_{n}, y_{n}, \MLE\right)  \right)^{\top} + \mathcal{J}\mathbb{E}\left[\left(\psi_{t-1}-\MLE\right)\left(\psi_{t-1}-\MLE\right)^{\top}\right]\mathcal{J}\right\rbrace  \\
        & =  \frac{1}{B} \left(\frac{1}{N^{2}} \Gamma \MLE \MLE^{\top} \Gamma^{\top} + \mathcal{J} \Sigma_{\psi}\mathcal{J}\right).
    \end{aligned}
\end{align}
Therefore, we have
\begin{align}
    \begin{aligned}
        \overline{C}_{\psi} &= \mathbb{E}\left[C_{3}(\psi_{t-1}) \right] -\mathbb{E} \left[ C_{4}(\psi_{t-1}) \right]\\
        & = \frac{1}{B} \left( \mathcal{I} - \frac{1}{N^{2}} \Gamma \MLE \MLE^{\top} \Gamma^{\top} + \frac{1}{N} \sum_{n=1}^{N} \mathcal{J}_{n} \Sigma_{\psi} \mathcal{J}_{n} -  \mathcal{J} \Sigma_{\psi}\mathcal{J}\right).
    \end{aligned}
    \label{eq: stationary noise for general loss}
\end{align}

\subsection{Proof of \cref{prop: mixing time}}
\label{sec: Proof of prop: mixing time}
We analyze the mixing behavior under the proxy dynamics \cref{eq: proxy algorithm framework for general loss}.
For each coordinate projection $f_i(\theta) \defas \theta_i$, the theoretical lag-$k$ autocorrelation is defined as
\begin{align}
\rho_{k,i}
&\defas
\mathrm{Corr}_{\pi_\theta}\!\bigl(\theta_{0,i}, \theta_{k,i}\bigr) \nonumber\\
&=
\frac{\mathrm{Cov}_{\pi_\theta}(\theta_{0,i}, \theta_{k,i})}
{\mathrm{Var}_{\pi_\theta}(\theta_{0,i})}
=
\frac{\E_{\pi_\theta}\!\left[
(\theta_{0,i}-\MLE_i)(\theta_{k,i}-\MLE_i)
\right]}
{(\Sigma_\psi)_{ii}} \nonumber\\
&=
\frac{\bigl(
\E_{\pi_\theta}\!\left[
(\theta_0-\MLE)(\theta_k-\MLE)^\top
\right]
\bigr)_{ii}}
{(\Sigma_\psi)_{ii}} .
\end{align}

Under the proxy update \cref{eq: proxy discrete-time update for general loss}, the iterates satisfy
\begin{align}
\psi_t - \MLE
=
( I - \Lambda \hess ) (\psi_{t-1}-\MLE)
- \Lambda \eta_{t-1}
+ \sqrt{2\beta^{-1}\Lambda}\,\xi_{t-1},
\end{align}
where $\xi_{t-1} \sim \Norm(0,I)$ and
\begin{align}
\eta_{t-1}
&=
\Bigl[
G_t(\MLE)
+ \nabla G_t(\MLE)(\psi_{t-1}-\MLE)
\Bigr]
-
\E\Bigl[
G_t(\MLE)
+ \nabla G_t(\MLE)(\psi_{t-1}-\MLE)
\Bigr].
\end{align}
Iterating forward,
\begin{align}
\psi_{t+k}-\MLE
=
( I - \Lambda \hess ) (\psi_{t+k-1}-\MLE)
- \Lambda \eta_{t+k-1}
+ \sqrt{2\beta^{-1}\Lambda}\,\xi_{t+k-1}.
\end{align}

Define the lag-$k$ cross-covariance matrix
\begin{align}
\Xi_k
&\defas
\E_{\pi_\psi}\!\left[
(\psi_{t+k}-\MLE)(\psi_t-\MLE)^\top
\right].
\end{align}
Then
\begin{align}
\Xi_k
&=
( I - \Lambda \hess ) \Xi_{k-1}
- \Lambda
\E_{\pi_\theta}\!\left[
\eta_{t+k-1}(\psi_t-\MLE)^\top
\right].
\end{align}
Since $\eta_{t+k-1}$ is conditionally mean-zero given the past,
\begin{align}
\E_{\pi_\theta}\!\left[
\eta_{t+k-1}(\psi_t-\MLE)^\top
\right]
=
\E\!\left[
\E\!\left[
\eta_{t+k-1}(\psi_t-\MLE)^\top \mid \mathcal{F}_{t+k-1}
\right]
\right]
= 0.
\end{align}
Therefore,
\begin{align}
\Xi_k
=
( I - \Lambda \hess )^k \Xi_0
=
( I - \Lambda \hess )^k \Sigma_\psi .
\end{align}

Consequently, the lag-$k$ autocorrelation for coordinate $i$ is
\begin{align}
\rho_{k,i}
=
\frac{\bigl(( I - \Lambda \hess )^k \Sigma_\psi\bigr)_{ii}}
{(\Sigma_\psi)_{ii}} .
\end{align}

When $\Lambda=\lambda I$ and $\widehat H$ is symmetric positive definite, the eigenvalues of
$I-\lambda \widehat H$ are $1-\lambda \mu_i(\widehat H)$. Hence
\[
\|I-\lambda \widehat H\|_2
=
\max_i |1-\lambda \mu_i(\widehat H)|.
\]
According to the condition $0<\lambda<2/\mu_{\max}(\widehat H)$, we have
$0<\lambda\mu_i(\widehat H)<2$ for all eigenvalues $\mu_i(\widehat H)$, and therefore
\[
|1-\lambda\mu_i(\widehat H)|<1.
\]
Since $\widehat H$ is symmetric positive definite,
the spectral norm satisfies
\[
\|A\|_2 \defas \sup_{\|x\|_2=1}\|Ax\|_2
= \max_i |\mu_i(A)|.
\]
Hence,
\[
\|I-\lambda \widehat H\|_2
=
\max_i |1-\lambda\mu_i(\widehat H)|
<1.
\]
Therefore, the Neumann series converges and
\begin{align}
\sum_{k=0}^{\infty}(I-\lambda\widehat H)^k
=
(\lambda\widehat H)^{-1}.
\end{align}
Hence,
\begin{align}
\sum_{k=1}^{\infty}\rho_{k,i}
&=
\frac{
\bigl(( (\Lambda\hess)^{-1}-I )\Sigma_\psi\bigr)_{ii}
}
{(\Sigma_\psi)_{ii}}
=
\frac{\bigl((\Lambda\hess)^{-1}\Sigma_\psi\bigr)_{ii}}
{(\Sigma_\psi)_{ii}}
-1 .
\end{align}

The coordinate-wise integrated autocorrelation time
\[
\tau_{\mathrm{int}}(f_i)
\defas
1 + 2\sum_{k=1}^{\infty}\rho_{k,i}
\]
thus satisfies
\begin{align}
\tau_{\mathrm{int}}(f_i)
=
2\frac{\bigl((\Lambda\hess)^{-1}\Sigma_\psi\bigr)_{ii}}
{(\Sigma_\psi)_{ii}}
-1 .
\end{align}

Let $w\defas \Sigma_\psi^{1/2}v$. Then
\[
\frac{v^\top(\Lambda\hess)^{-1}\Sigma_\psi v}{v^\top\Sigma_\psi v}
=
\frac{w^\top \Bigl(\Sigma_\psi^{1/2}(\Lambda\hess)^{-1}\Sigma_\psi^{-1/2}\Bigr) w}{w^\top w}.
\]
The matrix
\[
M \defas \Sigma_\psi^{1/2}(\Lambda\hess)^{-1}\Sigma_\psi^{-1/2}
\]
is similar to $(\Lambda\hess)^{-1}$, hence $\mu_{\max}(M)=\mu_{\max}((\Lambda\hess)^{-1})
=1/\mu_{\min}(\Lambda\hess)$.

By Rayleigh--Ritz,
\[
\sup_{v\neq 0}\frac{v^\top(\Lambda\hess)^{-1}\Sigma_\psi v}{v^\top\Sigma_\psi v}
=
\sup_{w\neq 0}\frac{w^\top M w}{w^\top w}
=
\lambda_{\max}(M)
=
\frac{1}{\lambda_{\min}(\Lambda\hess)}.
\]
Then we have
\[
\tau \defas \sup_v \tau_{\mathrm{int}}(f_v)
=
2\cdot \frac{1}{\mu_{\min}(\Lambda\hess)} - 1.
\]

\subsection{Proof of \cref{thm:wasserstein-error-bound}}

Since there exists a coupling of $\theta_{0} \dist \nu$ and $\psi_{0} \dist \nu'$ such that $W_{2}^{2}(\nu, \nu') = \E(\|\theta_{0} - \psi_{0}\|^{2})$,
we assume $(\theta_{0}, \psi_{0})$ follow this joint distribution.
Using the recursions for $\theta_{t}$ and $\psi_{t}$, and using the assumption that $\Lambda = \lambda I$, we have
\begin{equation}
\begin{aligned}
    	 \left\lVert\theta_{t} - \psi_{t}\right\rVert^{2} 
	&=  \left\lVert \theta_{t-1} - \psi_{t-1}\right\rVert^{2} 
         + \underbrace{\left\lVert 
    \lambda\,\Bigl[
      G_{t}(\theta_{t-1})
      - G_{t}(\MLE)
      - \nabla G_{t}(\MLE)(\psi_{t-1}-\MLE)\,
    \Bigr]\right\rVert^{2}}_{\star} \\
	&\phantom{=~~}  \underbrace{-~2\lambda\left\langle \theta_{t-1} - \psi_{t-1},  G_{t}(\theta_{t-1})
      - G_{t}(\MLE)
      - \nabla G_{t}(\MLE)(\psi_{t-1}-\MLE)\right\rangle}_{\star\star}
\end{aligned}
\end{equation}
Let $(\mathcal{F}_{t})_{t \ge 0}$ denote the filtration associated with $\{(\theta_{t}, \psi_{t})\}_{t \ge 0}$ and $\E_{t} \defas \E( \cdot \mid \mathcal{F}_{t})$. 
Let $I$ denote an independent random variable uniformly distributed on $\{1,\dots,N\}$.
We can bound the expected squared error as 
\begin{align}
\E_{t-1}(\star) 
&= \lambda^{2} \E_{t-1}\left[ \left\lVert  G_{t}(\theta_{t-1})
      - G_{t}(\MLE)
      - \nabla G_{t}(\MLE)(\psi_{t-1}-\MLE)\right\rVert^{2} \right]\\
&\leq  \frac{\lambda^{2}}{B} \sum_{n \in S_{t}}\E_{t-1}\left[ \left\lVert  \left(\nabla \ell_{n}(\theta_{t-1}) -\nabla \ell_{n}\big(\MLE\big) -\mathcal{J}_{n} \big(\psi_{t-1}-\MLE\big)\right) \right\rVert^{2} \right]\\
& =  \lambda^{2} \E_{t-1}\left[ \left\lVert  \left(\nabla \ell_{I}(\theta_{t-1}) -\nabla \ell_{I}\big(\MLE\big) -\mathcal{J}_{I} \big(\psi_{t-1}-\MLE\big)\right) \right\rVert^{2} \right]\\
&\leq 2\lambda^{2} \E_{t-1}\left[ \left\lVert \nabla \ell_{I}(\theta_{t-1}) - \nabla \ell_{I}(\psi_{t-1})\right\rVert^{2}\right] 
 +  2\lambda^{2} \E_{t-1}\left[ \left\lVert \nabla \ell_{I}(\psi_{t-1}) -\nabla \ell_{I}\big(\MLE\big) -\mathcal{J}_{I} \big(\psi_{t-1}-\MLE\big) \right\rVert^{2}\right].
\end{align}
It follows from Taylor's remainder theorem and Assumption \ref{assump:statistic-smooth} that 
\begin{align}
\left\lVert \nabla \ell_{n}(\psi_{t-1}) -\nabla \ell_{n}\big(\MLE\big) -\mathcal{J}_{n} \big(\psi_{t-1}-\MLE) \right\rVert^{2}
&\le \frac{M_{n}^{2}}{4}\norm{\psi_{t-1}-\MLE}^{4}. 
\end{align}
Using the fact that convexity and $L$-smoothness imply $L$-co-coercivity, we thus obtain 
\begin{align}
\lefteqn{\E_{t-1}(\star)} \\
& \leq 2L \lambda^{2} \E_{t-1}\left[ \left\langle \theta_{t-1}-\psi_{t-1}, \nabla \ell_{I}(\theta_{t-1}) - \nabla \ell_{I}(\psi_{t-1}) \right\rangle \right] + \frac{\lambda^{2}}{2} \E_{t-1}\left[M_{I}^{2} \right] \norm[\big]{\psi_{t-1} - \MLE}^{4}\\
& = 2L \lambda^{2} \left\langle \theta_{t-1}-\psi_{t-1}, \grad\loss(\theta_{t-1}) - \grad\loss(\psi_{t-1}) \right\rangle +  \frac{\lambda^{2} \overline{M^2}}{2} \norm[\big]{\psi_{t-1} - \MLE}^{4}. 
\end{align}
Furthermore, for any $c > 0$, and again using Taylor's remainder theorem and Assumption \ref{assump:statistic-smooth}, we have 
\begin{equation}
	\begin{aligned}
\E_{t-1}(\star\star) 
& = -2\lambda\E_{t-1}\left[\left< \theta_{t-1} - \psi_{t-1}, \left(\nabla \ell_{I}(\theta_{t-1}) -\nabla \ell_{I}(\MLE) -\mathcal{J}_{I} \big(\psi_{t-1}-\MLE\big)\right) \right>\right]\\
&= -2\lambda\left< \theta_{t-1} - \psi_{t-1}, \grad\loss(\theta_{t-1}) - \grad\loss(\MLE) -\mathcal{J}\left(\psi_{t-1} - \MLE\right) \right> \\
&\le  -2\lambda\left<\theta_{t-1} - \psi_{t-1}, \grad\loss(\theta_{t-1}) - \grad\loss(\psi_{t-1})\right> + \lambda \overline{M}\norm{\theta_{t-1} - \psi_{t-1}} \norm[\big]{\psi_{t-1} - \MLE}^{2} \\ 
&\le -2\lambda\left<\theta_{t-1} - \psi_{t-1}, \grad\loss(\theta_{t-1}) - \grad\loss(\psi_{t-1})\right> 
    + 2\lambda c \norm{\theta_{t-1} - \psi_{t-1}}^{2} + \frac{\lambda \overline{M}^2}{8c} \norm[\big]{\psi_{t-1} - \MLE}^{4}. 
\end{aligned}
\end{equation}
Thus, using Assumption \ref{assump:loss} and choosing $c = \mu/2$,
we have
\begin{align}
\E\rbra{\|\theta_{t} - \psi_{t}\|^{2}} 
&\le \rbra[\big]{1 - 2\lambda\mu + 2\lambda c + 2\lambda^{2}\mu L}\E\rbra{\|\theta_{t-1} - \psi_{t-1}\|^{2}}
    + \cbra*{\frac{\lambda^{2} \overline{M^2}}{2} + \frac{\lambda  \overline{M}^2}{8c}}\E\rbra*{\norm[\big]{\psi_{t-1} - \MLE}^{4}} \\
& = \cbra[\big]{1 - \lambda\mu \rbra*{1 - 2\lambda L}}\E\rbra{\|\theta_{t-1} - \psi_{t-1}\|^{2}}+ \lambda  \cbra*{\frac{\lambda  \overline{M^2}}{2} + \frac{\overline{M}^2}{4\mu}}\E\rbra*{\norm[\big]{\psi_{t-1} - \MLE}^{4}}.
\end{align}
Hence, we obtain the overall bound given in \cref{eq:wasserstein-error-bound}.

\subsection{Proof of \cref{cor:stationary-wasserstein-error-bound}}

First, we give a lemma bounding the stationary fourth moment of $\psi_{t-1}$.
\begin{lemma}\label{lem:fourth-moment-bound}
    Under the conditions of \cref{cor:stationary-wasserstein-error-bound}, if $\hat{\mu}$ and $\hat{L}$ denote, respectively, the smallest and largest eigenvalues of $\hess = \grad^2 \loss(\MLE)$ and $\lambda \leq \min\{1/(4\hat{\mu}), B\hat{\mu}/(200 L^2)\}$, then for $\psi_\infty \dist \pi_\psi$, satisfies
\begin{align}
\E(\norm{\psi_\infty - \MLE}^4 )
\le  
  96\frac{\lambda^2\tau_4^4}{\hat{\mu}^2B^2} \, 
  + 24 \frac{\lambda \tau_4^2}{\hat{\mu}^2 B \beta}D
  + 12 \frac{D^2}{\hat{\mu}^2\beta^2}
+48\frac{\lambda D(D+2)}{\hat{\mu}\beta^{2}},
\end{align}
where $\tau_4^4 \defas N^{-1}\sum_{n=1}^N \norm[\big]{\grad \ell(x_{I}, y_{I}, \MLE)}^{4}$.
\end{lemma}

\begin{proof}
The recursion for $\psi_t$ can be rewritten as
\begin{equation}\label{eq:schemeB-et}
\psi_t - \MLE
=(I-\lambda\hess)(\psi_{t-1}-\MLE) - \lambda \eta_{t-1}
+ \sqrt{2\beta^{-1}\lambda}\,\xi_{t-1},
\end{equation}
where
$\eta_{t-1}\defas G_t(\MLE) + \nabla G_t(\MLE)(\psi_{t-1}-\MLE) - \hess(\psi_{t-1}-\MLE)$ and $\xi_{t}\sim \mathcal{N}(0,I_D)$.

Since the minibatch at time $t$ is independent of $\psi_{t-1}$ and
$\E[G_t(\MLE)]=\nabla \loss(\MLE)=0$, $\E[\nabla G_t(\MLE)]=\nabla^2\loss(\MLE)=\hess$,
we have $\E_{t-1}[\eta_{t-1}]=0$.
Using the multinomial formula and the fact that the expected gradient is zero at $\MLE$, we obtain
\begin{align}\label{eq: second-fourth moment of Gt}
\E\cbra{\norm{G_{t}(\MLE)}^{4}}
&\le \frac{3}{B^{2}}\underbrace{\E\cbra*{\norm[\big]{\grad \ell(x_{I}, y_{I}, \MLE)}^{4}}}_{\tau_{4}^{4}\,\defas} \text{  and } \quad\mathbb E\|G_t(\MLE)\|^2
\le\frac{\tau_4^2}{B};
\end{align}
Fix $u:=\psi_{t-1}-\MLE$, which is $\mathcal F_{t-1}$-measurable.
With minibatch sampling \emph{with replacement}, we can write
\[
\nabla G_t(\MLE)=\frac{1}{B}\sum_{b=1}^B \nabla^2\ell(x_{I_b},y_{I_b},\MLE),
\qquad I_1,\dots,I_B \ \text{i.i.d. and independent of }\mathcal F_{t-1}.
\]
Let $H_{I}:=\nabla^2\ell(x_I,y_I,\MLE)$ and $H:=\E[H_I]=\nabla^2\loss(\MLE)$.
Define the i.i.d.\ random vectors
\[
Z_b := (H_{I_b}-H)u, \qquad b=1,\dots,B.
\]
Then $\E_{t-1}[Z_b]=0$ and
\[
(\nabla G_t(\MLE)-H)u=\frac{1}{B}\sum_{b=1}^B Z_b.
\]
Moreover, by $L$-smoothness at $\MLE$, then $\|H\|_2\le \E\|H_I\|_2\le L$,
we have $\|H_I-H\|_2\le \|H_I\|_2+\|H\|_2\le 2L$, hence
\[
\|Z_b\|\le 2L\|u\|,\qquad
\E_{t-1}\|Z_b\|^2\le 4L^2\|u\|^2,\qquad
\E_{t-1}\|Z_b\|^4\le 16L^4\|u\|^4.
\]

Using independence and $\E_{t-1}[Z_b]=0$, the cross terms vanish:
\[
\E_{t-1}\Big\|\frac{1}{B}\sum_{b=1}^B Z_b\Big\|^2
=\frac{1}{B^2}\sum_{b=1}^B \E_{t-1}\|Z_b\|^2
\le \frac{4L^2}{B}\|u\|^2.
\]
Let $S:=\sum_{b=1}^B Z_b$. Since $\E_{t-1}[Z_b]=0$ and the $Z_b$'s are independent,
\[
\|S\|^2=\sum_{b=1}^B\|Z_b\|^2+2\sum_{1\le i<j\le B}\langle Z_i,Z_j\rangle.
\]
By $\langle Z_i,Z_j\rangle^2\le \|Z_i\|^2\|Z_j\|^2$, we have
\[
\E_{t-1}\|S\|^4
\le
\E_{t-1}\Big(\sum_{b=1}^B\|Z_b\|^2\Big)^2
+4\,\E_{t-1}\Big(\sum_{i<j}\langle Z_i,Z_j\rangle\Big)^2
\le
B\,\E_{t-1}\|Z_1\|^4+3B(B-1)\big(\E_{t-1}\|Z_1\|^2\big)^2.
\]
Using $(\E_{t-1}\|Z_1\|^2)^2\le \E_{t-1}\|Z_1\|^4$, we get
$\E_{t-1}\|S\|^4\le 3B^2\,\E_{t-1}\|Z_1\|^4$, hence
\begin{align}
\E_{t-1}\Big\|\frac{1}{B}\sum_{b=1}^B Z_b\Big\|^4
=\frac{1}{B^4}\E_{t-1}\|S\|^4
\le \frac{3}{B^2}\,\E_{t-1}\|Z_1\|^4
\le \frac{3}{B^2}\cdot 16L^4\|u\|^4
= \frac{48L^4}{B^2}\|u\|^4.
\end{align}
Then we have
\begin{align}
\E_{t-1}\big\|(\nabla G_t(\MLE)-\hess)(\psi_{t-1}-\MLE)\big\|^2
&\le \frac{4L^2}{B}\,\|\psi_{t-1}-\MLE\|^2, \label{eq:HG-second}\\
\E_{t-1}\big\|(\nabla G_t(\MLE)-\hess)(\psi_{t-1}-\MLE)\big\|^4
&\le \frac{48L^4}{B^2}\,\|\psi_{t-1}-\MLE\|^4. \label{eq:HG-fourth}
\end{align}
Since $\eta_{t-1}=G_t(\MLE)+(\nabla G_t(\MLE)-\hess)(\psi_{t-1}-\MLE)$,
the inequalities $\|a+b\|^2\le 2\|a\|^2+2\|b\|^2$ and $\|a+b\|^4\le 8\|a\|^4+8\|b\|^4$
combined with \cref{eq: second-fourth moment of Gt}, \cref{eq:HG-second}, \cref{eq:HG-fourth} yield
\begin{align}
\E_{t-1}\|\eta_{t-1}\|^2
&\le \frac{2\tau_4^2}{B}+\frac{8L^2}{B}\,\|\psi_{t-1}-\MLE\|^2,
\label{eq:eta-second}\\
\E_{t-1}\|\eta_{t-1}\|^4
&\le \frac{24\tau_4^4}{B^2}+\frac{384L^4}{B^2}\,\|\psi_{t-1}-\MLE\|^4.
\label{eq:eta-fourth}
\end{align}
By \cref{eq:schemeB-et}, $\E_{t-1}[\eta_{t-1}]=0$, $\E[\xi_{t-1}]=0$,
\cref{eq:eta-second}, and $\|I-\lambda\hess\|_2\le 1-\lambda\hat\mu$, we can obtain
\begin{align}
\mathbb E_{t-1}\|\psi_t - \MLE\|^2
&= \|(I-\lambda\hess)(\psi_{t-1}-\MLE)\|^2 + \lambda^2\E_{t-1}\|\eta_{t-1}\|^2 + 2\beta^{-1}\lambda D \\
&\le (1-\lambda\hat\mu)^2\|\psi_{t-1}-\MLE\|^2 + \lambda^2\E_{t-1}\|\eta_{t-1}\|^2 + 2\beta^{-1}\lambda D \\
&\le \Bigl((1-\lambda\hat\mu)^2+8\lambda^2L^2/B\Bigr)\|\psi_{t-1}-\MLE\|^2 + 2\lambda^2\tau_4^2/B + 2\beta^{-1}\lambda D.
\end{align}
Letting $t\to\infty$ yields
\begin{equation}\label{eq:Vinf-schemeB}
\mathbb E\|\psi_\infty-\MLE\|^2
\le
\frac{\frac{2\lambda^2\tau_4^2}{B}+2\beta^{-1}\lambda D}
{\lambda\hat\mu(2-\lambda\hat\mu)-\frac{8\lambda^2L^2}{B}}.
\end{equation}
Since $\lambda \le 1/4\hat{\mu}$, we have $2-\lambda\hat\mu\ge 7/4$ and
$\frac{8\lambda^2L^2}{B}\le \frac{1}{25}\lambda\hat\mu$, hence
\begin{equation}\label{eq:Vinf-simplified}
\mathbb E\|\psi_\infty-\MLE\|^2
\le
\frac{6}{5}\Bigl(\frac{\lambda\tau_4^2}{\hat\mu B}+\frac{D}{\hat\mu\beta}\Bigr).
\end{equation}
Then we expand the fourth moment of $\psi_t$ and take conditional expectation,
\begin{align}
\E_{t-1}\|\psi_t-\MLE\|^4
&=
\E_{t-1}\|(I-\lambda\hess)(\psi_{t-1}-\MLE)\|^4
+\E_{t-1}\|-\lambda\eta_{t-1}+\sqrt{2\beta^{-1}\lambda}\xi_{t-1}\|^4 \nonumber\\
&\quad
+ 4\,\E_{t-1}\| -\lambda\eta_{t-1}+\sqrt{2\beta^{-1}\lambda}\xi_{t-1}\|^2 \Big\langle (I-\lambda\hess)(\psi_{t-1}-\MLE),\,
-\lambda\eta_{t-1}+\sqrt{2\beta^{-1}\lambda}\xi_{t-1}\Big\rangle\\
&\quad
+2\,\E_{t-1}\Big(\|(I-\lambda\hess)(\psi_{t-1}-\MLE)\|^2
\|-\lambda\eta_{t-1}+\sqrt{2\beta^{-1}\lambda}\xi_{t-1}\|^2\Big)\nonumber\\
&\quad
+4\,\E_{t-1}\Big\langle (I-\lambda\hess)(\psi_{t-1}-\MLE),\,
-\lambda\eta_{t-1}+\sqrt{2\beta^{-1}\lambda}\xi_{t-1}\Big\rangle^2.\label{eq:fourth-expand}
\end{align}
Using $\langle a,b\rangle^2\le \|a\|^2\|b\|^2$ and Cauchy-Schwarz inequality,
\begin{align}
\E_{t-1}\|\psi_t-\MLE\|^4
&\le
\E_{t-1}\|(I-\lambda\hess)(\psi_{t-1}-\MLE)\|^4
+3\E_{t-1}\|-\lambda\eta_{t-1}+\sqrt{2\beta^{-1}\lambda}\xi_{t-1}\|^4 \nonumber\\
&\quad
+8\,\E_{t-1}\Big(\|(I-\lambda\hess)(\psi_{t-1}-\MLE)\|^2
\|-\lambda\eta_{t-1}+\sqrt{2\beta^{-1}\lambda}\xi_{t-1}\|^2\Big)\nonumber.
\end{align}
Moreover,
\begin{align}
\E_{t-1}\|-\lambda\eta_{t-1}+\sqrt{2\beta^{-1}\lambda}\xi_{t-1}\|^2
&=
\lambda^2\E_{t-1}\|\eta_{t-1}\|^2+2\beta^{-1}\lambda D \nonumber\\
&\le
\lambda^2\Big(\frac{2\tau_4^2}{B}+\frac{8L^2}{B}\|\psi_{t-1}-\MLE\|^2\Big)+2\beta^{-1}\lambda D,
\label{eq:w2}
\end{align}
and using $\|u+v\|^4\le 8\|u\|^4+8\|v\|^4$ together with $\E\|\xi\|^4=D(D+2)$,
\begin{align}
\E_{t-1}\|-\lambda\eta_{t-1}+\sqrt{2\beta^{-1}\lambda}\xi_{t-1}\|^4
&\le
8\lambda^4\E_{t-1}\|\eta_{t-1}\|^4
+32\beta^{-2}\lambda^2 D(D+2)\nonumber\\
&\le
\frac{192\lambda^4\tau_4^4}{B^2}
+\frac{3072\lambda^4L^4}{B^2}\|\psi_{t-1}-\MLE\|^4
+32\beta^{-2}\lambda^2 D(D+2).
\label{eq:w4}
\end{align}
Combining \cref{eq:fourth-expand}, \cref{eq:w4} and
$\|(I-\lambda\hess)\|_2\le 1-\lambda\hat\mu$, we obtain
\begin{align}
\E_{t-1}\|\psi_t-\MLE\|^4
&\le
\Bigl((1-\lambda\hat\mu)^4+\frac{64(1-\lambda\hat\mu)^2\lambda^2L^2}{B}
+\frac{9216\lambda^4L^4}{B^2}\Bigr)\|\psi_{t-1}-\MLE\|^4 \nonumber\\
&\quad
+16(1-\lambda\hat\mu)^2\Bigl(\lambda^2\tau_4^2/B+\lambda D/\beta\Bigr)\|\psi_{t-1}-\MLE\|^2
+576\lambda^4\tau_4^4/B^2
+96\beta^{-2}\lambda^2 D(D+2).
\label{eq:fourth-recursion}
\end{align}
Taking full expectation and letting $t\to\infty$ gives
\begin{align}
\Bigl(1-(1-\lambda\hat\mu)^4-\frac{64(1-\lambda\hat\mu)^2\lambda^2L^2}{B}
-\frac{9216\lambda^4L^4}{B^2}\Bigr)\E\|\psi_\infty-\MLE\|^4
&\le
16(1-\lambda\hat\mu)^2\Bigl(\frac{\lambda^2\tau_4^2}{B}+\frac{\lambda D}{\beta}\Bigr)\E\|\psi_\infty-\MLE\|^2 \nonumber\\
&\quad+\frac{576\lambda^4\tau_4^4}{B^2}
+96\beta^{-2}\lambda^2 D(D+2).
\label{eq:M-infty-raw}
\end{align}
Under $\lambda\hat\mu\le 1/4$ and $\lambda\le B\hat\mu/(200L^2)$,
\begin{align}
1-(1-\lambda\hat\mu)^4 
-\frac{64(1-\lambda\hat\mu)^2\lambda^2 L^2}{B}
-\frac{9216\lambda^4 L^4}{B^2}
&\ge (4\lambda\hat\mu - 6\lambda^2\hat\mu^2) 
- 64\lambda^2 L^2/B - 9216\lambda^4 L^4/B^2 \nonumber\\
&\ge (5/2)\lambda\hat\mu - (8/25)\lambda\hat\mu - (36/625)\lambda\hat\mu \nonumber\\
&\ge 2\lambda\hat\mu.\label{eq:den-lower}
\end{align}
Using $(1-\lambda\hat\mu)^2\le 1$, \cref{eq:Vinf-simplified}, and \cref{eq:den-lower} in \cref{eq:M-infty-raw} yields
\begin{align}
\E\|\psi_\infty-\MLE\|^4
&\le
\frac{48}{5\hat{\mu}^{2}}
\Bigl(\frac{\lambda\tau_4^2}{B}+\frac{D}{\beta}\Bigr)^2
+72\frac{\lambda^2\tau_4^4}{\hat{\mu}^2B^2}
+48\frac{\lambda D(D+2)}{\hat{\mu}\beta^{2}}\\
&\le\Bigl( 72+ \frac{48}{5}\Bigr) \frac{\lambda^2\tau_4^4}{\hat{\mu}^2B^2} \, + \frac{96}{5} \frac{\lambda \tau_4^2}{\hat{\mu}^2 B \beta}D + \frac{48}{5} \frac{D^2}{\hat{\mu}^2\beta^2}
+48\frac{\lambda D(D+2)}{\hat{\mu}\beta^{2}}\\
&\le 96\frac{\lambda^2\tau_4^4}{\hat{\mu}^2B^2} \, + 24 \frac{\lambda \tau_4^2}{\hat{\mu}^2 B \beta}D + 12 \frac{D^2}{\hat{\mu}^2\beta^2}
+48\frac{\lambda D(D+2)}{\hat{\mu}\beta^{2}}
.\label{eq:almost-final}
\end{align}
\end{proof}

\begin{proof}[Proof of \cref{cor:stationary-wasserstein-error-bound}]
Taking the limit $t\to \infty$ and combining with \cref{lem:fourth-moment-bound} yields
\begin{align}
\begin{split}
W_2^2(\pi_\theta,\pi_\psi)
&\le
\frac{\lambda}{1-\bar\beta}\cbra*{\frac{\lambda \overline{M^2}}{2} + \frac{\overline{M}^2}{4\mu}}\,
\E\,\bigl\|\psi_\infty-\MLE\bigr\|^4\\
&\le
 \underbrace{\frac{2}{\mu}\cbra*{\frac{ \overline{M^2}}{8L} + \frac{\overline{M}^2}{4\mu}}}_{C_0} \times \cbra*{\frac{96\lambda^2\tau_4^4}{\hat{\mu}^2B^2} \, 
 +  \frac{24\lambda \tau_4^2}{\hat{\mu}^2 B \beta}D
 +  \frac{12 D^2}{\hat{\mu}^2\beta^2}
 +\frac{48\lambda D(D+2)}{\hat{\mu}\beta^{2}}  } \\
 &\le
 C_0 \times \cbra*{ \frac{96\lambda^2\tau_4^4}{\hat{\mu}^2B^2} \, 
 +  \frac{24\lambda \tau_4^2}{\hat{\mu}^2 B \beta}D 
 +  \frac{12 D^2}{\hat{\mu}^2\beta^2}
 +\frac{48\lambda D(D+2)}{\hat{\mu}\beta^{2}} } \\
&\le
\underbrace{\frac{96C_0\tau_4^4}{\hat\mu^2}}_{A_0}\,\frac{\lambda^2}{B^2}
+
\underbrace{\frac{12C_0\tau_4^2}{\hat\mu^2}\,D}_{A_1}\,\frac{2\lambda}{B\beta}
+
\underbrace{\bigl(\frac{12C_0}{\hat\mu^2}\,D^2+\frac{48C_0}{\hat\mu^2}\,D(D+2)\bigr)}_{A_2}\,\frac{1}{\beta^2}\\
&\leq A^2\left(\frac{\lambda}{B} + \frac{1}{\beta}\right)^2
\label{eq:wasserstein-error-bound-constant-sgld}
\end{split}
\end{align}
where $A^2 = \max\{A_0, A_1, A_2\}$
\end{proof}

\subsection{Proof of \cref{thm:covariance-error-bounds}}

\begin{proof}
Using \cref{cor:stationary-wasserstein-error-bound}, the proof of \cref{thm:covariance-error-bounds} is almost immediate.
Under Assumptions \ref{assump:statistic-convex}--\ref{assump:loss}, there exists a constant $c > 0$ such that 
$\Sigma_{\theta} \prec c\lambda I$ and $\Sigma_{\psi} \prec c\lambda I$ \citep[Theorem 4]{dieuleveut2020bridging}
and therefore $\sigma_{\theta,d}$ and $\sigma_{\psi,d}$ are of order $\lambda^{1/2}$.  
Hence, it follows from \cref{cor:stationary-wasserstein-error-bound,eq:general-stdev-and-cov-error-bounds} that the relative errors
of the stationary standard deviations and covariance satisfy \cref{eq:cov-relative-error,eq:stdev-relative-error}.
\end{proof}

\section{Proofs for Momentum Results} \label{sec:momentum-proofs}

\subsection{Proof of \cref{Proposition: stationary covariance analysis for SGLD with momentum}}

\begin{proof}
	Let 
        \begin{align}
        \eta_{t-1}
        = \Bigl[G_t(\MLE)
            + \nabla G_t(\MLE)(\psi_{t-1}-\MLE)\,\Bigr]
          - \E\Bigl[G_t(\MLE)
            + \nabla G_t(\MLE)(\psi_{t-1}-\MLE)\,\Bigr].
        \end{align}
	Then, we have
        \begin{align}
        \begin{aligned}
	\label{eq: sgd with momentum update}
        \psi_t - \MLE
        = \psi_{t-1}-\MLE
          - \Lambda\,\kappa\,m_{t-1}
          - \Lambda\E\Bigl[G_t(\MLE)
            + \nabla G_t(\MLE)(\psi_{t-1}-\MLE)\,\Bigr]
          - \Lambda\,\eta_{t-1}
          + \sqrt{2\beta^{-1}\Lambda}\,\xi_{t-1}.
        \end{aligned}
        \end{align}
	Since by assumption $ \reg(\psi) = \frac{1}{2} \psi^{\top} \Gamma \psi$, we have
	\begin{align}
		\E(G_t(\hat{\theta})) = \frac{1}{N}\sum_{n=1}^{N}\nabla \ell\left(x_{n}, y_{n}, \MLE\right) + \frac{1}{N}\Gamma \MLE = 0.
	\end{align}
	Then \cref{eq: sgd with momentum update} can be rewritten as 
	\begin{align}
		\begin{aligned}
			\psi_{t}-\MLE 
			& = \psi_{t-1}-\MLE - \Lambda \kappa m_{t-1} - \frac{\Lambda}{N} \sum_{n=1}^{N} \left\lbrace \mathcal{J}_{n} \left(\psi_{t-1}-\MLE\right)\right\rbrace - \frac{1}{N}\Lambda \Gamma (\psi_{t-1}-\MLE)- \Lambda \eta_{t-1}+ \sqrt{2\beta^{-1} \Lambda}\,\xi_{t-1}\\
			& = \left( I- \Lambda\mathcal{J}-\frac{\Lambda}{N} \Gamma \right)\left(\psi_{t-1}-\MLE\right) - \Lambda \kappa m_{t-1} - \Lambda \eta_{t-1}+ \sqrt{2\beta^{-1} \Lambda}\,\xi_{t-1}\\
			& = \left( I- \Lambda \hess \right)\left(\psi_{t-1}-\MLE\right) - \Lambda \kappa m_{t-1} - \Lambda \eta_{t-1}+ \sqrt{2\beta^{-1} \Lambda}\,\xi_{t-1}
		\end{aligned}
	\end{align}
    Assuming $\theta_{t}$ and $m_{t}$ are jointly sampled from stationary distribution, we have
	\begin{align}
		\begin{aligned}
			\Sigma_{\psi} &= \mathbb{E} \left[ \left( \psi_{t}-\MLE\right)  \left( \psi_{t}-\MLE\right)^{\top} \right]\\
			& =   \left( I - \Lambda \hess \right)  \Sigma_{\psi} \left( I - \Lambda \hess \right)^{\top} +\kappa^{2}\Lambda M \Lambda  + \Lambda \overline{C}_{\psi} \Lambda  - (D+ D^{\top}) + \frac{2\Lambda}{\beta} , 
		\end{aligned}
	\end{align}
	where $D = \kappa \left( I - \Lambda \hess \right) \mathbb{E} \left[ \left( \psi_{t-1} -\MLE \right)m_{t-1}^{\top}\right] \Lambda$, and $M = \mathbb{E} \left[ m_{t-1} m_{t-1}^{\top}\right]$. 
	
	The rest of proof mainly follows the proof of \citet[Theorem 3]{liu2021noise}. 
	According to \cref{eq: original SGD with momentum}, we have $\Lambda m_{t} = \psi_{t-1} - \psi_{t}$, so
	\begin{align}
		\begin{aligned}
			\Lambda M \Lambda 
			& = \mathbb{E} \left[ \left( \psi_{t-1}-\MLE-\psi_{t-2} + \MLE -\sqrt{2\beta^{-1} \Lambda}\,\xi_{t-1}\right) \left( \psi_{t-1}-\MLE-\psi_{t-2}+\MLE\right)^{\top} -\sqrt{2\beta^{-1} \Lambda}\,\xi_{t-1}\right] \\
			& = 2 \Sigma_{\psi} - \mathbb{E} \left[ \left(\psi_{t-1}-\MLE\right) \left(\psi_{t-2}^{\top}-\MLE\right)\right]- \mathbb{E} \left[ \left(\psi_{t-2}-\MLE\right) \left(\psi_{t-1}^{\top}-\MLE\right)\right] + \frac{2\Lambda}{\beta}.
		\end{aligned}
	\end{align}
	and 
	\begin{align}
		\begin{aligned}
			D 
			&= \kappa \left( I - \Lambda \hess \right) \mathbb{E} \left[ \left( \psi_{t-1} -\MLE \right)m_{t-1}^{\top}\right] \Lambda \\
			& = \kappa \left( I - \Lambda \hess \right) \mathbb{E} \left[ \left( \psi_{t-1} -\MLE \right)\left( \psi_{t-2}-\psi_{t-1}+\sqrt{2\beta^{-1} \Lambda}\,\xi_{t-1}\right)^{\top}\right]\\
			& = \kappa \left( I - \Lambda \hess \right) \left( \mathbb{E} \left[ \left(\psi_{t-1}-\MLE\right) \left(\psi_{t-2}-\MLE\right)^{\top}\right]  -\Sigma_{\psi}\right).
		\end{aligned}
	\end{align}
	Note that 
	\begin{align}
		\label{eq: equation 1}
		\begin{aligned}
			\mathbb{E} \left[ \left(\psi_{t-1}-\MLE\right) \left(\psi_{t-2}-\MLE\right)^{\top}\right] 
			& = \mathbb{E} \left[ \left(\psi_{t}-\MLE\right) \left(\psi_{t-1}-\MLE\right)^{\top} \right]\\
			& = \mathbb{E} \left[ \left(\left( I- \Lambda \hess \right)\left(\psi_{t-1}-\MLE\right) - \Lambda \kappa m_{t-1} - \Lambda \eta_{t-1}+\sqrt{2\beta^{-1} \Lambda}\,\xi_{t-1}\right) \left(\psi_{t-1}-\MLE\right)^{\top} \right]\\
			& = \left( I- \Lambda \hess \right) \Sigma_{\psi}-\Lambda \kappa \mathbb{E} \left[ m_{t-1} \left(\psi_{t-1}-\MLE\right)^{\top} \right]\\
			& = \left( I- \Lambda \hess \right) \Sigma_{\psi}-\kappa \mathbb{E} \left[ \left( \psi_{t-2}-\psi_{t-1}\right) \left(\psi_{t-1}-\MLE\right)^{\top} \right]\\
			& =  \left( I- \Lambda \hess \right) \Sigma_{\psi} + \kappa \Sigma_{\psi} - \kappa \mathbb{E}\left[ \left( \psi_{t-2}-\MLE\right) \left(\psi_{t-1}-\MLE\right)^{\top}  \right].
		\end{aligned}
	\end{align}
    Solving, we obtain 
    \begin{align}
        \mathbb{E} \left[ \left(\psi_{t-1}-\MLE\right) \left(\psi_{t-2}-\MLE\right)^{\top}\right]  &= \frac{1}{1 + \kappa} \left[\left( I- \Lambda \hess \right) \Sigma_{\psi} + \kappa \Sigma_{\psi} \right].
    \end{align}

	so we conclude that
    \begin{align}
		(1-\kappa)(\Lambda \hess \Sigma+\Sigma \hess \Lambda)+\frac{\kappa}{1-\kappa^2}(\Lambda \hess \Lambda \hess \Sigma+\Sigma \hess \Lambda \hess \Lambda)=\Lambda  \overline{C}_{\psi} \Lambda + \frac{1+\kappa^2}{1-\kappa^2} \Lambda \hess \Sigma \hess \Lambda + (1+\kappa^{2}) \frac{2\Lambda}{\beta}.
	\end{align}
	
\end{proof}

\subsection{Proof of \cref{thm:wasserstein-error-bound-momentum}}\label{app: momentum}
Since there exists a coupling of $\theta_0\dist\nu$ and $\psi_0\dist\nu'$ with
$W_2^2(\nu,\nu')=\E\|\theta_0-\psi_0\|^2$, we take $(\theta_0,\psi_0)$ from this joint
distribution, and initialize the momenta at $m_0=\nu_0=0$. The proof proceeds in three steps: first, a coupled $2\times2$ linear
recursion for $(\E\|\theta_t-\psi_t\|^2,\E\|m_t-\nu_t\|^2)$ with coefficient matrix $A$;
second, a weighted Lyapunov function $V_t=\E\|\theta_t-\psi_t\|^2+\gamma\E\|m_t-\nu_t\|^2$
whose optimal weight $\gamma^\star$ yields the sharpest contraction rate $\bar\beta=\rho(A)$;
third, iteration of the resulting one-step recursion.

By $L$-smoothness of $\nabla\ell_n$ and $\|a+b\|^2\le(1+c)\|a\|^2+(1+1/c)\|b\|^2$ with
$c=(1-\kappa)/\kappa$,
\begin{align}
\E_{t-1}\|m_t-\nu_t\|^2
&=\E_{t-1}\big\|\kappa(m_{t-1}-\nu_{t-1})+G_t(\theta_{t-1})-G_t(\MLE)-\grad G_t(\MLE)(\psi_{t-1}-\MLE)\big\|^2\nonumber\\
&\le\kappa\|m_{t-1}-\nu_{t-1}\|^2
+\frac{1}{1-\kappa}\E_{t-1}\big\|G_t(\theta_{t-1})-G_t(\MLE)-\grad G_t(\MLE)(\psi_{t-1}-\MLE)\big\|^2\nonumber\\
&\le\kappa\|m_{t-1}-\nu_{t-1}\|^2
+\frac{1}{1-\kappa}\Big(2L^2\|\theta_{t-1}-\psi_{t-1}\|^2+\frac{\overline{M^2}}{2}\|\psi_{t-1}-\MLE\|^4\Big).
\end{align}
Similarly, expanding $\E_{t-1}\|\theta_t-\psi_t\|^2$ and bounding the inner-product term via
Taylor's theorem and $2ab\le\mu a^2+b^2/\mu$,
\begin{align}
\E_{t-1}\|\theta_t-\psi_t\|^2
&\le\Big(1-\lambda\mu+\lambda\kappa+\frac{2\lambda^2L^2}{1-\kappa}\Big)\|\theta_{t-1}-\psi_{t-1}\|^2
+\lambda\kappa(1+\lambda)\|m_{t-1}-\nu_{t-1}\|^2\nonumber\\
&\quad+\Big(\frac{\lambda^2\overline{M^2}}{2(1-\kappa)}+\frac{\lambda\overline M^2}{4\mu}\Big)\|\psi_{t-1}-\MLE\|^4.
\end{align}
Taking full expectation, this yields the linear recursion
\begin{equation}\label{eq:matrix-A-thm}
\begin{pmatrix}\E\|\theta_t-\psi_t\|^2\\ \E\|m_t-\nu_t\|^2\end{pmatrix}
\le
\underbrace{
\begin{pmatrix}
1-\lambda\mu+\lambda\kappa+2\lambda^2L^2/(1-\kappa)& \lambda\kappa(1+\lambda)\\
2L^2/(1-\kappa) & \kappa
\end{pmatrix}}_{\displaystyle A\defas(a_{ij})_{2\times2}}
\begin{pmatrix}\E\|\theta_{t-1}-\psi_{t-1}\|^2\\ \E\|m_{t-1}-\nu_{t-1}\|^2\end{pmatrix}
+\begin{pmatrix}\mathcal A\\ \mathcal B\end{pmatrix}C_{t-1},
\end{equation}
where $\mathcal A=\dfrac{\lambda^2}{2(1-\kappa)}\overline{M^2}+\dfrac{\lambda}{4\mu}\overline M^2$
and $\mathcal B=\dfrac{\overline{M^2}}{2(1-\kappa)}$.

For $\gamma>0$, let $V_t:=\E\|\theta_t-\psi_t\|^2+\gamma\E\|m_t-\nu_t\|^2$. From
\cref{eq:matrix-A-thm},
\begin{align}
V_t
&=(a_{11}+\gamma a_{21})\E\|\theta_{t-1}-\psi_{t-1}\|^2
+(a_{12}+\gamma a_{22})\E\|m_{t-1}-\nu_{t-1}\|^2+(\mathcal A+\gamma\mathcal B)C_{t-1}\nonumber\\
&\le\beta(\gamma)\,V_{t-1}+(\mathcal A+\gamma\mathcal B)\,C_{t-1},
\end{align}
where $\beta(\gamma):=\max\{a_{11}+\gamma a_{21},\ a_{22}+a_{12}/\gamma\}.
$

Minimizing $\beta(\gamma)$ over $\gamma>0$ at
$\gamma^\star=\big[(a_{22}-a_{11})+\sqrt{(a_{11}-a_{22})^2+4a_{12}a_{21}}\big]/(2a_{21})>0$
gives 
\begin{equation}\label{eq:beta-bar-def}
\bar\beta:=\frac{1-\lambda\mu+\tfrac{2\lambda^2L^2}{1-\kappa}+(1+\lambda)\kappa+\sqrt{X^2+Y}}{2},
\end{equation}
where $X:=1-\lambda\mu+\tfrac{2\lambda^2L^2}{1-\kappa}-(1-\lambda)\kappa$ and
$Y:=\frac{8L^2\lambda\kappa(1+\lambda)}{1-\kappa}$, and the source coefficient
\begin{equation}\label{eq:definition-of-P}
\mathcal P:=\mathcal A+\gamma^\star\mathcal B
=\frac{\lambda^2}{2(1-\kappa)}\overline{M^2}+\frac{\lambda}{4\mu}\overline M^2
+\frac{\overline{M^2}}{8L^2}\big(\sqrt{X^2+Y}-X\big).
\end{equation}
Since $\E\|\theta_t-\psi_t\|^2\le V_t$ and $V_t\le\bar\beta V_{t-1}+\mathcal P\,C_{t-1}$,
iterating and using $V_0=\E\|\theta_0-\psi_0\|^2$ (as $m_0=\nu_0=0$) gives
\begin{equation}\label{eq:thm-iterate}
\E\|\theta_t-\psi_t\|^2
\le\bar\beta^{\,t}\,\E\|\theta_0-\psi_0\|^2+\mathcal P\sum_{s=1}^{t}\bar\beta^{\,t-s}C_{s-1}.
\end{equation}

\subsection{Proof of \cref{prop:iterate-average-momentum}}
According to \cref{sec: Proof of prop: mixing time}, assume that the SG-MCMC iterates have reached stationarity. 

Define the lag-$k$ autocovariance
\begin{align}
\Xi_k \defas 
\mathbb{E}_{\pi_\psi}\!\left[
(\psi_{t+k}-\MLE)(\psi_t-\MLE)^\top
\right].
\end{align}
Under stationarity and the linearized dynamics, this autocovariance satisfies
\begin{align}
\Xi_k
=( I - \Lambda \hess )^k \Xi_0
=( I - \Lambda \hess )^k \Sigma_\psi ,
\end{align}
where $\Sigma_\psi \defas \Xi_0$ denotes the stationary covariance.

Next, we approximate the stationary covariance of the averaged iterate
\begin{align}
\bar{\psi}_{k} \defas \frac{1}{k} \sum_{k^{\prime}=1}^{k} \psi_{k^{\prime}} .
\end{align}
Assuming stationarity, its covariance can be computed as

\begin{align}
\Sigma_{\psi}^{(k)}
&\defas 
\mathbb{E}
\left[
(\bar{\psi}_{k}-\MLE)(\bar{\psi}_{k}-\MLE)^\top
\right]  \\
&= \frac{1}{k^{2}}
\mathbb{E}
\left[
\left( \sum_{k^{\prime}=1}^{k} (\psi_{k^{\prime}}-\MLE) \right)
\left( \sum_{k^{\prime\prime}=1}^{k} (\psi_{k^{\prime\prime}}-\MLE) \right)^{\top}
\right] \\
&= \frac{1}{k^{2}}
\left(
k \Sigma_{\psi}
+
2 \sum_{k^{\prime}=1}^{k-1} \Xi_{k^{\prime}}
\right) \\
&= \frac{1}{k^{2}}
\left(
k \Sigma_{\psi}
+
2 \sum_{k^{\prime}=1}^{k-1}
( I - \Lambda \hess )^{k^{\prime}} \Sigma_{\psi}
\right).
\end{align}

\subsection{Proof of \cref{cor:stationary-wasserstein-error-bound-momentum}}
\begin{lemma}\label{lem:fourth-moment-compact}
Under the conditions of \cref{cor:stationary-wasserstein-error-bound-momentum}, let
$\hat\mu,\hat L$ denote the smallest and largest eigenvalues of $\hess=\grad^2\loss(\MLE)$, and
let $\underline{\mathsf U},\underline{\mathsf V}$ be constants such that
$\E\|\psi_t-\MLE\|^2\le\underline{\mathsf U},
\E\|\nu_t\|^2\le\underline{\mathsf V}.$ For any step size
$\lambda\le\min\{1/\hat L,\,1/(4\hat\mu),\,B\hat\mu/(CL^2),\,(\hat\mu/C)^{1/2}\}$
 and momentum $\kappa$ as in
\cref{cor:stationary-wasserstein-error-bound-momentum}, 
there exists $C>0$ such that, whenever $M=(m_{ij})$ below satisfies $\rho(M)<1$,
the stationary iterate $\psi_\infty\dist\pi_\psi$ of the linearized proxy with momentum
\cref{eq: proxy SGD with momentum} satisfies
\begin{align}
\E\|\psi_\infty-\MLE\|^4 \le
\frac{(1-m_{22})\,\mathcal A+m_{12}\,\mathcal B}{(1-m_{11})(1-m_{22})-m_{12}m_{21}},
\end{align}
where
\begin{align}
m_{11}&\defas(1+\lambda\hat\mu)(1-\lambda\hat\mu)^4
  +C\Big\{\tfrac{\lambda^2L^2}{B}
          +\hat\mu^{-1}\tfrac{\lambda^3\kappa^2L^2}{B}
          +\tfrac{\lambda^4L^4}{B^2}\Big\},
  \label{eq:m11-def}\\[2pt]
m_{12}&\defas C\Big\{\hat\mu^{-3}\lambda\kappa^4
          +\hat\mu^{-1}\tfrac{\lambda^3\kappa^2L^2}{B}\Big\},
  \label{eq:m12-def}\\[2pt]
m_{21}&\defas C\Big\{\hat L^4+\kappa^2\hat L^2
          +\tfrac{\hat L^2L^2}{B}
          +\tfrac{L^4}{B^2}\Big\},
  \label{eq:m21-def}\\[2pt]
m_{22}&\defas C\Big\{\kappa^4+\kappa^2\hat L^2\Big\},
  \label{eq:m22-def}\\[4pt]
\mathcal A&\defas
  C\Bigl(\tfrac{\lambda^4\tau_4^4}{B^2}
        +\tfrac{\lambda^2(D^2+2D)}{\beta^2}\Bigr)+C\Bigl(\tfrac{\lambda^2\tau_4^2}{B}+\tfrac{\lambda D}{\beta}\Bigr)
   \Bigl[(1+\lambda\hat\mu)(1-\lambda\hat\mu)^2\,\underline{\mathsf U}
        +\bigl(1+(\lambda\hat\mu)^{-1}\bigr)\lambda^2\kappa^2\,
         \underline{\mathsf V}\Bigr],
  \label{eq:A-def}\\[2pt]
    \mathcal B&\defas
  C\tfrac{\tau_4^4}{B^2}
  +C\tfrac{\tau_4^2}{B}
   \bigl(\hat L^2\,\underline{\mathsf U}+\kappa^2\,\underline{\mathsf V}\bigr)
  +C\tfrac{\kappa^2L^2}{B}\,\underline{\mathsf V}.
  \label{eq:B-def}
\end{align}
\end{lemma}

\begin{proof}   
The true linearized proxy with momentum at $\theta=\MLE$ can be written as
\begin{equation}\label{eq:schemeB-mom}
\left\{
\begin{array}{l}
\nu_t=\kappa\nu_{t-1}+\hess(\psi_{t-1}-\MLE)+\eta_{t-1},\\
\psi_t-\MLE=(I-\lambda\hess)(\psi_{t-1}-\MLE)-\lambda\kappa\nu_{t-1}-\lambda\eta_{t-1}
+\sqrt{2\beta^{-1}\lambda}\,\xi_{t-1},
\end{array}\right.
\end{equation}
where $\eta_{t-1}
:=
G_t(\MLE)+\nabla G_t(\MLE)(\psi_{t-1}-\MLE)-\hess(\psi_{t-1}-\MLE)$ and $\xi_{t-1}\sim\mathcal N(0,I_D)$.

Following the 
proof of \cref{thm:wasserstein-error-bound} and $\E_{t-1}[\eta_{t-1}]=0$, we have
\begin{align}
\E_{t-1}\|\eta_{t-1}\|^2 &\le \frac{2\tau_4^2}{B}+\frac{8L^2}{B}\|\psi_{t-1}-\MLE\|^2, \label{eq:eta2-mom}\\
\E_{t-1}\|\eta_{t-1}\|^4 &\le \frac{24\tau_4^4}{B^2}+\frac{384L^4}{B^2}\|\psi_{t-1}-\MLE\|^4. \label{eq:eta4-mom}
\end{align}
Let $e_t:=\psi_t-\MLE$, \cref{eq:schemeB-mom} can be written as 
\begin{align}
e_t=(I-\lambda\hess)e_{t-1}-\lambda\kappa\nu_{t-1}-\lambda\eta_{t-1}
+\sqrt{2\beta^{-1}\lambda}\,\xi_{t-1}.
\end{align}
Then
\begin{align}
\E_{t-1}\|e_t\|^2
&=
\|(I-\lambda \hess)e_{t-1}-\lambda\kappa\nu_{t-1}\|^2
+\lambda^2\E_{t-1}\|\eta_{t-1}\|^2
+2\beta^{-1}\lambda D .
\end{align}

Under the stepsize condition
$\|I-\lambda\hess\|_2\le 1-\lambda\hat\mu$ and by $\|x-y\|^2
\le (1+a)\|x\|^2 + (1+a^{-1})\|y\|^2$ ($a>0$), we can obtain
\begin{align}
\E_{t-1}\|e_t\|^2
&\le
(1+\lambda\hat\mu)(1-\lambda\hat\mu)^2\|e_{t-1}\|^2
+
\bigl(1+(\lambda\hat\mu)^{-1}\bigr)\lambda^2\kappa^2\|\nu_{t-1}\|^2
\nonumber\\
&\qquad
+\lambda^2\E_{t-1}\|\eta_{t-1}\|^2
+2\beta^{-1}\lambda D .
\end{align}
By \cref{eq:eta2-mom}, 
\begin{align}
\E_{t-1}\|e_t\|^2
&\le
\left\{
(1+\lambda\hat\mu)(1-\lambda\hat\mu)^2
+\frac{8\lambda^2L^2}{B}
\right\}\|e_{t-1}\|^2
\nonumber\\
&\qquad
+
\bigl(1+(\lambda\hat\mu)^{-1}\bigr)\lambda^2\kappa^2\|\nu_{t-1}\|^2
+\frac{2\lambda^2\tau_4^2}{B}
+2\beta^{-1}\lambda D .
\label{eq:psi2-rec-mom}
\end{align}
Similarly, using $\|x+y\|^2 \le 2\|x\|^2 + 2\|y\|^2$ yields
\begin{align}
\E_{t-1}\|\nu_t\|^2 
&= \|\kappa\nu_{t-1} + \hess\,e_{t-1}\|^2 + \E_{t-1}\|\eta_{t-1}\|^2 \nonumber\\
&\le \Big(2\hat L^2 + 8L^2/B\Big)\|e_{t-1}\|^2 + 2\kappa^2\|\nu_{t-1}\|^2 + 2\tau_4^2/B.
\label{eq:nu2-rec-mom}
\end{align}

Taking expectations in \cref{eq:psi2-rec-mom,eq:nu2-rec-mom} gives the linear recursion
\begin{equation}\label{eq:matrix-A-mom}
\binom{\E\|e_t\|^2}{\E\|\nu_t\|^2}
\le
\underbrace{
\begin{pmatrix}
(1+\lambda\hat\mu)(1-\lambda\hat\mu)^2+\frac{8\lambda^2L^2}{B}
&
\bigl(1+(\lambda\hat\mu)^{-1}\bigr)\lambda^2\kappa^2
\\
2\hat L^2+\frac{8L^2}{B}
&
2\kappa^2
\end{pmatrix}}_{\displaystyle A\defas(a_{ij})_{2\times2}}
\binom{\E\|e_{t-1}\|^2}{\E\|\nu_{t-1}\|^2}
+
\binom{c_e}{c_\nu},
\end{equation}
where
\[
c_e \defas \frac{2\lambda^2\tau_4^2}{B}+2\beta^{-1}\lambda D,
\qquad
c_\nu \defas \frac{2\tau_4^2}{B}.
\]
Under the step-size and momentum restrictions, we bound the entries of $A$.

Observe that
\[
a_{11}=(1+\lambda\hat\mu)(1-\lambda\hat\mu)^2+\frac{8\lambda^2L^2}{B},
\]
and let $s:=\lambda\hat\mu\le \frac14$, we have
\[
(1+s)(1-s)^2=1-s-s^2+s^3\le 1-s.
\]
Then
\[
1-a_{11}
\ge \lambda\hat\mu-\frac{8\lambda^2L^2}{B}.
\]
Thus, by choosing the universal constant in the stepsize condition
$
\lambda\le \frac{B\hat\mu}{C L^2}
$
large enough, we have
$
1-a_{11}\ge \frac12\lambda\hat\mu .
$

Note that $a_{22}=2\kappa^2$ and $\kappa\le 1/2$,
\[
1-a_{22}=1-2\kappa^2\ge \frac12 .
\]

It remains to control the off-diagonal product: 
\[a_{12}=(1+(\lambda\hat\mu)^{-1})\lambda^2\kappa^2\le C\lambda^3 , \]
and 
\[a_{21}=2\hat L^2+8L^2/B\le C .\]
For $\lambda\le(\hat\mu/C)^{1/2}$
\[
a_{12}a_{21}\le C\lambda^3\le\tfrac18\lambda\hat\mu.
\]
Thus,
\[
\Delta_2
:=
(1-a_{11})(1-a_{22})-a_{12}a_{21}
\ge
\frac12\lambda\hat\mu\cdot \frac12
-\frac18\lambda\hat\mu
=
\frac18\lambda\hat\mu
>0 .
\]
Together with $a_{11},a_{22}<1$ and $A$ entrywise nonnegative, it implies $\rho(A)<1$.

Starting from $e_0=0$ and $\nu_0=0$, the recursion gives, for every $t\ge 1$,
\[
\binom{\E\|e_t\|^2}{\E\|\nu_t\|^2}
\le
\sum_{j=0}^{t-1} A^j
\binom{c_e}{c_\nu}.
\]
Since $\rho(A)<1$, the geometric series is bounded entrywise by
\[
\sum_{j=0}^{t-1}A^j
\le
\sum_{j=0}^{\infty}A^j
=
(I-A)^{-1}.
\]
Hence
\[
\binom{\E\|e_t\|^2}{\E\|\nu_t\|^2}
\le
(I-A)^{-1}
\binom{c_e}{c_\nu}.
\]
Writing
\[
I-A=
\begin{pmatrix}
1-a_{11} & -a_{12}\\
-a_{21} & 1-a_{22}
\end{pmatrix},
\]
Cramer's rule yields
\[
(I-A)^{-1}
\binom{c_e}{c_\nu}
=
\frac1{\Delta_2}
\binom{
(1-a_{22})c_e+a_{12}c_\nu
}{
a_{21}c_e+(1-a_{11})c_\nu
}.
\]
Therefore, with
\begin{equation}\label{eq:UV-def-mom}
\underline{\mathsf U}
=
\frac{(1-a_{22})\,c_e+a_{12}\,c_\nu}{\Delta_2},
\qquad
\underline{\mathsf V}
=
\frac{a_{21}\,c_e+(1-a_{11})\,c_\nu}{\Delta_2},
\end{equation}
we have, uniformly for all $t\ge 1$,
\[
\E\|\psi_t-\MLE\|^2\le \underline{\mathsf U},
\qquad
\E\|\nu_t\|^2\le \underline{\mathsf V}.
\]
Thus constants as in the statement exist.

If $u$ is deterministic and $v$ is a mean-zero random vector, then
\begin{equation}\label{eq:uplusv}
\E\|u+v\|^4
\le
\|u\|^4
+8\|u\|^2\,\E\|v\|^2
+3\,\E\|v\|^4 .
\end{equation}

From \cref{eq:schemeB-mom}, write
\[
e_t=u_{t-1}+w_{t-1},
\]
where
\[
u_{t-1}:=(I-\lambda\hess)e_{t-1}-\lambda\kappa\nu_{t-1},
\qquad
w_{t-1}:=-\lambda\eta_{t-1}
+\sqrt{2\beta^{-1}\lambda}\,\xi_{t-1}.
\]
Conditioning on the $\mathcal{F}_{t-1}$, $u_{t-1}$ is deterministic. Moreover, since $\E_{t-1}[\eta_{t-1}]=0$ and $\E[\xi_{t-1}]=0$, we have
\[
\E_{t-1}[w_{t-1}]=0 .
\]
Applying \cref{eq:uplusv} therefore gives
\begin{equation}\label{eq:psi4-start}
\E_{t-1}\|e_t\|^4
\le
\|u_{t-1}\|^4
+8\|u_{t-1}\|^2\,\E_{t-1}\|w_{t-1}\|^2
+3\,\E_{t-1}\|w_{t-1}\|^4 .
\end{equation}

For the leading term, we apply the weighted Young inequality: for any $\epsilon\in(0,1)$,
\[
\|x+y\|^4
\le
(1+\epsilon)\|x\|^4
+C\epsilon^{-3}\|y\|^4 .
\]
Taking $\epsilon=\lambda\hat\mu$, which is admissible under $\lambda\hat\mu\le 1/4$, and using
\[
x=(I-\lambda\hess)e_{t-1},
\qquad
y=-\lambda\kappa\nu_{t-1},
\]
we obtain
\[
\|u_{t-1}\|^4
\le
(1+\lambda\hat\mu)\|(I-\lambda\hess)e_{t-1}\|^4
+
C(\lambda\hat\mu)^{-3}\lambda^4\kappa^4\|\nu_{t-1}\|^4 .
\]
Note that $\|(I-\lambda\hess)e_{t-1}\|\le (1-\lambda\hat\mu)\|e_{t-1}\|$, we have
\begin{equation}\label{eq:u4-bound}
\|u_{t-1}\|^4
\le
(1+\lambda\hat\mu)(1-\lambda\hat\mu)^4\|e_{t-1}\|^4
+
C\hat\mu^{-3}\lambda\kappa^4\|\nu_{t-1}\|^4 .
\end{equation}

Since $\xi_{t-1}$ is independent of $(\mathcal F_{t-1},\eta_{t-1})$ and mean zero, the cross term vanishes, and using \cref{eq:eta2-mom},
\begin{align}
\E_{t-1}\|w_{t-1}\|^2
&=
\lambda^2\E_{t-1}\|\eta_{t-1}\|^2+2\beta^{-1}\lambda D \\
&\le
\lambda^2\Big(2\tau_4^2/B+8L^2/B\,\|e_{t-1}\|^2\Big)
+2\beta^{-1}\lambda D.
\end{align}
Using $\|x+y\|^4\le 8(\|x\|^4+\|y\|^4)$, $\E\|\xi_{t-1}\|^4=D^2+2D$ and \cref{eq:eta4-mom},
\begin{align}
\E_{t-1}\|w_{t-1}\|^4
&\le
8\lambda^4\E_{t-1}\|\eta_{t-1}\|^4
+8(2\beta^{-1}\lambda)^2\E\|\xi_{t-1}\|^4 \nonumber\\
&\le
8\lambda^4\Big(24\tau_4^4/B^2+384L^4/B^2\|e_{t-1}\|^4\Big)
+32\beta^{-2}\lambda^2(D^2+2D).
\label{eq:w4-bound}
\end{align}

From the inequality $\|x-y\|^2\le (1+\lambda\hat\mu)\|x\|^2+(1+(\lambda\hat\mu)^{-1})\|y\|^2$, we obtain
\begin{align}
\label{eq:bound of mu}
\|u_{t-1}\|^2
\le
(1+\lambda\hat\mu)(1-\lambda\hat\mu)^2\|e_{t-1}\|^2
+
\bigl(1+(\lambda\hat\mu)^{-1}\bigr)\lambda^2\kappa^2\|\nu_{t-1}\|^2.
\end{align}
By \cref{eq:bound of mu} and $\|e\|^2\|\nu\|^2\le \tfrac12(\|e\|^4+\|\nu\|^4)$, it follows that
\begin{align}
8\|u_{t-1}\|^2\E_{t-1}\|w_{t-1}\|^2
&\le
C\left[
(1+\lambda\hat\mu)(1-\lambda\hat\mu)^2\frac{\lambda^2L^2}{B}
+
(1+(\lambda\hat\mu)^{-1})\frac{\lambda^4\kappa^2L^2}{B}
\right]\|e_{t-1}\|^4\\
&\quad+
C\left(\lambda^2\tau_4^2/B+\lambda D/\beta\right)
(1+\lambda\hat\mu)(1-\lambda\hat\mu)^2\|e_{t-1}\|^2 \\
&\quad+
C(1+(\lambda\hat\mu)^{-1})\frac{\lambda^4\kappa^2L^2}{B}\|\nu_{t-1}\|^4
\\
&\quad+C
\left(\lambda^2\tau_4^2/B+\lambda D/\beta\right)
(1+(\lambda\hat\mu)^{-1})\lambda^2\kappa^2\|\nu_{t-1}\|^2
.\label{eq:mixed-term-psi}
\end{align}

Combining \cref{eq:psi4-start,eq:u4-bound,eq:w4-bound,eq:mixed-term-psi,eq:UV-def-mom}, then taking full expectation, 
\begin{align}
\E\|e_t\|^4
\le
m_{11}\E\|e_{t-1}\|^4
+m_{12}\E\|\nu_{t-1}\|^4
+\mathcal A.
\end{align}
From \cref{eq:schemeB-mom}, $\nu_t = b_{t-1}+\eta_{t-1}$ with
$b_{t-1}:=\kappa\nu_{t-1}+\hess e_{t-1}$ and $\E_{t-1}[\eta_{t-1}]=0$.
Applying \cref{eq:uplusv} gives
\begin{align}
\E_{t-1}\|\nu_t\|^4
\le
\|b_{t-1}\|^4+8\|b_{t-1}\|^2\E_{t-1}\|\eta_{t-1}\|^2+3\E_{t-1}\|\eta_{t-1}\|^4.
\end{align}
Using $\|a+b\|^2\le 2\|a\|^2+2\|b\|^2$, $\|\hess\|_2=\hat L$, and $2ab\le a^2+b^2$,
\begin{align}
\|b_{t-1}\|^4
\le
C(\hat L^4+\kappa^2\hat L^2)\|e_{t-1}\|^4
+C(\kappa^4+\kappa^2\hat L^2)\|\nu_{t-1}\|^4.
\end{align}
For the cross term, $\|b_{t-1}\|^2\le C(\hat L^2\|e_{t-1}\|^2+\kappa^2\|\nu_{t-1}\|^2)$ together
with \cref{eq:eta2-mom} gives
\begin{align}
8\|b_{t-1}\|^2\E_{t-1}\|\eta_{t-1}\|^2
&\le
C\Big(\tfrac{\hat L^2L^2}{B}+\tfrac{L^4}{B^2}\Big)\|e_{t-1}\|^4
+C\tfrac{\tau_4^2}{B}\bigl(\hat L^2\|e_{t-1}\|^2+\kappa^2\|\nu_{t-1}\|^2\bigr)
+C\tfrac{\kappa^2L^2}{B}\|\nu_{t-1}\|^2.
\end{align}
Taking full expectation, bounding the second moments $\|e_{t-1}\|^2,\|\nu_{t-1}\|^2$ by
$\underline{\mathsf U},\underline{\mathsf V}$ (so that they contribute to $\mathcal B$), and using
\cref{eq:eta4-mom} for the remaining $\E_{t-1}\|\eta_{t-1}\|^4$ term, yields
\[
\E\|\nu_t\|^4
\le
m_{21}\E\|e_{t-1}\|^4
+m_{22}\E\|\nu_{t-1}\|^4
+\mathcal B.
\]
Thus
\begin{align}
\begin{pmatrix}
\E\|e_t\|^4\\
\E\|\nu_t\|^4
\end{pmatrix}
\le
M
\begin{pmatrix}
\E\|e_{t-1}\|^4\\
\E\|\nu_{t-1}\|^4
\end{pmatrix}
+
\begin{pmatrix}
\mathcal A\\
\mathcal B
\end{pmatrix}.
\end{align}
Since $M$ is entrywise nonnegative and $\rho(M)<1$, $I-M$ is a nonsingular $M$-matrix; hence
\[\Delta_4:=\det(I-M)=(1-m_{11})(1-m_{22})-m_{12}m_{21}>0\] and $(I-M)^{-1}=\sum_{j=0}^\infty M^j$
is entrywise nonnegative. 

As the recursion is a componentwise inequality with nonnegative
$M$, then we have
\begin{equation}
\limsup_{t\to\infty}
\begin{pmatrix}\E\|e_t\|^4\\ \E\|\nu_t\|^4\end{pmatrix}
\le(I-M)^{-1}\begin{pmatrix}\mathcal A\\ \mathcal B\end{pmatrix}.
\end{equation}
Since $\psi_t\dist\pi_\psi$ as $t\to\infty$, the first component on the left-hand side
converges to $\E\|\psi_\infty-\MLE\|^4$,
\begin{equation}
\E\|\psi_\infty-\MLE\|^4 \le \frac{(1-m_{22})\,\mathcal A + m_{12}\,\mathcal B}{(1-m_{11})(1-m_{22})-m_{12}m_{21}}.
\end{equation}
\end{proof}

We now prove \cref{cor:stationary-wasserstein-error-bound-momentum}. Under the notation of
\cref{thm:wasserstein-error-bound-momentum} and \cref{lem:fourth-moment-compact}, combining
the two results gives the master inequality
\begin{equation}\label{eq:S3-master}
W_2^2(\pi_\theta,\pi_\psi)
\le\frac{\mathcal P}{1-\bar\beta}\cdot
\frac{(1-m_{22})\,\mathcal A+m_{12}\,\mathcal B}{(1-m_{11})(1-m_{22})-m_{12}m_{21}}.
\end{equation}
Throughout, $C>0$ denotes a generic constant (depending only on
$\hat\mu,\hat L,L,\mu,c_\kappa$ and possibly changing between displays), while
$c_2,c_3,c_4$ and $C_1$ denote the specific constants fixed in the statement and below.
We bound the two
factors in turn.

For the quantities $X,Y$ of \cref{thm:wasserstein-error-bound-momentum}, the conditions
$\kappa=c_\kappa\lambda\le\tfrac12$ and $\lambda\le\mu/(c_2L^2)$
(whence $\lambda\mu\le\mu^2/(c_2L^2)\le1/c_2$, using $\mu\le L$) give $X\ge\tfrac14$ and
$Y\le32L^2c_\kappa\lambda^2$, hence by $\sqrt{X^2+Y}-X\le Y/(2X)$,
\begin{equation}\label{eq:sqrt-diff}
\sqrt{X^2+Y}-X\le\frac{Y}{2X}\le64L^2c_\kappa\lambda^2.
\end{equation}
Writing $1-\bar\beta=(1-\kappa)-\tfrac12(X+\sqrt{X^2+Y})$ and substituting
$X=1-\lambda\mu+\tfrac{2\lambda^2L^2}{1-\kappa}-(1-\lambda)\kappa$ with $\kappa=c_\kappa\lambda$,
\begin{equation}\label{eq:beta-bar-lower}
1-\bar\beta\ge(1-\kappa)-X-Y
=\lambda\mu-\frac{2\lambda^2L^2}{1-\kappa}-\lambda\kappa-Y
\ge\tfrac14\lambda\mu,
\end{equation}
the last step using $\lambda\le\mu/(c_2L^2)$. 
Substituting \cref{eq:sqrt-diff,eq:beta-bar-lower}
into the definition of $\mathcal P$,
\begin{equation}\label{eq:C1}
\frac{\mathcal P}{1-\bar\beta}
\le\frac{4}{\lambda\mu}\Big[
\frac{\lambda^2}{2(1-\kappa)}\overline{M^2}
+\frac{\lambda}{4\mu}\overline M^2
+\frac{\overline{M^2}}{8L^2}\big(\sqrt{X^2+Y}-X\big)\Big]
\le\frac{\overline M^2}{\mu^2}+\frac{4(1+8c_\kappa)}{\mu}\overline{M^2}\defas C_1.
\end{equation}
Write $S:=\lambda\tau_4^2/B+D/\beta$. Since $c_e=2\lambda S$, $c_\nu=2\tau_4^2/B$,
$a_{21}\le C$, $1-a_{11}\le C\lambda$, and $\Delta_2\ge\tfrac18\lambda\hat\mu$, Cramer's rule gives
\begin{equation}\label{eq:UV-bounds}
\underline{\mathsf U}\le CS,\qquad\underline{\mathsf V}\le CS+\frac{C\tau_4^2}{B}.
\end{equation}

We claim the step-size restrictions imply
\begin{equation}\label{eq:three-key-ineq}
1-m_{11}\ge\lambda\hat\mu,\qquad 1-m_{22}\ge\tfrac12,\qquad m_{12}m_{21}\le\tfrac14\lambda\hat\mu.
\end{equation}
We verify the three inequalities in \cref{eq:three-key-ineq} in turn.

For the first, with $s:=\lambda\hat\mu\le\tfrac14$,
\[
(1+s)(1-s)^4\le1-\tfrac{19}{8}s,
\qquad
C\Big\{\tfrac{\lambda^2L^2}{B}+\hat\mu^{-1}\tfrac{\lambda^3\kappa^2L^2}{B}+\tfrac{\lambda^4L^4}{B^2}\Big\}
\le\tfrac{11}{8}\lambda\hat\mu
\]
under $\lambda\le B\hat\mu/(c_3L^2)$, so subtracting gives $1-m_{11}\ge\lambda\hat\mu$.

For the second, since $\kappa=c_\kappa\lambda\le\tfrac12$,
\[
m_{22}=C\{\kappa^4+\kappa^2\hat L^2\}\le C c_\kappa^2\lambda^2(1+\hat L^2)\le\tfrac12
\]
for $c_\kappa$ and $\lambda$ within the stated ranges, so $1-m_{22}\ge\tfrac12$.

For the third, $m_{12}\le C\lambda^5$ and $m_{21}\le C$ give
\[
m_{12}m_{21}\le C\lambda^5\le\tfrac14\lambda\hat\mu
\qquad\text{whenever }\lambda\le(\hat\mu/c_4)^{1/4}.
\]

Combining \cref{eq:three-key-ineq},
\begin{equation}\label{eq:Delta4-lower}
\Delta_4=(1-m_{11})(1-m_{22})-m_{12}m_{21}
\ge\tfrac12\lambda\hat\mu-\tfrac14\lambda\hat\mu=\tfrac14\lambda\hat\mu>0.
\end{equation}
Since $M$ is entrywise nonnegative with $m_{11},m_{22}<1$ and $\det(I-M)=\Delta_4>0$, it
follows that $\rho(M)<1$.

Using $\Delta_4\ge\tfrac14\lambda\hat\mu$ and $1-m_{22}\le1$,
\begin{equation}\label{ineq:fourth-bound-compact}
\frac{(1-m_{22})\mathcal A+m_{12}\mathcal B}{\Delta_4}
\le C\lambda^{-1}\mathcal A+C\lambda^{-1}m_{12}\mathcal B.
\end{equation}
For the first term, since $(1+(\lambda\hat\mu)^{-1})\lambda^2\kappa^2\le C\lambda^3$,
\cref{lem:fourth-moment-compact} and $\lambda\le1$ give
\begin{equation}\label{eq:A-over-lambda-compact}
\lambda^{-1}\mathcal A
\le C\Big(\tfrac{\lambda^3\tau_4^4}{B^2}+\tfrac{\lambda(D^2+2D)}{\beta^2}
+S^2+\lambda^3S^2+\lambda^3\tfrac{\tau_4^2}{B}S\Big)
\le C\Big(\tfrac{\lambda^2\tau_4^4}{B^2}+\tfrac{\lambda D\tau_4^2}{B\beta}+\tfrac{D^2+2D}{\beta^2}\Big).
\end{equation}
For the second term, $m_{12}\le C\lambda^5$ and $\mathcal B\le C(\tau_4^4/B^2+\tau_4^2S/B)$ give
\begin{equation}\label{eq:m12B-over-lambda-compact}
\lambda^{-1}m_{12}\mathcal B
\le C\lambda^4\Big(\tfrac{\tau_4^4}{B^2}+\tfrac{\tau_4^2}{B}S\Big)
\le C\Big(\tfrac{\lambda^2\tau_4^4}{B^2}+\tfrac{\lambda D\tau_4^2}{B\beta}\Big).
\end{equation}
Substituting \cref{eq:sqrt-diff,eq:beta-bar-lower} into the definition of $\mathcal P$
in \cref{eq:definition-of-P},
\begin{equation}\label{eq:Amom}
\begin{split}
W_2^2(\pi_\theta,\pi_\psi)
&\le
\underbrace{C\,\tau_4^4}_{A_0^\star}\frac{\lambda^2}{B^2}
+\underbrace{C\,\tau_4^2 D}_{A_1^\star}\frac{\lambda}{B\beta}
+\underbrace{C\,(D^2+2D)}_{A_2^\star}\frac{1}{\beta^2}\\
&\le A_\star^2\Big(\tfrac{\lambda}{B}+\tfrac{1}{\beta}\Big)^2,
\end{split}
\end{equation}
where $A_\star^2\defas\max\{A_0^\star,A_1^\star,A_2^\star\}$ depends only on $(\mu,L,\hat\mu,\hat L,c_\kappa,\tau_4,D,\overline M^2,\overline{M^2})$
and is independent of $(\lambda,B,\beta)$.

\end{document}